\documentclass{article}

\usepackage{microtype}
\usepackage{graphicx}
\usepackage{subfigure}
\usepackage{booktabs} 

\usepackage{hyperref}

\usepackage[accepted]{icml2024}

\usepackage{amsmath}
\usepackage{amssymb}
\usepackage{mathtools}
\usepackage{amsthm}

\usepackage[capitalize,noabbrev]{cleveref}

\theoremstyle{plain}
\newtheorem{theorem}{Theorem}[section]

\newtheorem{lemma}[theorem]{Lemma}

\theoremstyle{definition}

\newtheorem{assumption}[theorem]{Assumption}
\theoremstyle{remark}

\usepackage[textsize=tiny]{todonotes}

\usepackage{ifthen}
\usepackage{thm-restate}
\usepackage{tikz}
\usetikzlibrary{positioning,chains,fit,shapes,calc}

\usepackage{amsfonts}
\usepackage{enumerate}
\usepackage{flushend}
\usepackage{mathrsfs}
\usepackage{makecell}
\usepackage{setspace}
\usepackage{xcolor}
\usepackage{algorithm}
\usepackage{algorithmic}
\usepackage{bbm}

\newcommand{\R}{\mathbb{R}}

\newcommand{\cA}{\mathcal{A}}

\newcommand{\cD}{\mathcal{D}}
\newcommand{\cE}{\mathcal{E}}
\newcommand{\cF}{\mathcal{F}}

\newcommand{\cH}{\mathcal{H}}

\newcommand{\cK}{\mathcal{K}}

\newcommand{\cM}{\mathcal{M}}

\newcommand{\cO}{\mathcal{O}}

\newcommand{\cS}{\mathcal{S}}

\newcommand{\cU}{\mathcal{U}}

\newcommand{\cX}{\mathcal{X}}

\newcommand{\argmax}{\operatornamewithlimits{argmax}}
\newcommand{\argmin}{\operatornamewithlimits{argmin}}
\mathchardef\mhyphen="2D
\newcommand{\ex}{\mathbb{E}}

\newcommand{\kl}{\textup{KL}}

\newcommand{\sbr}[1]{\left( #1 \right)}
\newcommand{\mbr}[1]{\left[ #1 \right]}
\newcommand{\lbr}[1]{\left\{ #1 \right\}}
\newcommand{\abr}[1]{\left| #1 \right|}
\newcommand{\nbr}[1]{\left\| #1 \right\|}

\newcommand{\indicator}[1]{\mathbbm{1}\left\{ #1 \right\}}
\newcommand{\trace}{\textup{tr}}

\newcommand{\algrlhf}{\mathtt{PG\mbox{-}RLHF}}
\newcommand{\algnpg}{\mathtt{NPG\mbox{-}Update}}

\newcommand{\algnnrlhf}{\mathtt{NN\mbox{-}PG\mbox{-}RLHF}}
\newcommand{\algnnnpg}{\mathtt{NN\mbox{-}NPG\mbox{-}Update}}

\newcommand{\unif}{\textup{Unif}}
\newcommand{\hf}{\textup{HF}}
\newcommand{\cover}{\textup{cov}}
\newcommand{\proj}{\textup{Proj}}
\newcommand{\supproj}{\textup{proj}}
\newcommand{\subsgd}{\textup{SGD}}
\newcommand{\bias}{\textup{bias}}
\newcommand{\stat}{Q}
\newcommand{\mle}{\textup{MLE}}
\newcommand{\nn}{\textup{NN}}
\newcommand{\init}{\textup{init}}
\newcommand{\poly}{\textup{Poly}}
\newcommand{\base}{\textup{base}}
\newcommand{\scale}{\textup{scale}}
\newcommand{\submid}{\textup{mid}}
\newcommand{\out}{\textup{out}}

\newcommand{\compilehidecomments}{true}
\ifthenelse{ \equal{\compilehidecomments}{true} }{
	\newcommand{\yihan}[1]{}
	\newcommand{\srikant}[1]{}
	\newcommand{\shie}[1]{}
	\newcommand{\gal}[1]{}
	\newcommand{\anna}[1]{}
}{
	\newcommand{\yihan}[1]{{\color{teal} [\text{Yihan:} #1]}}
	\newcommand{\srikant}[1]{{\color{red} [\text{Srikant:} #1]}}
	\newcommand{\shie}[1]{{\color{purple} \hfill SM: #1}}
	\newcommand{\gal}[1]{{\color{blue} [\text{Gal:} #1]}}
	\newcommand{\anna}[1]{{\color{orange} [\text{Anna:} #1]}}
}

\allowdisplaybreaks
\icmltitlerunning{Exploration-Driven Policy Optimization in RLHF: Theoretical Insights on Efficient Data Utilization}

\begin{document}
	
	\twocolumn[
	\icmltitle{Exploration-Driven Policy Optimization in RLHF: Theoretical Insights on Efficient Data Utilization}
	
	
	
	\icmlsetsymbol{equal}{*}
	
	\begin{icmlauthorlist}
		\icmlauthor{Yihan Du}{uiuc}
		\icmlauthor{Anna Winnicki}{uiuc}
		\icmlauthor{Gal Dalal}{nvidia}
		\icmlauthor{Shie Mannor}{technion,nvidia}
		\icmlauthor{R. Srikant}{uiuc}
	\end{icmlauthorlist}
	
	\icmlaffiliation{uiuc}{University of Illinois Urbana-Champaign}
	\icmlaffiliation{technion}{Technion}
	\icmlaffiliation{nvidia}{NVIDIA Research}
	
	\icmlcorrespondingauthor{Yihan Du}{yihandu@illinois.edu}
	\icmlcorrespondingauthor{R. Srikant}{rsrikant@illinois.edu}
	
	\icmlkeywords{Machine Learning, ICML}
	
	\vskip 0.3in
	]
	
	
	
	\printAffiliationsAndNotice{}  
	
	\begin{abstract}
		Reinforcement Learning from Human Feedback (RLHF) has achieved impressive empirical successes while relying on a small amount of human feedback. However, there is limited theoretical justification for this phenomenon. Additionally, most recent studies focus on value-based algorithms despite the recent empirical successes of policy-based algorithms.
		In this work, we consider an RLHF algorithm based on policy optimization (PO-RLHF). The algorithm is based on the popular Policy Cover-Policy Gradient (PC-PG) algorithm, which assumes knowledge of the reward function. In PO-RLHF, knowledge of the reward function is not assumed, and the algorithm uses trajectory-based comparison feedback to infer the reward function. We provide performance bounds for PO-RLHF with low query complexity, which provides insight into why a small amount of human feedback may be sufficient to achieve good performance with RLHF. 
		A key novelty
		is a trajectory-level elliptical potential analysis, which bounds the reward estimation error when comparison feedback (rather than numerical reward observation) is given.
		We provide and analyze algorithms $\algrlhf$ and $\algnnrlhf$ for two settings: linear and neural function approximation, respectively. 
		
		
	\end{abstract}

	\section{Introduction}
	
	Reinforcement Learning (RL)~\cite{sutton2018reinforcement, agarwal2021theory} is a classic sequential decision-making problem where an agent interacts with an unknown environment in order to maximize the expected cumulative reward.
	In many applications, e.g., robotics and Large Language Models (LLMs)~\cite{ouyang2022training,achiam2023gpt}, the goal of the agent is complex and related to human evaluation. Additionally, the reward function may be hard to manually design. 
	
	To handle these challenges, a framework called Reinforcement Learning from Human Feedback (RLHF)~\cite{christiano2017deep} has been proposed and has achieved huge empirical successes in ChatGPT~\cite{achiam2023gpt}. In RLHF, the agent does not directly observe rewards, but has access to queries from humans on preferences based on trajectories. The agent learns the quality of trajectories (policies) from the preference feedback over time in order to optimize performance.
	Existing empirical works have demonstrated the practical efficiency of RLHF: human feedback can solve complex RL tasks by using fewer than $1\%$ of the data from the agent's interactions with the environment~\cite{christiano2017deep}.
	
	Recently, there have also been a number of theoretical RL papers which seek to provide analyze RLHF, e.g., \cite{pacchiano2021dueling,chen2022human,zhu2023principled,wang2023rlhf}.
	Most of these works consider value-based algorithms or a given dataset of human feedback, while in many applications, e.g., ChatGPT, policy optimization algorithms are often used. Our goal is to quantify the query and sample complexities of policy-based algorithms when used in conjunction with RLHF, and show that the query complexity is a small fraction of the overall sample complexity.
	
	In order to address the aforementioned issues, we study Policy Optimization for RLHF (PO-RLHF) with active human feedback, through which we provide insights on the query efficiency of RLHF. 	
	The algorithm can be summarized as follows. It is an iterative process where at each iteration there is a policy, and several trajectories are drawn by following the policies obtained so far. The trajectories are compared to trajectories generated by following a baseline policy. Humans make comparisons between the trajectories generated by the two policies. Assuming a Bradley-Terry model~\cite{bradley1952rank}, the algorithm uses the results of the comparison queries to update the estimate of the underlying reward function. Then, there is an inner loop where the algorithm follows several steps of the PC-PG policy optimization algorithm~\cite{agarwal2020pc} using the estimated reward model. At each step of the inner loop, there is a current estimate of a parameter corresponding to the policy, and Monte Carlo simulations of the policy are used to determine the subsequent policy parameter. 
	
	Under this formulation, we consider two settings, i.e., linear function approximation and neural function approximation, where the reward function is linear and belongs to a neural function class, respectively.		
	For both settings, we design policy gradient algorithms, $\algrlhf$ and $\algnnrlhf$, which can efficiently explore the unknown environment and collect human feedback adapting to the exploration. We establish sample and query complexity guarantees for these two algorithms. 
	
	While our algorithm is based on the PC-PG algorithm \cite{agarwal2020pc}, unlike PC-PG where the reward function is assumed to be known, we assume the use of human feedback. In order to take into account human feedback, we first extend the PC-PG analysis techniques to incorporate sources of error in the rewards. Then we directly quantify the error from human feedback. Characterizing error from human feedback is challenging for the following reason. The standard tool to analyze policy-based methods with exploration is the elliptical potential lemma \cite{abbasi2011improved}. However, this lemma has previously been used only in the case where the reward information is generated and  observed from each individual state-action.
	\emph{A key novelty in our paper is in transforming our error terms involving feature covariance matrices into a trajectory-wise form such that the elliptical potential lemma can be applied to our RLHF situation}, i.e., where rewards are not observed and two trajectories are compared based on the sum of rewards at all (state, action) pairs in each trajectory. 
	To address this issue, we develop a novel trajectory-level elliptical potential analysis technique (for more, see Section~\ref{sec:theoretical_guarantee_linear}). 
	
	Our results are consistent with the empirical observation that a small amount of human feedback is sufficient for RLHF to be successful. The reason is clear: human feedback is used to estimate the reward function which is then used in the policy-based RL algorithm. In other words, the during the policy update and policy evaluation phases of our algorithm, the reward estimate is fixed. While it may take many iterations of gradient ascent and many samples to evaluate each policy, the number of queries required to estimate the reward function is a small fraction of the overall sample complexity.
	
	We summarize our main contributions as follows:
	\begin{itemize}
		\item Motivated by the success of RLHF, we study policy optimization for RLHF with exploration and active human feedback collection, and seek to theoretically explain the practical efficiency of RLHF.
		
		\item For linear and neural function approximation, we design provably efficient algorithms $\algrlhf$ and $\algnnrlhf$, which  simultaneously explore the unknown environment and adaptively collect human data according to the exploration.
		
		\item We develop novel analytical techniques, including a trajectory-level elliptical potential argument and a  biased MLE guarantee with neural approximation.
		
		\item We provide justification for the practical efficiency of RLHF through a rigorous comparison of sample complexity between RLHF and standard RL. 
	\end{itemize}

	\section{Related Work}
	
	In this section, we discuss works that are most closely related to ours, and defer a detailed review to Appendix~\ref{apx:related_works}.
	
	RLHF~\cite{christiano2017deep} has shown great empirical successes, especially in LLMs~\cite{ouyang2022training,achiam2023gpt}. Recently, a number of works have started to theoretically analyze RLHF. \citet{xu2020preference,novoseller2020dueling,pacchiano2021dueling} study online RLHF for tabular MDPs. \citet{chen2022human,wang2023rlhf} consider online RLHF with general function approximation. \citet{wang2023rlhf} design a reduction framework for RLHF, and prove that the sample complexity for RLHF is no higher than that for standard RL. \citet{zhu2023principled,zhan2023provable} study offline RLHF with function approximation.
	\citet{ji2023provable} seek to understand the empirical success of RLHF from the perspective of intrinsic data bias. 
	
	Different from the above works which mostly consider value-based algorithms, we analyze policy gradient RLHF algorithms with exploration, and show that the amount of data needed to implement RLHF is a small fraction of the amount of data needed to train an RL algorithm.
	
	Our work is also related to prior neural RL works, e.g.,~\cite{cai2019neural,wang2019neural,xu2021crpo}, which theoretically analyze neural function approximation.
	%

	\section{Formulation}\label{Formulation Section}
	
	In this section, we formally define the PO-RLHF problem.
	
	We consider a discounted MDP $\cM(\cS,\cA,r,P,\gamma,s_{\init})$. Specifically, $\cS$ is the state space, and $\cA$ is the action space. $r:\cS\times\cA\rightarrow[0,1]$ is an underlying reward function, so that $r(s,a)$ specifies the reward of taking action $a$ in state $s$. In the RLHF setting, the agent cannot directly observe $r(s,a)$, and instead, can only observe comparison feedback between trajectories generated according to $r$ (detailed shortly).
	$P:\cS \times \cA \rightarrow \triangle_{\cS}$ is an unknown transition distribution, and $P(s'|s,a)$ gives the transition probability of transitioning to $s'$ if action $a$ is taken in state $s$. Here for any set $\cX$, $\triangle_{\cX}$ denotes the space of all distributions over $\cX$, and
	$\gamma \in [0,1)$ is a discount factor.
	We define a policy as a mapping $\pi:\cS\rightarrow\triangle_{\cA}$ which specifies what action to take in a state.
	
	Let the state at step $h$ be denoted by $s_h$, and the action taken at step $h$ be denoted by $a_h$. The value function 
	$$
	V^{\pi}(s):=\ex\bigg[\sum_{h=0}^{\infty} \gamma^h r(s_h,a_h)|s_0=s,\pi\bigg]
	$$ 
	and the state-action value function 
	$$
	Q^{\pi}(s,a):=\ex\bigg[\sum_{h=0}^{\infty} \gamma^h r(s_h,a_h)|s_0=s,a_0=a,\pi\bigg]$$
	denote the expected sum of  discounted rewards received under policy $\pi$, starting from a given state $s$ and  state-action pair $(s,a)$, respectively. We define the optimal policy as $\pi^*:=\argmax_{\pi}V^{\pi}(s_{\init})$.

	The RLHF model is as follows. The agent starts from an initial state $s_{\init}$. At each step $h$, the agent first observes the current state $s_h$, and then takes an action $a_h$ according to her policy. After that, she obtains an underlying reward $r(s,a)$ (not observed), and transitions to a next state $s_{h+1} \sim P(\cdot|s_h,a_h)$. The agent can choose to terminate the current trajectory with probability $1-\gamma$ and restart from $s_\init$ at each step.
	The agent can query humans to compare trajectories $\tau^{(1)}$ and $\tau^{(2)}$, and observe preference feedback $y$. Following the literature~\cite{pacchiano2021dueling,zhu2023principled}, we consider the classic Bradley-Terry model~\cite{bradley1952rank} to formulate preference generation:
	\begin{align}
		\Pr[y=1]&= \frac{ 1 }{ 1 + \exp( -  \tilde{r}^{\tau^{(1)},\tau^{(2)}}  ) } , \label{eq:BT_model}
	\end{align}
	with $\Pr[y=0]=1-\Pr[y=1].$
	Here $y=1$ represents that $\tau^{(1)}$ is preferred to $\tau^{(2)}$, and $y=0$ denotes the opposite case. 
	$$
	\tilde{r}^{\tau^{(1)},\tau^{(2)}}:=\sum_{h=0}^{H(\tau^{(1)})} r(s^{(1)}_h,a^{(1)}_h)-\sum_{h=0}^{H(\tau^{(2)})} r(s^{(2)}_h,a^{(2)}_h),
	$$ 
	and $H(\tau)$ denotes the length of trajectory $\tau$.

	Given a confidence parameter $\delta$ and an accuracy parameter $\varepsilon$, the goal of the agent is to identify an $\varepsilon$-optimal policy $\hat{\pi}$ which satisfies $V^{\pi^*}(s_{\init})-V^{\hat{\pi}}(s_{\init}) \leq \varepsilon$ with probability at least $1-\delta$. Before we describe our reward function model, we first introduce some useful notation.
	

	\textbf{Notation.}
	For any $(s',a') \in \cS \times \cA$ and policy $\pi$, let $d^{\pi}_{s',a'}(s,a):=(1-\gamma)\ex[\sum_{h=0}^{\infty} \gamma^h \Pr[s_h=s,a_h=a|s_0=s',a_0=a',\pi]$ denote the discounted state-action distribution of starting from $(s',a')$ and executing $\pi$.
	With a slight abuse of notation, for any $s' \in \cS$, let $d^{\pi}_{s'}(s,a):=\ex_{a' \in \pi(\cdot|s')}[d^{\pi}_{s',a'}(s,a)]$. 
	For any initial distribution $\rho \in \triangle_{\cS\times\cA}$, let $d^{\pi}_{\rho}(s,a):=\ex_{(s',a') \sim \rho}[d^{\pi}_{s',a'}(s,a)]$.
	In addition, for any $(s',a') \in \cS \times \cA$ and policy $\pi$, let $\cO_{s',a'}^{\pi}$ be the distribution of the trajectory generated by starting from $s',a'$, executing $\pi$ and terminating with probability $1-\gamma$ at each step, which we call a \emph{discounted trajectory distribution}. 
	For any $\rho \in \triangle_{\cS\times\cA}$, let $\cO_{\rho}^{\pi}$ be the discounted trajectory distribution of starting from $\rho$ and executing $\pi$.
	
	Under this formulation, we consider linear and neural function approximation settings for the reward model. 	
	\subsection{Linear Function Approximation}
	
	In the linear setting, we consider the log-linear policy parameterization and linear reward function. Specifically, there exists a known feature mapping $\phi: \cS \times \cA \rightarrow \R^{d}$ which specifies the feature vectors of state-action pairs, and satisfies $\|\phi(s,a)\|\leq 1$ for all $(s,a) \in \cS \times \cA$. For parameter $w \in \R^{d}$, the log-linear policy is represented as
	\begin{align*}
		\pi_{w}(a|s) := \frac{\exp(\phi(s,a)^\top w)}{\sum_{a' \in \cA} \exp(\phi(s,a')^\top w)} .
	\end{align*}
	
	We make the following assumption on the reward function.
	
	\begin{assumption}[Linear Reward Function]
		There exists some reward parameter $\mu^{*} \in \R^{d}$ such that
		\begin{align*}
			r(s,a) := \phi(s,a)^\top \mu^{*} .
		\end{align*}
	\end{assumption}
	
	%



	\subsection{Neural Function Approximation}
	
	In the neural function approximation setting, we parameterize the policy,  value function and reward by neural networks.
	
	A two-layer ReLU neural network with input feature $\phi(s,a)$, parameter $w$ and width $m$ is represented by~\cite{cai2019neural,xu2021crpo}
	\begin{align*}
		f(s,a;w) \!=\! \frac{1}{\sqrt{m}} \sum_{\ell=1}^{m} b_{\ell} \ \mathbbm{1}\{ \phi(s,a)^\top \! [w]_{\ell} \!>\! 0\} \ \phi(s,a)^{\!\top} [w]_{\ell} ,
	\end{align*}
	where $b:=[b_1,\dots,b_m]^\top \in \R^{m}$, and $w:=[[w]_1;\dots;[w]_m] \in \R^{md}$ are the network parameters. 
	
	We initialize the parameters by $b_{\ell} \sim \unif([-1,1])$ and $[w^0]_{\ell} \sim \cD_{\init}$ for any $\ell \in [m]$.  Here $\cD_{\init}$ is an initialization distribution, such that for any $w' \in \R^d$ in the support of $\cD_\init$, $\underline{c} \leq \|w'\|_2 \leq \bar{c}$ for some constants $\underline{c},\bar{c}>0.$ During training, we keep $b$ fixed and only update $w$.

	With a temperature parameter $\alpha \in \R$ and a network parameter $w \in \R^{md}$, a policy is represented by
	\begin{align*}
		\pi_{\alpha,w}:=\frac{\exp( \alpha f(s,a;w))}{\sum_{a' \in \cA} \exp( \alpha f(s,a;w))} ,
	\end{align*}
	
	We also use $f(s,a;\theta)$ to approximate the state-action value function $Q^{\pi}$ with another parameter $\theta \in \R^{md}$ and the same initialization as $w$, i.e., $\theta^0=w^0$.
	
	Moreover, we approximate the reward function $r(s,a)$ by
	\begin{align*}
		h(s,a;\mu)\!:=\!\frac{1}{\sqrt{m}} \!\sum_{\ell=1}^{m} b'_{\ell} \ \mathbbm{1}\{ \phi(s,a)^{\!\!\top}  [\mu]_{\ell} \!>\! 0\} \ \phi(s,a)^{\!\!\top} [\mu]_{\ell} ,
	\end{align*}
	where $b':=[b'_1,\dots,b'_m]^\top \in \R^{m}$ and $\mu:=[[\mu]_1;\dots;[\mu]_m] \in \R^{md}$ are the reward network parameters. Similarly, we initialize $b'_{\ell} \sim \unif([-1,1])$ and $[\mu^0]_{\ell} \sim \cD_{\init}$ for any $\ell \in [m]$, and only update $\mu$ during training.

	For any parameter $\mu \in \R^{md}$ and $(s,a) \in \cS \times \cA$, let $[\psi_{\mu}(s,a)]_{\ell}:= \frac{b'_{\ell}}{\sqrt{m}} \ \ \indicator{ \phi(s,a)^\top [\mu]_{\ell} >0} \ \ \phi(s,a) \in \R^{d}$ for any $\ell \in [m]$. Let $\psi_{\mu}(s,a):=[[\psi_{\mu}(s,a)]_{1};\dots;[\psi_{\mu}(s,a)]_{m}] \in \R^{md}$. We can similarly define $\psi_{w}(s,a)$.
	
	Define a neural function class~\cite{rahimi2007random}:
	\begin{align*}
		&\cF^{\mu}_{R,\infty} \!:=\! \Big\{h(s,a)=h(s,a;\mu^0) \!+\! \int \indicator{ \phi(s,a)^\top \mu >0 } \!\cdot
		\\
		&\phi(s,a)^\top \nu(\mu) \ d p(\mu) :\ \nbr{ \nu(\mu) }_{\infty} \leq \frac{R}{\sqrt{d}} \Big\} ,
	\end{align*}
	where $p:\R^d \rightarrow \R$ is the density function of $\cD_{\init}$, and $\nu:\R^d \rightarrow \R^d$ together with $h(s,a;\mu^0)$ parameterize the element of $\cF^{\mu}_{R,\infty}$. 
	
	In the neural setting, we make the following assumptions.
	
	\begin{assumption}[Neural Realizability of $r$]
		$r \in \cF_{R,\infty}$.
	\end{assumption}
	
	This is a standard realizability assumption, and also made in prior neural RL works~\cite{wang2019neural,xu2021crpo}.
	
	
	\begin{assumption}[Regularity of State-action Distribution]\label{assumption:scale_theta_0}
		There exists an absolute constant $c_{\scale}\in(0,1)$ such that for any $v \in \R^{d}$, $x>0$, $(s',a') \in \cS \times \cA$ and policy $\pi$,
		\begin{align*}
			\ex_{(s,a) \sim d_{s',a'}^{\pi}}\mbr{\indicator{ \abr{\phi(s,a)^\top v } \leq x }} \leq   \frac{c_{\scale} x}{\|v\|_2} .
		\end{align*}
	\end{assumption}
	\gal{justify also the 2nd assumption}\yihan{added}
	
	This is also a standard regularity assumption in the neural RL literature~\cite{cai2019neural,wang2019neural,xu2021crpo}. For a random state-action pair $(s,a) \sim d_{s',a'}^{\pi}$, the probability of $|\phi(s,a)^\top v | \leq x$ scales with $x$ and $\|v\|_2^{-1}$.
	
	\subsection{Baseline Policy}
	
	We assume that we have access to a baseline policy, which will be used for comparison in our algorithms. 
	
	For any trajectory $\tau=(s_0,a_0,\dots,s_{H(\tau)},a_{H(\tau)})$ and feature mapping $\chi \in \{\phi,\psi_{\mu^0}\}$, 
	let $\chi(\tau):=\sum_{h=0}^{H(\tau)} \chi(s_h,a_h)$.
	
	\begin{assumption}[Baseline Policy]\label{key} \label{assumption:baseline_policy}
		The baseline policy $\pi^{\base}$ satisfies that for any $(s,a) \in \cS \times \cA$ and policy $\pi$,
		\begin{align*}
			&\ex_{\begin{subarray}{l} \tau^{(1)} \sim \cO^{\pi}_{s,a}\\ \tau^{(2)} \sim \cO^{\pi^{\base}}_{s_\init} \end{subarray}}  [ (\chi(\tau^{(1)}) - \chi(\tau^{(2)})) (\chi(\tau^{(1)}) - \chi(\tau^{(2)}))^\top ] 
			\\
			&\succeq c_{\base} \ex_{\tau^{(2)} \sim \cO^{\pi^{\base}}_{s_\init}}  [  \chi(\tau^{(2)}) \chi(\tau^{(2)})^\top ] 
		\end{align*}
		for some absolute constant $c_{\base}\in(0,1)$. Here $\chi=\phi$ in the case of linear function approximation, and $\chi=\psi$ in the case of neural function approximation.
	\end{assumption}

	We discuss Assumption~\ref{assumption:baseline_policy} in more detail in Appendix~\ref{apx:justify_assump_baseline}.
	
	\section{PO-RLHF with Linear Function Approximation}
	
	We first study PO-RLHF with linear function approximation. We develop a policy gradient algorithm $\algrlhf$ which can explore the environment and adaptively collect human data.
	

	\begin{algorithm}[t]
		\caption{$\algrlhf$}
		\label{alg:pg_rlhf}
		\begin{algorithmic}[1]
			\STATE {\bfseries Input:} $\varepsilon, \delta, N,  K, M_{\hf}, \zeta_{\hf}, \zeta_{\cover}, \pi^{\base}, W_{\mu}, \pi^0$.
			\FOR{$n=0,\dots,N-1$} \label{line:outer_loop_start}
			\STATE Sample $\{s_i,a_i\}_{i=1}^{K} \sim d^{\pi^n}_{s_{\init}}$, and $\hat{\Sigma}^n \leftarrow \frac{1}{K} \sum_{i=1}^{K} \phi(s_i,a_i) \phi(s_i,a_i)^\top$ \label{line:sample_cov}
			\STATE $\hat{\Sigma}^n_{\cover} \leftarrow \sum_{i=0}^{n} \hat{\Sigma}^i + \zeta_{\cover} I$
			\STATE Let $\rho^{n}_{\cover}:=\frac{1}{n+1}\sum_{i=0}^{n} d^{\pi^{i}}_{s_{\init}}$ \label{line:rho_cover_n}
			\STATE $\cO^n_{\hf}:=\frac{1}{n}\sum_{i=1}^{n} \cO^{\pi^{i}}_{\rho^{i-1}_{\cover}}, \forall n\geq1$, and $\cO^0_{\hf}:= \cO^{\pi^{0}}_{s_{\init}}$ \label{line:def_cO_hf}
			\FOR{$i=1,\dots,M_{\hf}$} \label{line:human_data_start}
			\STATE Sample  trajectories $\tau^{(1)}_i \!\sim\! \cO^n_{\hf}$ and $\tau^{(2)}_i \!\sim\! \cO^{\pi^{\base}}_{s_{\init}}$ \label{line:def_traj_dis}
			\STATE Observe the comparison outcome $y_i$ 
			\ENDFOR \label{line:human_data_end}
			\STATE Estimate $\hat{\mu}^{n}$ via MLE as in Eq.~\eqref{eq:mle_mu_r} \label{line:mle_mu}
			\STATE $\pi^{n+1} \leftarrow \algnpg(\rho^n_{\cover}, \hat{\Sigma}^n_{\cover}, \hat{\mu}^{n})$
			\ENDFOR \label{line:outer_loop_end}
			\STATE {\bfseries return} $\unif(\pi^{1},\dots,\pi^{N})$
		\end{algorithmic}
	\end{algorithm}

	\begin{algorithm}[t]
		\caption{$\algnpg$}
		\label{alg:npg}
		\begin{algorithmic}[1]
			\STATE {\bfseries Input:} $\rho^n_{\cover}, \hat{\Sigma}^n_{\cover}, \hat{\mu}^{n},T,\xi,\eta, W_{\theta}$. 
			\STATE Let $\hat{r}^n(\cdot,\cdot) := \phi(\cdot,\cdot)^\top \hat{\mu}^{n}$ and $\Theta \!:=\! \{\theta: \|\theta\|_2 \leq W_{\theta}\}$ \label{line:def_r}
			\STATE Let $b^n(\cdot,\cdot):=\frac{1}{1-\gamma} \mathbbm{1}\{ \phi(\cdot,\cdot)^\top (\hat{\Sigma}^n_{\cover})^{-1} \phi(\cdot,\cdot) \geq \beta \}$ \label{line:def_b}
			\STATE Let $\cK^n:=\{s \in \cS: \forall a \in \cA,\ b^n(s,a)=0 \}$ \label{line:def_cK}
			\STATE For $s \in \cK^n$, initialize $w^0$ such that $\pi^0(\cdot|s):=\pi_{w^0}(\cdot|s)=\unif(\cA)$. For $s \notin \cK^n$, $\pi^0(\cdot|s):=\unif(\{a \in \cA: b^n(s,a)=\frac{1}{1-\gamma}\})$
			\FOR{$t=0,\dots,T-1$} \label{line:inner_loop_start}
			\STATE Initialize $\theta^{t,0}$
			\FOR{$i=0,\dots,M_{\subsgd}-1$} \label{line:sgd_start}
			\STATE \hspace*{-0.5em} Sample $(s_i,a_i) \!\sim\! \rho^n_{\cover}$ and estimate $\hat{Q}^{\pi^t}(s_i,a_i;\hat{r}^n+b^n)$ using Monte Carlo sampling \label{line:monte_carlo}
			\STATE $\theta^{t,i+1} \leftarrow \proj_{\Theta} ( \theta^{t,i} - 2\xi ( \phi(s_i,a_i)^\top \theta^{t,i} - (\hat{Q}^{\pi^t}(s_i,a_i;\hat{r}^n+b^n)-b^n(s_i,a_i)) ) \cdot \phi(s_i,a_i) ) $ 
			\ENDFOR \label{line:sgd_end}
			\STATE  $\theta^{t} \leftarrow \frac{1}{M_{\subsgd}} \sum_{i=0}^{M_{\subsgd}-1} \theta^{t,i}$\label{line:sgd_result_theta_t}
			\STATE $w^{t+1} \leftarrow w^t+\eta \theta^{t}$
			\STATE $\forall s \in \cK^n$,  $\pi^{t+1}(\cdot|s) = \pi_{w^{t+1}}(\cdot|s) \propto  \exp(\phi(s,\cdot)^\top w^{t+1}) $. $\forall s \notin \cK^n$,  $\pi^{t+1}(\cdot|s) = \pi^0(\cdot|s)$ \label{line:set_policy_t+1}
			\ENDFOR \label{line:inner_loop_end}
			\STATE {\bfseries return} $\unif(\pi^0,\dots,\pi^{T-1})$
		\end{algorithmic}
	\end{algorithm}
	
	\subsection{Algorithm $\algrlhf$} \label{sec:alg_rlhf}
	
	\shie{need to clearly explain differences from vanilla algorithm and also in the algorithm description per se, perhaps emphasizing in boldface}\yihan{added it}
	$\algrlhf$ builds upon the policy gradient algorithm PC-PG~\cite{agarwal2020pc} for standard RL. 
	Our algorithm is described in
	Algorithm~\ref{alg:pg_rlhf}. $\algrlhf$ runs $N$ outer-loop phases for coverage update and reward estimation (Lines~\ref{line:outer_loop_start}-\ref{line:outer_loop_end} in Algorithm~\ref{alg:pg_rlhf}), and $T$ inner-loop iterations for policy optimization under given coverage and reward model (Lines~\ref{line:inner_loop_start}-\ref{line:inner_loop_end} in Algorithm~\ref{alg:npg}).
	In each phase $n$, $\algrlhf$ first estimates the feature covariance matrix $\hat{\Sigma}^n_{\cover}$ and updates the state-action coverage distribution $\rho^n_{\cover}$, which is the average of the state-action visitation distribution of all the policies $\pi^0,\dots,\pi^n$ used so far (Line~\ref{line:rho_cover_n} in Algorithm~\ref{alg:pg_rlhf}).
	$\rho^n_{\cover}$ will be the initial state-action distribution of the policy optimization in the inner-loop, and is gradually expanded in each phase to improve the coverage.
	
	\textbf{Human Feedback Collection.} Next, we collect human data for reward estimation.
	For any phase $n\geq1$, let $\cO^n_{\hf}$ be the distribution of the trajectory generated by starting from state-action distribution $\rho^{\bar{n}-1}_{\cover}$, executing $\pi^{\bar{n}}$ and terminating with probability $1-\gamma$ at each step, where $\bar{n} \sim \unif([n])$; For phase $n=0$, $\cO^n_{\hf}:=\cO^{\pi^0}_{s_{\init}}$ (Line~\ref{line:def_cO_hf}). In addition, let $\cO^{\pi^{\base}}_{s_{\init}}$ be the distribution of the trajectory generated by starting from $s_{\init}$, executing $\pi^{\base}$ and stopping with probability $1-\gamma$ at each step, where $\pi^{\base}$ is a baseline policy (Line~\ref{line:def_traj_dis}). We sample trajectories $\tau^{(1)}_i$ and $\tau^{(2)}_i$ from $\cO^n_{\hf}$ and $\cO^{\pi^{\base}}_{s_{\init}}$, respectively, and observe a comparison outcome $y_i$. This process is independently repeated $M_{\hf}$ times, and then we obtain human data $\{\tau^{(1)}_i,\tau^{(2)}_i,y_i\}_{i=1}^{M_{\hf}}$ (Lines~\ref{line:human_data_start}-\ref{line:human_data_end}). 
	
	With the human data, we use the maximum likelihood estimator (MLE) to estimate the reward parameter as 
	\begin{align}
		\hat{\mu}^{n} & = \argmin_{\|\mu\|_2 \leq W_{\mu}} \bigg( - \sum_{i=1}^{M_{\hf}} \log \Big( \frac{ \indicator{y_i=1} }{ 1+ \exp\big(  -(\tilde{\phi}^{\tau^{(1)}_i, \tau^{(2)}_i})^\top \mu \big) } 
		\nonumber\\ & \ \quad + \frac{ \indicator{y_i=0} }{ 1+ \exp\big(  (\tilde{\phi}^{\tau^{(1)}_i, \tau^{(2)}_i})^\top \mu \big) } \Big) \bigg) , \label{eq:mle_mu_r}
	\end{align}
	where $\tilde{\phi}^{\tau^{(1)}_i, \tau^{(2)}_i}:=\sum_{h=0}^{H(\tau^{(1)}_i)}\phi(s^{(1)}_{i,h},a^{(1)}_{i,h}) - \sum_{h=0}^{H(\tau^{(2)}_i)}\phi(s^{(2)}_{i,h},a^{(2)}_{i,h})$, and $(s^{(\ell)}_{i,h},a^{(\ell)}_{i,h})$ denotes the state-action at step $h$ in trajectory $\tau^{(\ell)}_i$ for any $\ell \in \{1,2\}$.
	
	Comparing to a fixed baseline policy helps to de-correlate the comparison (difference) relationship between two trajectories in the Bradley-Terry model (Eq.~\eqref{eq:BT_model}), and provides a better control for the properties of the human data covariance matrix to cover the state-actions that we care about.
	\gal{why do we compare to a single baseline all along? And how come we don't have any special requirements for this baseline?  can we use a random policy instead and would it be better to re-sample multiple such policies? It's also good to explain what we talked about in the meeting -- that we need an external policy to de-correlate the samples etc.}\yihan{Added the explanation. I found that the baseline policy needs to satisfy Assumption~\ref{assumption:baseline_policy}, and added it}
	
	\textbf{Intuition of Human Feedback Collection.} The idea behind our human data collection scheme is as follows. Since we will do policy optimization with initial state-action distribution $\rho^n_{\cover}$ and obtain policy $\pi^{n+1}$ in each phase $n$, our performance will be influenced by the reward estimation accuracy on the state-actions guided by $\pi^{n+1}$ starting from $\rho^n_{\cover}$ for $n=0,1,\dots,N-1$. Therefore, using the human data generated by $\pi^{\bar{n}}$ and $\rho^{\bar{n}-1}_{\cover}$ ($\bar{n} \sim \unif([n])$) can guarantee a small reward estimation error on the state-action space that we care about (where our performance is measured).
	
	With the coverage distribution $\rho^n_{\cover}$, coverage covariance matrix $\hat{\Sigma}^n_{\cover}$ and estimated reward model $\hat{r}^n(\cdot,\cdot) =\phi(\cdot,\cdot)^\top \hat{\mu}^{n}$, we call subroutine $\algnpg$ to perform policy optimization. 
	In $\algnpg$ (Algorithm~\ref{alg:npg}), we first define the exploration bonus $b^n(s,a):=\frac{1}{1-\gamma}$ for the state-actions that are not sufficiently explored, and define $b^n(s,a):=0$ for those that are sufficiently explored according to $\hat{\Sigma}^n_{\cover}$ (Line~\ref{line:def_b}). According to $b^n(s,a)$, we implicitly divide the state space into two state sets, one with well-explored state-actions (i.e., $\cK^n$), and the other one with under-explored state-actions (Line~\ref{line:def_cK}). 
	
	
	Then, we perform natural policy gradient (NPG)~\cite{agarwal2021theory} with initial state-action distribution $\rho^n_{\cover}$ and bonus-incentivized reward $\hat{r}^n+b^n$ (Lines~\ref{line:inner_loop_start}-\ref{line:inner_loop_end}). Formally, the optimization objective can be written as
	\begin{align*}
		\max_{\pi} \ \ex_{(s,a) \sim \rho^n_{\cover}} \mbr{ Q^{\pi}(s,a; \hat{r}^n+b^n) } .
	\end{align*}
	In the $t$-th iteration of NPG, we use projected stochastic gradient descent (SGD)~\cite{shalev2014understanding} to fit (Lines~\ref{line:sgd_start}-\ref{line:sgd_end})
	\begin{align}
		&\argmin_{\|\theta\|\leq W_{\theta}} \ \ex_{(s,a) \sim \rho^n_{\cover}} \bigg[ \Big( \phi(s,a)^\top \theta 
		\nonumber\\
		&\quad\ - \big( Q^{\pi^t}(s,a; \hat{r}^n+b^n) - b^n(s,a) \big) \Big)^2 \bigg] . \label{eq:objective_npg}
	\end{align}
	At step $i$ of SGD, we compute the stochastic gradient by $2( \phi(s_i,a_i)^\top \theta - ( \hat{Q}^{\pi^t}(s_i,a_i; \hat{r}^n+b^n) - b^n(s,a) ) ) \cdot \phi(s,a)$, where $(s_i,a_i)$ is sampled from $\rho^n_{\cover}$ and $\hat{Q}^{\pi^t}(s_i,a_i; \hat{r}^n+b^n)$ is estimated by Monte Carlo sampling (Line~\ref{line:monte_carlo}). 
	After SGD, we obtain $\theta^t$ such that $\phi(s,a)^\top \theta^t + b^n(s,a)$ well fits the state-action value function $\hat{Q}^{\pi^t}(s_i,a_i; \hat{r}^n+b^n)$ (Line~\ref{line:sgd_result_theta_t}).
	
	Then, we update the policy parameter by $w^{t+1} \leftarrow w^t+\eta \theta^{t}$.
	Furthermore, we set the policy $\pi^{t+1}$ as the log-linear policy with parameter $w^{t+1}$ for $s \in \cK^n$, and the uniform policy over all under-explored actions for $s \notin \cK^n$ (Line~\ref{line:set_policy_t+1}). 
	
	After NPG, we obtain $\pi^{n+1} = \unif(\pi^0,\dots,\pi^{T-1})$, which both optimizes the value function and has an incentive to explore the unvisited space.
	In the next phase, $\pi^{n+1}$ is used to improve the coverage, and also expand the space where we collect human data and can guarantee accurate reward estimation. 
	\gal{the last few parahraphs were supposed to explain the intuition of the HF collection process, but what they really explain is (a part of) details of the technique and its proof. Can we say something more meaningful? }\yihan{the insight of human data collection is highlighted in the paragraph in bold font}

	\textbf{Computational Efficiency.}
	We remark that the computational complexity of $\algnpg$ is independent of $\cS$. $b^n$, $\cK^n$ and $\pi^t$ are only implicitly maintained by computing  $\hat{\Sigma}^n_{\cover}$ and $w^t$. When we encounter some state $s$ in Monte Carlo sampling (Line~\ref{line:monte_carlo} in Algorithm~\ref{alg:npg}), we can identify if $s$ is in $\cK^n$ and compute $b^n(s,a)$ by $\hat{\Sigma}^n_{\cover}$ (for all $a$).
	To execute $\pi^t$ in state $s$, if $s \in \cK^n$, we choose an action according to $\pi_{w^t}$; If $s \notin \cK^n$, we uniformly choose an action from the actions with $b^n(s,a)=\frac{1}{1-\gamma}$.
	
	\textbf{Sample Complexity.} We note that is $N$ is the number of times we update the coverage distribution. Between two coverage updates, (i) we observe $K$ trajectories to update the feature covariance matrix, (ii) we perform $M_{\hf}$ human pairwise trajectory comparisons, and (iii) we run $T$ iterations of NPG, and within each NPG iteration, we run $M_{\subsgd}$ steps of SGD for policy evaluation. Further, for each step of SGD, we sample two trajectories, one to sample from the coverage distribution and one to estimate the $Q$-value function (Line~\ref{line:monte_carlo} in Algorithm~\ref{alg:npg}). So the overall number of trajectories used by our algorithm is $(K+2M_{\hf}+2TM_{\subsgd})N.$ Since the number of transitions observed for each trajectory is $\tilde{O}(\frac{1}{1-\gamma}),$ the number of samples used by our algorithm is $\tilde{O}((K+M_{\hf}+TM_{\subsgd})\frac{N}{1-\gamma})$.
	
	\subsection{Theoretical Guarantee of Algorithm $\algrlhf$} \label{sec:theoretical_guarantee_linear}
	
	Now we provide performance guarantees for algorithm $\algrlhf$.
	
	First, following \cite{agarwal2020pc}, we define a bounded transfer function approximation error. Let $\theta^{t}_{*}=\argmin_{\|\theta\|\leq W_{\theta}} \ex_{(s,a) \sim \rho^n_{\cover}} [ ( \phi(s,a)^\top \theta - ( Q^{\pi^t}(s,a; r+b^n) - b^n(s,a) ) )^2 ]$, and $d^{\star}_{s_{\init}}(s,a):=d^{\pi^*}_{s_{\init}}(s) \circ \unif_{\cA}(a)$.
	
	\begin{assumption}[Bounded Transfer Error]
		For any phase $n\geq0$ and iteration $t\geq0$, there exists some $\varepsilon_{\bias}>0$ which satisfies
		\begin{align}
			&\ex_{(s,a) \sim d^{\star}_{s_{\init}}} \bigg[ \Big( \phi(s,a)^\top \theta^{t}_{*} 
			\nonumber\\
			&\quad - \big( Q^{\pi^t}(s,a;r+b^n)- b^n(s,a) \big) \Big)^2 \bigg] \leq \varepsilon_{\bias} . \label{eq:def_bias}
		\end{align}
	\end{assumption}
	$\varepsilon_{\bias}$ measures the  error of using the best fit $\theta^{t}_{*}$ with log-linear policies under $\rho^{n}_{\cover}$ to predict the state-action value function under $d^{\star}_{s_{\init}}$. For tabular or linear MDPs~\cite{yang2019sample,jin2020provably}, $\theta^{t}_{*}$ perfectly fits the value function for all $(s,a)$ with log-linear policies, and $\varepsilon_{\bias}=0$. Then, we formally state the performance of $\algrlhf$.

	\begin{theorem}\label{thm:ub_pgrlhf}
		With probability at least $1-\delta$, the output policy of algorithm $\algrlhf$ satisfies
		\begin{align*}
			& V^{\pi^{*}}(s_{\init}) - V^{\pi^{\out}}(s_{\init}) \leq \tilde{O} \Bigg( \frac{\sqrt{|\cA| \varepsilon_{\bias} }}{1-\gamma} + \frac{W_{A}}{ (1-\gamma)  \sqrt{T} } 
			\nonumber\\
			& 
			\!\!\!\!\!+\!\!
			\frac{W_{Q} \sqrt{\beta N} }{ (1-\gamma) (M_{\subsgd})^{\frac{1}{4}} } \!\!+\!\! 
			\frac{  \sqrt{\beta W_{Q} N d} }{ (1-\gamma)  \sqrt{c_{\mle}} c_{\base}^{\frac{1}{4}} M_{\hf}^{\frac{1}{4}} } 
			\!\!+\!\! 
			\frac{d }{N \beta (1-\gamma)} \! \Bigg) \! .
		\end{align*}

		Furthermore, by tuning parameters as in Eq.~\eqref{eq:set_parameter_linear} in Appendix~\ref{apx:main_thm_proof_linear}, we can guarantee
		$$
		V^{\pi^{*}}(s_{\init}) - V^{\pi^{\out}}(s_{\init}) 
		\leq \varepsilon + \frac{2 \sqrt{|\cA| \varepsilon_{\bias} }}{1-\gamma} ,
		$$
		with
		$
		\tilde{O}( \poly(W_Q, W_{\mu}, \zeta_{\hf}, d, (1-\gamma)^{-1} , \varepsilon^{-1} , c_{\base}^{-1}, c_{\mle}^{-1} ) ) 
		$ samples.
		Here $W_{Q}:=\frac{2}{(1-\gamma)^2}$,  $c_{\mle}:=(2+\exp(-2W_{\tau} W_{\mu})+\exp(2W_{\tau} W_{\mu}))^{-1}$, and $W_{\tau}:=\tilde{O}(\frac{1}{1-\gamma})$ denotes the high probability bound of trajectory length.
	\end{theorem}
	
	See the full bounds in  Eqs.~\eqref{eq:apx_suboptimality} and \eqref{eq:apx_samples} in Appendix~\ref{apx:main_thm_proof_linear}.
	
	\textbf{Remark.}
	As shown in Theorem~\ref{thm:ub_pgrlhf}, the suboptimality can be decomposed into the following components: (i) the transfer function approximation error $\sqrt{\varepsilon_{\bias}}$, (ii) the NPG regret  $\frac{1}{\sqrt{T}}$, (iii) the policy evaluation error $(M_{\subsgd})^{-\frac{1}{4}}$, (iv) the reward estimation error $M_{\hf}^{-\frac{1}{4}}$, and (v) the error due to the exploration bonus construction $\frac{1}{N}$.
	The statistical error (ii)-(v) will converge to zero as the number of samples increases, while the transfer function approximation error (i) can still remain even with infinite samples. 
	
	Theorem~\ref{thm:ub_pgrlhf} demonstrates that algorithm $\algrlhf$ can efficiently utilize human feedback to learn a near-optimal policy up to the intrinsic function approximation error of the MDP.
	For tabular MDPs and linear MDPs~\cite{yang2019sample,jin2020provably}, we have  $\varepsilon_{\bias}=0$, and $\algrlhf$ can identify an $\varepsilon$-optimal policy. 
	
	Below we give a proof sketch, and introduce a novel trajectory-level elliptical potential analysis for bounding the feature vector sum of human data.
	\begin{table*}[t]
		\renewcommand\arraystretch{1.2}
		\centering
		\begin{tabular}{c|c|c}
			\hline
			& $\algrlhf$ for RLHF & PC-PG~\cite{agarwal2020pc} for Standard RL  \\
			\hline
			\# Samples & $\tilde{O}((NK+NTM_{\subsgd}+NM_{\hf})  \frac{1}{1-\gamma})$& $\tilde{O}((NK+NTM_{\subsgd}) \frac{1}{1-\gamma})$ \\
			\hline
			\# True rewards & $0$ & $\tilde{O}((NK+NTM_{\subsgd}) \frac{1}{1-\gamma})$  \\
			\hline
			\# Queries & $O(NM_{\hf})$ & $0$ \\
			\hline
		\end{tabular}
		\caption{Comparison of sample complexity, the number of true rewards and the number of queries between $\algrlhf$ and PC-PG~\cite{agarwal2020pc} for standard RL.} \label{table:comparison}
	\end{table*}		
	
	\emph{Proof Sketch.}
	For any $r:\cS\times\cA\rightarrow\R$, let $F^{r}(\theta):=\ex_{(s,a) \sim \rho^n_{\cover}} [ ( \phi(s,a)^\top \theta - ( Q^{\pi^t}(s,a;r+b^n)- b^n(s,a) ) )^2 ]$. Let $\theta^{t}_{*}$ and $\theta^{t}_{\submid}$ be the optimal solutions to minimize $F^{r}(\theta)$ and $F^{\hat{r}^n}(\theta)$, respectively. Recall that $\theta^{t}$ is a near-optimal solution to minimize $F^{\hat{r}^n}(\theta)$ obtained by SGD in our algorithm.
	Applying the performance difference lemma as in \cite{agarwal2020pc}, we can decompose the suboptimality into
	\begin{align}
		&\ \quad V^{*}(s_{\init}) - V^{\pi^t}(s_{\init}) 
		\nonumber\\
		&\leq \ex_{(s,a) \sim d^{\star}_{s_{\init}}}  \bigg[ \frac{\indicator{ s\in \cK^n }}{1-\gamma}  \Big( \underbrace{ \bar{\phi}^t(s,a)^\top \theta^{t} + \bar{b}^{n,t}(s,a) }_{ \Gamma_{\textup{NPG}} } 
		\nonumber\\&\ \quad + \underbrace{ A^{\pi^t}(s,a;r+b^n)  - \big( \bar{\phi}^t(s,a)^\top \theta^{t}_{*} + \bar{b}^{n,t}(s,a)  \big)  }_{ \Gamma_{\textup{bias}} }
		\nonumber\\&\ \quad +  \underbrace{  \bar{\phi}^t(s,a)^\top ( \theta^{t}_{*} - \theta^{t}_{\submid} ) }_{ \Gamma_{r} }
		+  \underbrace{  \bar{\phi}^t(s,a)^\top (\theta^{t}_{\submid} - \theta^{t} ) }_{ \Gamma_{\textup{SGD}} } 
		\nonumber\\
		&\quad\ +  \underbrace{\sum_{(s,a) \notin \cK^n} d^{\pi^{n+1}}_{s_{\init}}(s,a)
			\Big)}_{ \Gamma_{b} } \bigg] . \label{eq:subopt_decomposition}
	\end{align}
	Here $A^{\pi^t}(s,a;r+b^n):= Q^{\pi^t}(s,a;r+b^n) -  V^{\pi^t}(s;r+b^n)$, $\bar{\phi}^{t}(s,a) := \phi(s,a) - \ex_{a' \sim \pi^t(\cdot|s)} \mbr{ \phi(s,a') }$ and  $\bar{b}^{n,t}(s,a) := b^n(s,a) - \ex_{a' \sim \pi^t(\cdot|s)} \mbr{ b^n(s,a') }$. 
	Similar to \cite{agarwal2020pc}, we can bound  $\Gamma_{\textup{NPG}}$, $\varepsilon_{\bias}$, $\Gamma_{\textup{SGD}}$ and $\Gamma_{b}$ due to NPG regret, transfer function approximation error, policy evaluation error and optimistic bonus construction, respectively.
	
	Then, the remaining challenge is to bound the reward estimation error $\Gamma_{r}$.
	To tackle this, we develop a novel \emph{trajectory-level} elliptical potential analysis to deal with human data.
	
	\textbf{Trajectory-level Elliptical Potential Analysis.}
	According to the definitions of $\theta^{t}_{\submid}$ and $\theta^{t}_{*}$, to bound term $\Gamma_{r}$, it suffices to bound
	\begin{align}
		&\quad\ \ex_{(s,a) \sim \rho^n_{\cover}} \big[\big| Q^{\pi^t}(s,a;\hat{r}^n+b^n) - Q^{\pi^t}(s,a;r+b^n) \big|\big]
		\nonumber\\
		&\leq \! \ex_{\tau \sim \cO^{\pi^t}_{\rho^n_{\cover}}} \! \bigg[ \Big\| \!\sum_{h=0}^{H(\tau)}\! \phi(s_h,a_h) \Big\|_{(\hat{\Sigma}_{\hf}^{n})^{\!-1}} \|\hat{\mu}^{n} \!-\! \mu^*\|_{\hat{\Sigma}_{\hf}^{n}} \bigg] , \label{eq:phi_times_mu}
	\end{align}
	Here $\hat{\Sigma}_{\hf}^{n}:=\frac{1}{M_{\hf}} \sum_{i=1}^{M_{\hf}} \tilde{\phi}^{\tau^{(1)}_i,\tau^{(2)}_i} ( \tilde{\phi}^{\tau^{(1)}_i,\tau^{(2)}_i} )^\top + \frac{\zeta_{\hf}}{n} I$ is the feature covariance matrix of human data, and concentrates to
	\begin{align}
		\!\Sigma_{\hf}^{n}& \!:=\! \frac{1}{n} \! \sum_{i=1}^{n} \! \Big(  \ex_{
			\tau^{(1)} \sim \cO^{\pi^i}_{\rho^{i-1}_{\cover}} \!, \tau^{(2)} \sim \cO^{\pi^{\base}}_{s_{\init}} } \big[ \tilde{\phi}^{\tau^{(1)} \!\!,\tau^{(2)}\!\!}  ( \tilde{\phi}^{\tau^{(1)} \!\!,\tau^{(2)}} \!)^{\!\!\top} \big] \!\Big) 
		\nonumber\\
		&\quad\ + \frac{\zeta_{\hf}}{n}  I , \label{eq:Sigma_hf}
	\end{align}
	for any $n\geq1$. Let $\hat{\Sigma}_{\hf}^{0}=\Sigma_{\hf}^{0}:=\zeta_{\hf} I$.
	In addition, $\tilde{\phi}^{\tau^{(1)}_i, \tau^{(2)}_i}:=\sum_{h=0}^{H(\tau^{(1)}_i)}\phi(s^{(1)}_{i,h},a^{(1)}_{i,h}) - \sum_{h=0}^{H(\tau^{(2)}_i)}\phi(s^{(2)}_{i,h},a^{(2)}_{i,h})$.
	
	Eq.~\eqref{eq:phi_times_mu} is a \emph{key} step. Specifically, we decompose the error of state-action value function due to reward estimation into: (i) The error of reward parameter $\|\hat{\mu}^{n} - \mu^*\|_{\hat{\Sigma}_{\hf}^{n}}$, which is bounded by $\tilde{O}(\frac{1}{\sqrt{M_{\hf}}})$ due to the MLE guarantee; 
	(ii) The trajectory-level feature norm $\| \sum_{h=0}^{H(\tau)} \phi(s_h,a_h) \|_{(\hat{\Sigma}_{\hf}^{n})^{-1}}$, \emph{instead of} the state-action-level feature norm $\sum_{(s,a)\sim d^{\pi^t}_{\rho_{\cover}^n}} \|\phi(s,a) \|_{(\hat{\Sigma}_{\hf}^{n})^{-1}}$. 
	
	
	Since $\pi^{\out}$ is the average of all obtained policies and $\hat{\Sigma}_{\hf}^{n}$ concentrates to $\Sigma_{\hf}^{n}$, with the Cauchy-Schwarz inequality, it suffices to bound the summation of the squared feature norm under $\Sigma_{\hf}^{n}$ as
	\begin{align}
		\frac{1}{NT} \sum_{n=0}^{N-1} \sum_{t=0}^{T-1} \ex_{\tau \sim \cO^{\pi^t}_{\rho^n_{\cover}}}  \bigg[ \Big\| \sum_{h=0}^{H(\tau)} \phi(s_h,a_h) \Big\|^2_{\sbr{\Sigma_{\hf}^{n}}^{-1}} \bigg] . \label{eq:sum_sqaured_feature_norm}
	\end{align}
	
	A \emph{nice} thing is that the covariance matrix $\Sigma_{\hf}^{n}$ (Eq.~\eqref{eq:Sigma_hf}) involves trajectory-level features, and here each summed term $ \| \sum_{h=0}^{H(\tau)} \phi(s_h,a_h) \|$ is also a trajectory-wise feature norm. This enables us to apply the elliptical potential lemma~\cite{abbasi2011improved} to bound this summation, which validates our decomposition scheme in Eq.~\eqref{eq:phi_times_mu}.
	
	Then, using $\pi^{n+1}=\unif(\{\pi^{t}\}_{t=0}^{T-1})$, Eq.~\eqref{eq:sum_sqaured_feature_norm} is bounded by
	\begin{align*}
		&\frac{1}{N} \sum_{n=0}^{N-1} \ex_{\tau^{(1)} \sim \cO^{\pi^{n+1}}_{\rho^n_{\cover}}}  \bigg[ \Big\| \sum_{h=0}^{H(\tau^{(1)})} \phi(s_h^{(1)},a_h^{(1)})^{\top} \Big\|^2_{\sbr{\Sigma_{\hf}^{n}}^{-1}} \bigg] 
		\\
		&\overset{\textup{(a)}}{\leq} 2 \underbrace{\sum_{n=0}^{N-1}  \ex_{\begin{subarray}{l} \tau^{(1)} \sim \cO^{\pi^{n+1}}_{\rho^n_{\cover}} \\ \tau^{(2)} \sim \cO^{\pi^{\base}}_{s_{\init}} \end{subarray}}  \bigg[ \Big\| \tilde{\phi}^{\tau^{(1)},\tau^{(2)}} \Big\|^2_{\sbr{n \Sigma_{\hf}^{n}}^{-1}} \bigg]}_{\Gamma_{\textup{traj}}}
		\\
		& \quad + 2 \sum_{n=0}^{N-1} \ex_{\tau^{(2)} \sim \cO^{\pi^{\base}}_{s_{\init}}}  \bigg[ \Big\| \!\!\!\!\sum_{h=0}^{H(\tau^{(2)})} \phi(s_h^{(2)},a_h^{(2)})^{\top} \Big\|^2_{\sbr{n \Sigma_{\hf}^{n}}^{\!-1}} \! \bigg] 
		\\
		&\overset{\textup{(b)}}{=} O \sbr{  d\log \Big( 1+\frac{NW_{\tau}^2}{d\zeta_{\hf}} \Big) + \frac{d}{c_{\base}}\log ( N ) } .
	\end{align*}
	Here we make the convention that $(0 \Sigma_{\hf}^{0}):=\zeta_{\hf} I$.  Inequality (a) comes from adding and subtracting $ \sum_{h=0}^{H(\tau^{(2)})} \phi(s_h^{(2)},a_h^{(2)})^{\top}$.
	
	With consistency between the summed term and the covariance matrix $\Sigma_{\hf}^n$ (both in a trajectory and difference form), $\Gamma_{\textup{traj}}$ is an effective elliptical potential summation.
	Then, inequality (b) follows from applying the elliptical potential lemma~\cite{abbasi2011improved} and Assumption~\ref{assumption:baseline_policy}.
	See Lemmas~\ref{lemma:Q_decomposition_traj}, \ref{lemma:sum_matrix_norm} in Appendix~\ref{apx:human_feedback} for full proofs.\hfill $\Box$
	
	
	


	\subsection{Insight into the Practical Efficiency of RLHF} \label{sec:comparison_rlhf_standard_rl}
	
	Below we compare our $\algrlhf$ and prior standard RL algorithm PC-PG~\cite{agarwal2020pc}, and provide an insight behind the empirical success of RLHF.
	
	Table~\ref{table:comparison} shows that $\algrlhf$ needs additional $\tilde{O}( \frac{NM_{\hf}}{1-\gamma})$ samples due to the lack of direct reward signals. We have $O(M_{\hf})\approx O(M_{\subsgd})$, since their convergence rates are the same (see Theorem~\ref{thm:ub_pgrlhf}). Then, the additional samples needed by $\algrlhf$ is negligible compared to the total sample complexity. This implies that RLHF does not introduce much hardness in terms of sample complexity, which matches the finding of recent RLHF work~\cite{wang2023rlhf}.

	Regarding the cost on reward observations, in standard RL, we require $\tilde{O}((NK+NTM_{\subsgd}) \frac{1}{1-\gamma})$ observations of true rewards.
	However, in RLHF, we do not need any observation of true rewards, but only use $O(NM_{\hf})$ human queries. 
	The ratio of the number of queries needed to the total sample complexity is about  $\frac{NM_{\hf}}{NTM_{\subsgd}}=\frac{1}{T}$. This theoretically explains the empirical success of RLHF --- RLHF only needs a small amount of comparison queries to achieve good performance as standard RL~\cite{christiano2017deep}.
	
	From the perspective of improving RLHF practice, our results provide two insights.
	Policy optimization can consist of three phases: sampling for exploration, policy evaluation and policy improvement. One of our insights is that, inserting reward model learning before multiple iterations of policy evaluation and improvement is efficient, i.e., we get policies that are nearly as good as the case where the rewards are known, while using only a small amount of human data compared to the overall sample complexity. Another insight is that querying human feedback on the state-action space which is rarely visited or is more likely induced by the optimal policy, helps improve the exploration of RLHF algorithms in reward estimation.
	
	
	\gal{could we have gained something from performing reward estimation more than once (log / sqrt / else) to trade off more reward estimation estimation with the policy evaluation error (which we absorb in each iteration)?}\yihan{this is an interesting question. my current thought is: the policy evaluation error is standard, which is usually multiplied by $T$. it seems there is no need to perform more updates during the process of policy gradient ascent. but there can more things to explore}

	\begin{figure*}[t]
		\centering     
		\subfigure[$T=50$] {
			\begin{minipage}[t]{0.31\linewidth}
				\includegraphics[width=1\textwidth]{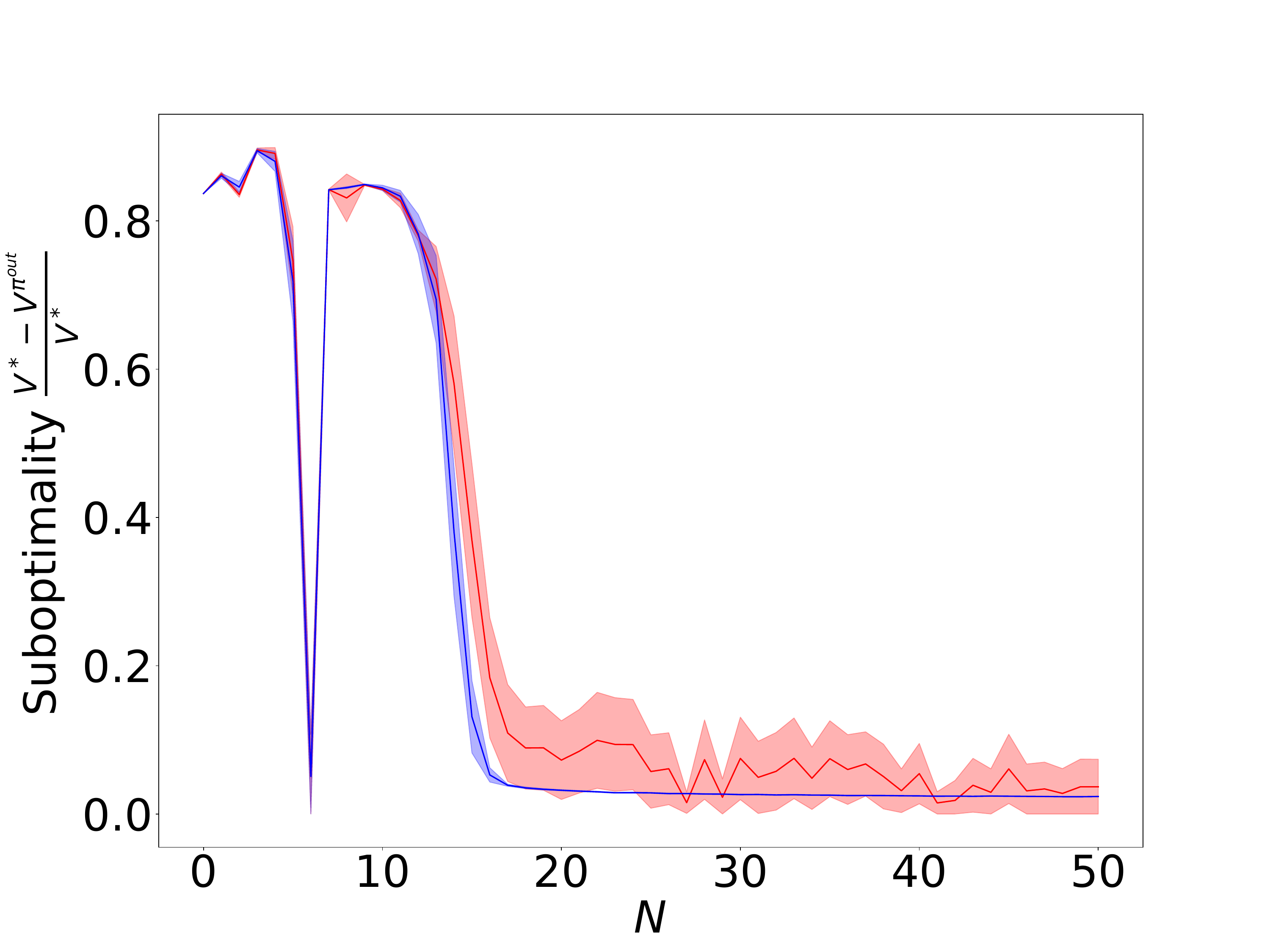}
				\hspace*{-1.6em}
				\includegraphics[width=1.1\textwidth]{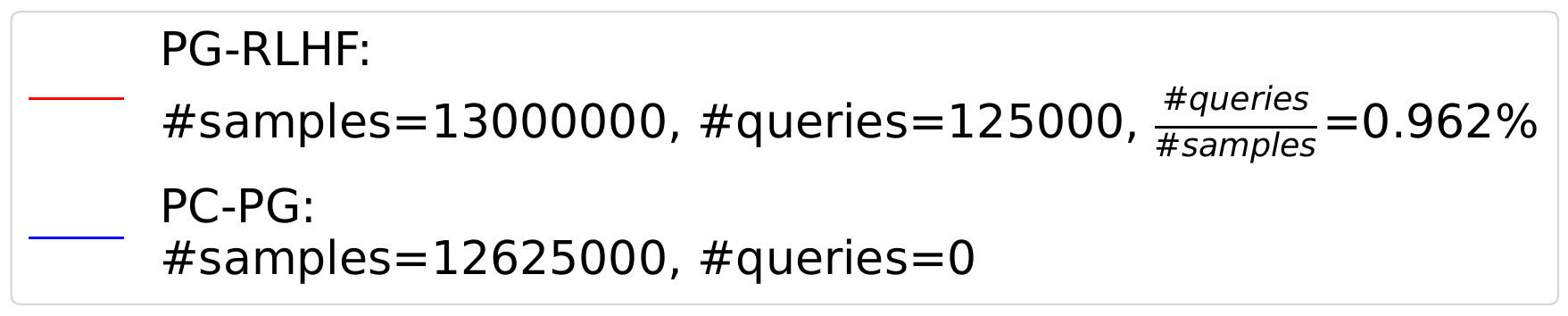} 
				\vspace*{0.1em}
			\end{minipage}
		}   
		\subfigure[$T=100$] {
			\begin{minipage}[t]{0.31\linewidth}
				\includegraphics[width=1\textwidth]{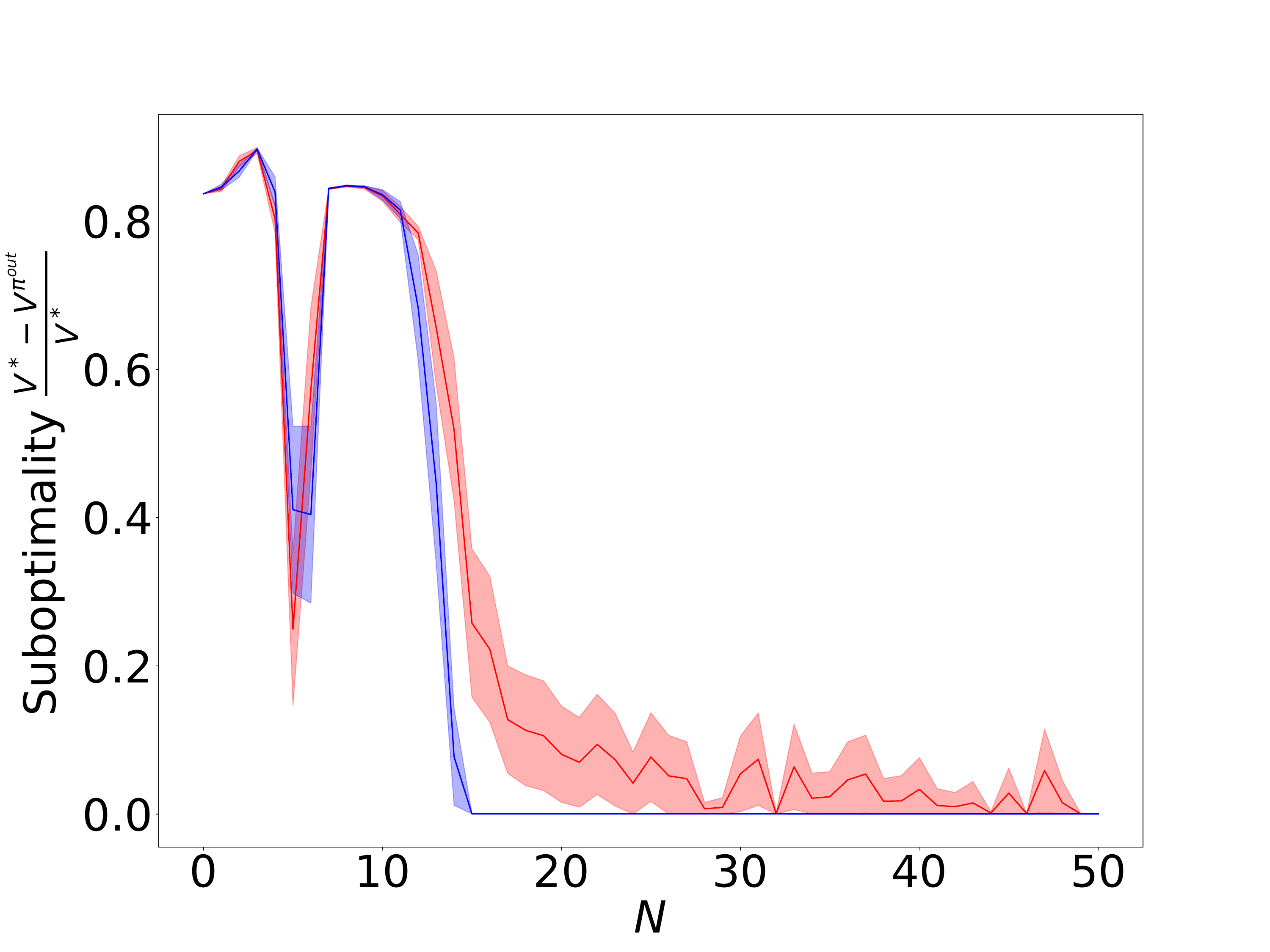}
				\hspace*{-0.93em}
				\includegraphics[width=1.1\textwidth]{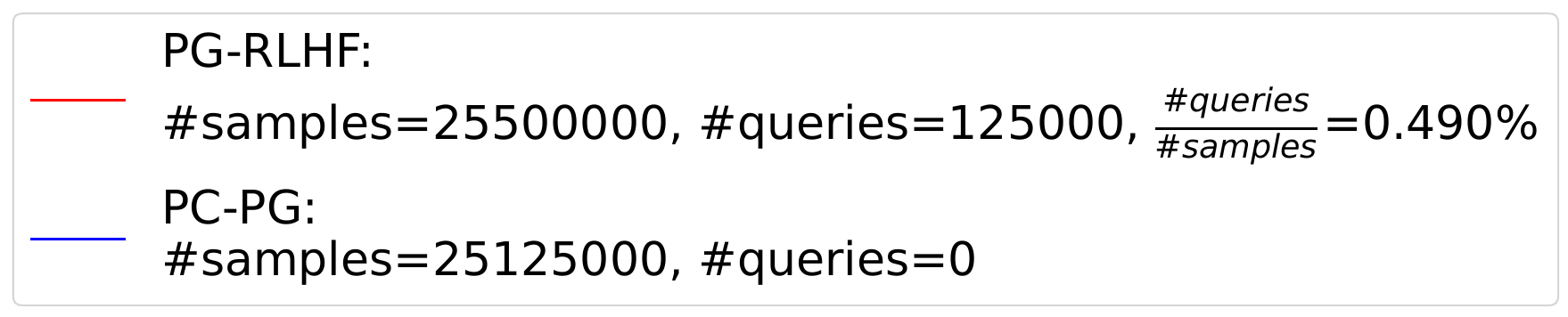} 
				\vspace*{0.1em}
			\end{minipage}
		}   
		\subfigure[$T=200$] {
			\begin{minipage}[t]{0.31\linewidth}
				\includegraphics[width=1\textwidth]{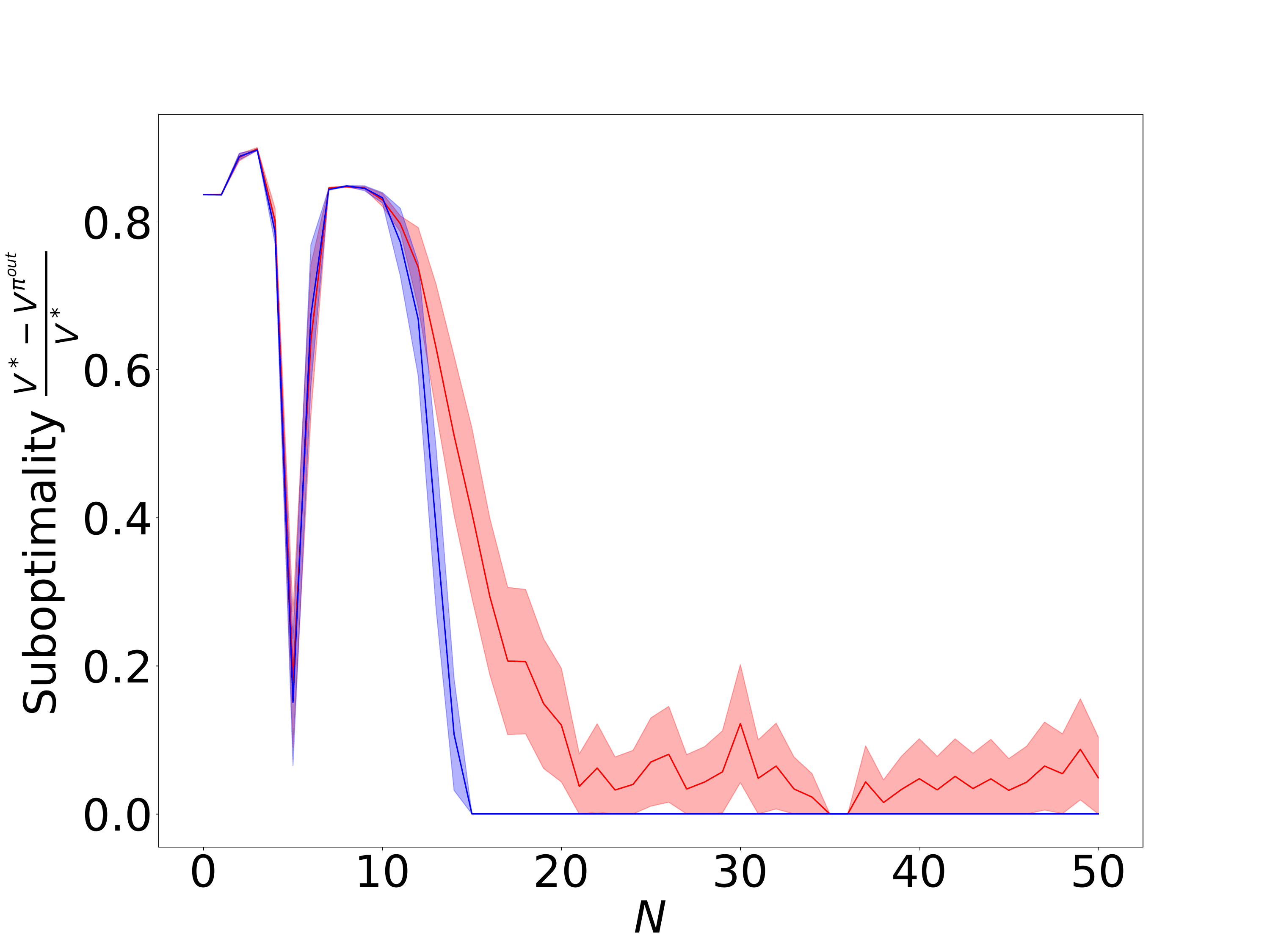}
				\includegraphics[width=1.1\textwidth]{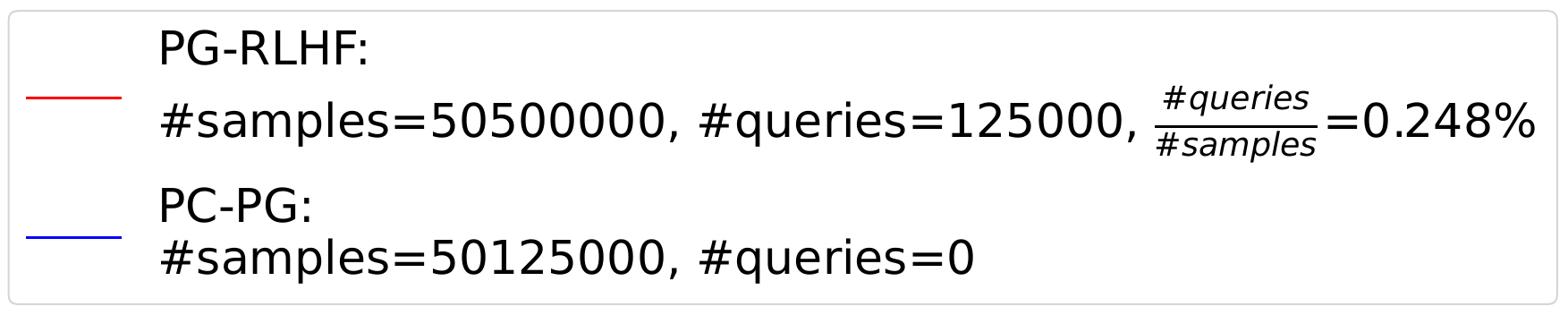} 
				\vspace*{0.1em}
			\end{minipage} 
		}  
		\caption{Experimental results of algorithms $\algrlhf$ and PC-PG.}\label{fig:experiment}
	\end{figure*}

	\section{PO-RLHF with Neural Function Approximation}
	
	In this section, we turn to the neural setting.
	We design an efficient algorithm $\algnnrlhf$, and derive a biased MLE guarantee with neural approximation in analysis.

	\subsection{Algorithm $\algnnrlhf$}
	
	A detailed description and pseudo-code are provided in Appendix~\ref{apx:detailed_alg_nnrlhf}. Here we provide a brief outline of the algorithm.
	$\algnnrlhf$ actively collects human data as exploration, learns a reward network with human data, and trains a policy network and a Q-network to optimize the policy. 
	Similar to $\algrlhf$, $\algnnrlhf$ estimates the feature covariance matrix and updates the coverage with the neural feature $\psi_{w^0}(s,a)$. 
	Then, it generates preference data by past coverage, past policies and the baseline policy. The reward and Q-function networks are trained using an MLE loss function and a least-squares loss function, respectively.

	Now we provides theoretical guarantees on $\algnnrlhf$.
	Let $\theta^{\nn,t}_{*}=
	\argmin_{\|\theta-\theta^0\|\leq R} \ \ex_{(s,a) \sim \rho^n_{\cover}} [ ( \psi_{w^0}(s,a)^\top \theta  - ( Q^{\pi^t}(s,a; r+b^n) - b^n(s,a) ) )^2 ] 
	$
	denote the optimal solution to the approximated version of the Q-network training objective with neural feature $\psi_{w^0}(s,a)$.
	
	Similar to Eq.~\eqref{eq:def_bias}, we assume that the error of using the best fit $\theta^{\nn,t}_{*}$
	under $\rho_{\cover}^n$ 
	to predict the state-action value function under $d^{\star}_{s_{\init}}$ is bounded.
	
	\begin{assumption}[Bounded Neural Transfer Error]
		For any phase $n\geq0$ and iteration $t\geq0$, there exists some $\varepsilon^{\nn}_{\bias}>0$ which satisfies
		\begin{align*}
			& \ex_{(s,a) \sim d^{\star}_{s_{\init}}} \bigg[ \Big( \psi_{w^0}(s,a)^\top \theta^{\nn,t}_{*} 
			\\
			& \qquad - \big( Q^{\pi^t}(s,a;r+b^n)- b^n(s,a) \big) \Big)^2 \bigg] \leq \varepsilon^{\nn}_{\bias} .
		\end{align*} 
	\end{assumption}

	\begin{theorem}\label{thm:nn_ub_pgrlhf}
		With probability at least $1-\delta$, the output policy of algorithm $\algnnrlhf$ satisfies
		\begin{align*}
			& V^{\pi^{*}}(s_{\init}) - V^{\pi^{\out}}(s_{\init}) \leq
			\frac{2 \sqrt{|\cA| \varepsilon^{\nn}_{\bias} }}{1-\gamma} 
			\nonumber\\
			&
			+ \tilde{O} \Bigg( \frac{W^{\nn}}{ (1-\gamma)  \sqrt{T} } 
			+ \frac{ \sqrt{\beta N R W^{\nn} } }{ (1-\gamma) (M^{\theta}_{\subsgd})^{\frac{1}{4}} } 
			\nonumber\\
			&
			+\! \frac{  m^{\frac{1}{4}} d^{\frac{1}{4}}  \sqrt{ \beta W^{\nn} } }{ c_{\base}^{\frac{1}{4}} (1-\gamma)} 
			\!\cdot\!
			\sbr{ \frac{ m^{\frac{1}{4}} d^{\frac{1}{4}} \sqrt{N}    }{  \sqrt{c^{\nn}_{\mle}} M_{\hf}^{\frac{1}{4}}  } 
				\!+\! 
				\frac{ W_{\tau}^{\frac{1}{4}} R^{\frac{1}{4}} \sqrt{N}    }{  (c^{\nn}_{\mle})^{\frac{1}{4}} (M^{\mu}_{\subsgd})^{\frac{1}{8}}  }  }
			\nonumber\\
			&
			+ \frac{md }{ (1-\gamma) N \beta} 
			+  B\sbr{\frac{1}{m^{\frac{1}{16}}}} \Bigg) . 
		\end{align*}
		Here $M^{\mu}_{\subsgd}$ and $M^{\theta}_{\subsgd}$ are the numbers of iterations of the SGD for the reward network and Q-network training, respectively. $B(m^{-\frac{1}{16}})$ is a neural approximation error term scaling as $m^{-\frac{1}{16}}$. $W^{\nn}:=\sqrt{m} \bar{c} + R$, and $c^{\nn}_{\mle}:= (2+\exp(-2W_{\tau}W^{\nn})+\exp(2W_{\tau}W^{\nn}))^{-1}$.
		
	\end{theorem}

	Theorem~\ref{thm:nn_ub_pgrlhf} demonstrates that the suboptimality becomes small with sufficiently large $T$, $N$, $M_{\hf}$, $M_{\subsgd}^{\theta}$ and $M_{\subsgd}^{\mu}$, up to the neural transfer error $O((\varepsilon^{\nn}_{\bias})^{\frac{1}{2}})$ and the neural approximation error $\tilde{O}(m^{-\frac{1}{16}})$.
	See the full bound in Eq.~\eqref{eq:nn_suboptimality} in Appendix~\ref{apx:nn_main_thm_proof}.
	
	\textbf{Biased Neural MLE Analysis.}
	Due to the gap between the true reward $r$ and the functions that $h(s,a;\mu)$ can represent, our MLE reward training is biased.
	To tackle this difficulty, we develop a novel biased MLE analysis with neural approximation.
	
	Specifically, let $\mu^{\supproj}_{r}$ be the network parameter of the projection of $r$ onto neural function class $\{\psi_{\mu^0}(s,a)^\top \mu\}$. Then, we have that $\psi_{\mu^0}(s,a)^\top \mu^{\supproj}_{r}$ is close to $r$ up to a neural approximation error scaling as $\frac{1}{m}$.
	Let $\mu^{n}_{\mle}$ be the optimal solution to the approximated version of the MLE objective  with feature $\psi_{\mu^0}(s,a)$. 
	Note that the human data are generated \emph{almost} according to $\mu^{\supproj}_{r}$ (since it is close to $r$), and $\mu^{n}_{\mle}$ has a larger likelihood than $\mu^{\supproj}_{r}$. Utilizing these two facts, we can bound $\|\mu^{n}_{\mle}-\mu^{\supproj}_{r}\|$ up to the standard MLE error $\tilde{O}(\frac{1}{\sqrt{M_{\hf}}})$ and a neural approximation error. 
	Furthermore, the SGD result $\hat{\mu}^n$ obtained in our algorithm is close to the MLE optimal solution $\mu^{n}_{\mle}$ up to the SGD error. Combining the SGD, MLE and neural approximation error, we can bound $\|\hat{\mu}^n-\mu^{\supproj}_{r}\|$. We refer interested readers to Lemma~\ref{lemma:nn_mle} in Appendix~\ref{apx:nn_human_feedback}.

	\section{Experiments}

	In this section, we present experiments to demonstrate the practical efficacy of our algorithm and validate our theoretical results.
	
	Following the experimental setup of existing algorithm PC-PG~\cite{agarwal2020pc}, we evaluate algorithms in an RL environment called Bidirectional Lock, which was also used in other prior works, e.g., \cite{zhang2021made}. The details of this environment are deferred to Appendix~\ref{apx:experiment_env}. 
	
	In our experiments, $S=22$, $A=5$, $\gamma=0.9$, $\delta=0.005$, $\eta=0.3$, $N=30$, $T\in\{50,100,200\}$, $K=2500$, $M_{\subsgd}=2500$ and $M_{\hf}=2500$. The feature vectors $\phi(s,a)$ are one-hot vectors of state-actions, and $d=110$.
	We compare our algorithm $\algrlhf$ with the standard RL algorithm PC-PG~\cite{agarwal2020pc}. Each algorithm is performed for $50$ independent runs.
	Figure~\ref{fig:experiment} plots the normalized suboptimalities of output policies $\frac{V^*-V^{\pi^{\out}}}{V^*}$ with $95\%$ confidence intervals, and reports the sample complexities and query complexities in the legend. 
	(Since the numbers of samples and queries are computed before performing policy optimization and reward learning in algorithms $\algrlhf$ and PC-PG, the sample complexity and query complexity are the same for all runs.)
	
	From Figure~\ref{fig:experiment}, we see that $\algrlhf$ effectively learns the optimal policy without observing true rewards, and achieves comparable performance to PC-PG while using a few more samples and a small amount of preference queries. When the number of iterations in policy optimization $T$ increases, the ratio of query complexity to the overall sample complexity (scaling as $\frac{1}{T}$) decreases, which matches our theoretical results.

	\section{Conclusion}
	
	In this work, we study exploration-driven policy optimization for RLHF. For the linear and neural function approximation settings, we propose efficient algorithms with active human data collection. Through the comparison of results between RLHF and standard RL, we give a theoretical explanation for the query efficiency of RLHF. There is still a large space for future investigation. For example, it is interesting to explore other potential reasons behind the success of RLHF, e.g., the structural advantage of preference feedback over numerical feedback. \anna{An interesting direction of future research involves effectively utilizing the Monte Carlo simulations and making queries based on the trajectories generated by the Monte Carlo simulations to further refine the estimates of the reward function at each iteration.}
	
	\section*{Acknowledgement}
	The work of Yihan Du, Anna Winnicki and R. Srikant is supported in part by AFOSR Grant FA9550-24-1-0002, ONR Grant N00014-19-1-2566, and NSF Grants CNS 23-12714, CNS 21-06801, CCF 19-34986, and CCF 22-07547.
	
	\section*{Impact Statement}
	
	This work proposes efficient policy gradient RLHF algorithms with sample complexity guarantees, and provides a theoretical insight for the empirical success of RLHF.
	We believe that this work may have potential societal impacts on RLHF applications, but as a theoretical work, it does not involve ethical concerns.

	\bibliographystyle{icml2024}
	\bibliography{icml24_PG_RLHF_ref}

	
	\newpage
\appendix
\onecolumn

\section{Details of the Experimental Environment}\label{apx:experiment_env}

\begin{figure}[t]
	\centering   
	\vspace*{0.8em}
	\includegraphics[width=0.8\columnwidth]{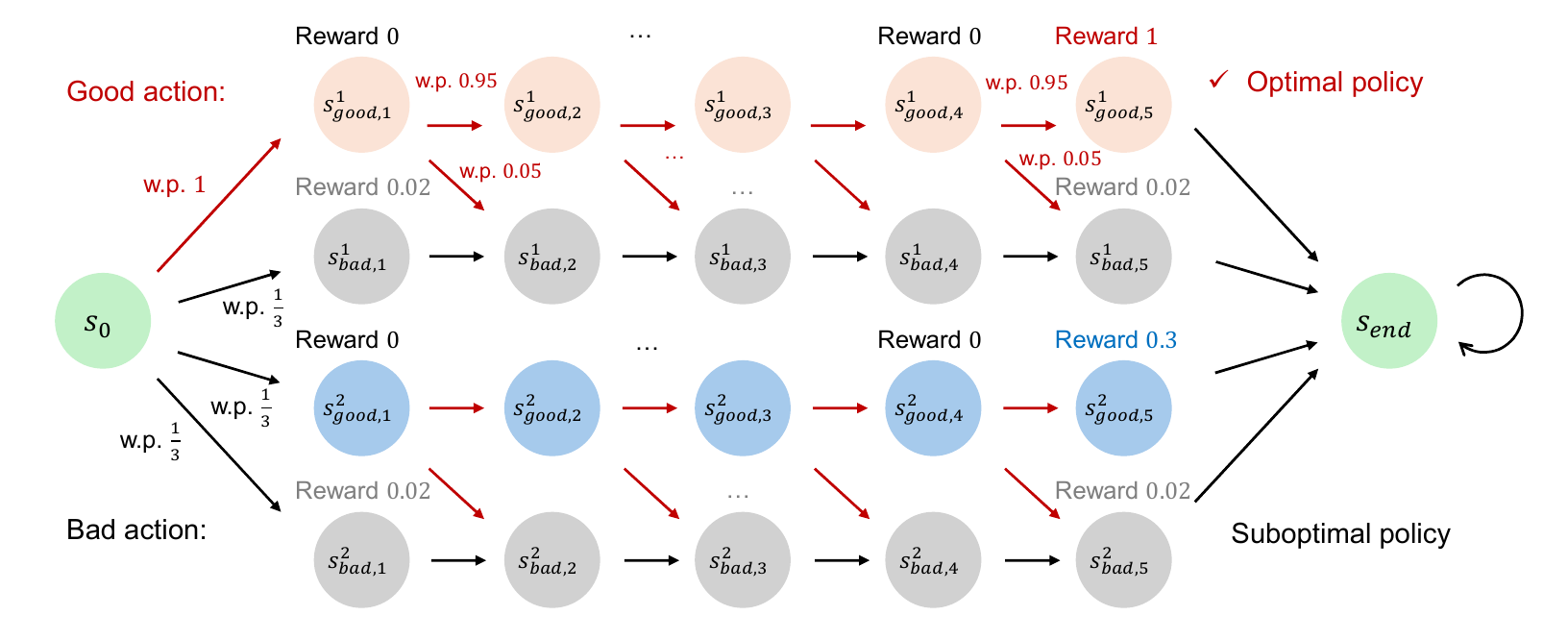}
	\caption{The Bidirectional Lock environment.} \label{fig:bidirectional_lock}
\end{figure} 

In this section, we describe the  Bidirectional Lock environment used in our experiments. 

As shown in Figure~\ref{fig:bidirectional_lock}, there are two locks, and each of them has $2H$ states. These states are denoted by $s^{\ell}_{good,h}$ and $s^{\ell}_{bad,h}$, were $\ell \in \{1,2\}$ is the index of the lock, and $h\in [H]$ is the horizon of the lock. In addition to these $2H$ states, there are an initial state $s_0$ and an absorbing ending state $s_{\bot}$. There are $5$ actions, including a good action $a_{good}$ and $4$ bad actions.
The reward function depends only on states. Only the last good states of locks give high rewards, i.e., $r(s^{1}_{good,H},\cdot)=1$ and $r(s^{2}_{good,H},\cdot)=0.3$. 
Other good states of locks $\{s^{\ell}_{good,h}\}_{\ell\in\{1,2\}, h\in [H-1]}$, $s_0$ and $s_{\bot}$ induce zero reward.
The bad states of locks give tiny rewards, i.e., $r(s^{\ell}_{bad,h},\cdot)=\frac{0.1}{H}$ for any $\ell \in \{1,2\}$ and $h \in [H]$. We set $H=5$ in our experiments.

The agent starts from $s_0$. Under $a_{good}$, she transitions to $s^{1}_{good,1}$ deterministically; Under other actions, she transitions to $s^{1}_{bad,1}$, $s^{2}_{good,1}$ and $s^{2}_{bad,1}$ all with probability $\frac{1}{3}$. 
For any $\ell \in \{1,2\}$ and $h\in[H-1]$, in state $s^{\ell}_{good,h}$, under $a_{good}$, the agent transitions to $s^{\ell}_{good,h+1}$ and $s^{\ell}_{bad,h+1}$ with probabilities $0.95$ and $0.05$, respectively; Under other actions, she transitions to $s^{\ell}_{bad,h+1}$ deterministically. For any $\ell \in \{1,2\}$ and $h\in[H-1]$, in state  $s^{\ell}_{bad,h}$, the agent transitions to $s^{\ell}_{bad,h+1}$ deterministically under any action. Once the agent achieves $s^{\ell}_{good,H}$ or $s^{\ell}_{bad,H}$ for any $\ell \in \{1,2\}$, she transitions to $a_{\bot}$ deterministically.

This environment has sparse rewards. The optimal policy is to always take $a_{good}$, and obtain the high final reward of lock 1. A suboptimal (myopic) policy results in getting stuck in bad states or only obtaining the low final reward of lock 2.

\section{Detailed Review of Related Works} \label{apx:related_works}

In the following, we present a more detailed review of related works.

\textbf{RLHF.}
RLHF~\cite{christiano2017deep,kaufmann2023survey} has gained a huge empirical success, especially in LLMs~\cite{ouyang2022training,achiam2023gpt}. Recently, a number of works have emerged to theoretically analyze RLHF. \citet{xu2020preference,novoseller2020dueling,pacchiano2021dueling} study online RLHF for tabular MDPs. \citet{chen2022human,wang2023rlhf} consider online RLHF with general function approximation. \citet{wang2023rlhf} design a reduction framework for RLHF, and prove that the sample complexity for RLHF is no higher than that for standard RL. \citet{zhu2023principled,zhan2023provable,li2023reinforcement} study offline RLHF with function approximation.
\citet{xiong2023gibbs} introduce a KL-constrained framework for RLHF, and \citet{zhan2023query,wu2023making} consider how to optimize query complexity via experimental design and posterior sampling. \citet{ji2023provable} also seek to understand the empirical success of RLHF in the offline contextual bandit setting, but different from our work, \citet{ji2023provable} explain it from the perspective of intrinsic human data bias. 

In contrast to the above works which most consider value-based algorithms, we analyze policy gradient RLHF algorithms with exploration, and theoretically explain why RLHF only needs a small amount of human feedback to attain good performance, from the perspective of the efficiency of RLHF (reward learning) algorithmic procedure itself.

\textbf{RL with Neural Function Approximation.}
There have been several theoretical RL works, e.g.,~\cite{cai2019neural,wang2019neural,liu2019neural,fan2020theoretical,xu2021crpo}, use neural networks to approximate value functions and policies, and provide guarantees based on existing analysis for overparameterized neural networks~\cite{jacot2018neural,arora2019fine}. Our work also considers neural function approximation for the RLHF environment.
In addition, our work is also related to \cite{agarwal2020pc}, which designs a policy gradient algorithm enabling exploration for standard RL.

\section{Detailed Description of Algorithm $\algnnrlhf$} \label{apx:detailed_alg_nnrlhf}

\begin{algorithm}[t]
	\caption{$\algnnrlhf$}
	\label{alg:neural_pg_rlhf}
	\begin{algorithmic}[1]
		\STATE {\bfseries Input:} $\varepsilon, \delta, N, K, M_{\hf}, \zeta_{\hf}, \zeta_{\cover}, \pi^{\base}, \pi^0$.
		\STATE Initialize $\alpha^0=1$ and $[\mu^0]_{\ell}, [w^0]_{\ell}\sim\cD_{\init}$, $\forall \ell \in [m]$. 
		\FOR{$n=0,\dots,N-1$}
		\STATE Sample $\{s_i,a_i\}_{i=1}^{K} \sim d^{\pi^n}_{s_{\init}}$, and $\hat{\Sigma}^n \leftarrow \frac{1}{K} \sum_{i=1}^{K} \psi_{w^0}(s_i,a_i) \psi_{w^0}(s_i,a_i)^\top$ \label{line:nn_sample_cov}
		\STATE $\hat{\Sigma}^n_{\cover} \leftarrow \sum_{i=0}^{n} \hat{\Sigma}^i + \zeta_{\cover} I$
		\STATE Let $\rho^{n}_{\cover}:=\frac{1}{n+1}\sum_{i=0}^{n} d^{\pi^{i}}_{s_{\init}}$
		\STATE $\cO^n_{\hf}:=\frac{1}{n}\sum_{i=1}^{n} \cO^{\pi^{i}}_{\rho^{i-1}_{\cover}}, \forall n\geq1$, and $\cO^0_{\hf}:= \cO^{\pi^{0}}_{s_{\init}}$
		\FOR{$i=1,\dots,M_{\hf}$}
		\STATE Sample trajectories $\tau^{(1)}_i \!\sim\! \cO^n_{\hf}$ and $\tau^{(2)}_i \!\sim\! \cO^{\pi^{\base}}_{s_{\init}}$
		\STATE Observe the preference outcome $y^r_i$
		\ENDFOR
		\STATE Train the reward network $h(s,a;\mu^0)$ with the MLE objective Eq.~\eqref{eq:apx_nn_mle_mu_r} by projected SGD, and obtain $\hat{\mu}^n$ \label{line:nn_train_reward}
		\STATE $\pi^{n+1} \leftarrow \algnnnpg(\rho^n_{\cover}, \hat{\Sigma}^n_{\cover}, \hat{\mu}^{n})$
		\ENDFOR
		\STATE {\bfseries return} $\unif(\pi^{1},\dots,\pi^{N})$
	\end{algorithmic}
\end{algorithm}

\begin{algorithm}[t]
	\caption{$\algnnnpg$}
	\label{alg:neural_npg}
	\begin{algorithmic}[1]
		\STATE {\bfseries Input:} $\rho^n_{\cover}, \hat{\Sigma}^n_{\cover}, \hat{\mu}^{n},\eta,T,\beta,\alpha^0,w^0.$ 
		\STATE Let $\hat{r}^n(\cdot,\cdot):=h(\cdot,\cdot;\hat{\mu}^n)$
		\STATE  $b^n(\cdot,\cdot):=\frac{1}{1-\gamma} \mathbbm{1}\{ \psi_{w^0}(\cdot,\cdot)^\top (\hat{\Sigma}^n_{\cover})^{-1} \psi_{w^0}(\cdot,\cdot) \geq \beta\}$
		\STATE Let $\cK^n:=\{s \in \cS: \forall a \in \cA,\ b^n(s,a)=0 \}$
		\STATE For $s \in \cK^n$, $\pi^0(\cdot|s):=\pi_{\alpha^0,w^0}$. For $s \notin \cK^n$, $\pi^0(\cdot|s):=\unif(\{a \in \cA: b^n(s,a)=\frac{1}{1-\gamma}\})$
		\FOR{$t=0,\dots,T-1$}
		\STATE $\theta^{t,0} \leftarrow w^0$
		\STATE Train the Q-network $f(s,a;\theta^{t,0})$ with the objective Eq.~\eqref{eq:apx_nn_objective_npg} by projected SGD, and obtain $\theta^t$ \label{line:nn_train_Q_network}
		\STATE Update  policy network: $\alpha^{t+1} w^{t+1} \leftarrow \alpha^{t} w^t + \eta \theta^{t}$ \label{line:nn_update_policy_network}
		\STATE $\forall s \in \cK^n$,  $\pi^{t+1}(\cdot|s) = \pi_{\alpha^{t+1},w^{t+1}}(\cdot|s) \propto  \exp( \alpha^{t+1} f(s,a;w^{t+1}) ) $. $\forall s \notin \cK^n$,  $\pi^{t+1}(\cdot|s) = \pi^0(\cdot|s)$
		\ENDFOR
		\STATE {\bfseries return} $\unif(\pi^0,\dots,\pi^{T-1})$
	\end{algorithmic}
\end{algorithm}

In this section, we present the pseudo-code of algorithm $\algnnrlhf$, and give a more detailed algorithm description.


Algorithm~\ref{alg:neural_pg_rlhf} illustrates the procedure of $\algnnrlhf$.
Similar to $\algrlhf$, in each phase $n$, $\algnnrlhf$ first estimates the feature covariance matrix $\hat{\Sigma}_{\hf}^n$ and updates the coverage  distribution $\rho^{n}_{\cover}$. Then, it generates $M_{\hf}$ pairs of preference data using past coverage distributions $\rho^{i-1}_{\cover}$, past policies $\pi^{i}$ and a baseline policy $\pi^{\base}$  ($i=0,1,\dots,n$). 
With these data, $\algrlhf$ trains the reward network $h(s,a;\mu^0)$ to minimize the following MLE objective by projected SGD (Line~\ref{line:nn_train_reward}):
\begin{align}
	& \argmin_{\|\mu-\mu^0\|_2 \leq R} \bigg( - \sum_{i=1}^{M_{\hf}} \log \Big( \frac{ \indicator{y_i=1} }{ 1+ \exp\big(  -\tilde{h}(\tau^{(1)}_i, \tau^{(2)}_i;\mu)  \big) } 
	+ \frac{ \indicator{y_i=0} }{ 1+ \exp\big(  \tilde{h}(\tau^{(1)}_i, \tau^{(2)}_i;\mu) \big) } \Big) \bigg) , \label{eq:apx_nn_mle_mu_r}
\end{align}
where $\tilde{h}(\tau^{(1)}_i, \tau^{(2)}_i;\mu) := \sum_{h=0}^{H(\tau^{(1)}_i)}h(s^{(1)}_{i,h},a^{(1)}_{i,h};\mu) - \sum_{h=0}^{H(\tau^{(2)}_i)}h(s^{(2)}_{i,h},a^{(2)}_{i,h};\mu)$.
After training, we call a subroutine $\algnnnpg$ (Algorithm~\ref{alg:neural_npg}) with $\rho^n_{\cover}$, $\hat{\Sigma}_{\cover}$ and $h(\cdot,\cdot;\hat{\mu}^n)$ to perform policy optimization.

In $\algnnnpg$, 
we train the Q-network $f(s,a;\theta^{t,0})$ to fit the state-action value function with initial distribution $\rho^n_{\cover}$ by project SGD (Line~\ref{line:nn_train_Q_network}):
\begin{align}
	&\argmin_{\|\theta-\theta^0\|\leq R} \ \ex_{(s,a) \sim \rho^n_{\cover}} \bigg[ \Big( f(s,a;\theta) 
	- \big( Q^{\pi^t}(s,a; \hat{r}^n+b^n) - b^n(s,a) \big) \Big)^2 \bigg] . \label{eq:apx_nn_objective_npg}
\end{align}
With the trained Q-network $f(s,a;\theta^t)$, we update the  policy network parameter $\alpha^{t+1} w^{t+1}$ using $\theta^t$ (Line~\ref{line:nn_update_policy_network}).
After the natural policy gradient, we obtain an improved policy network $\pi^{n+1}$, which is used to improve the coverage and guide the human data collection in the next phase.

\section{Proofs for PO-RLHF with Linear Function Approximation}



In this section, we give the proofs for algorithm $\algrlhf$.
In our analysis, the ideas of  MDP construction and natural policy gradient (Lemmas~\ref{lemma:M_n_equal_to_M_b}-\ref{lemma:regret_npg}) for optimistic MDPs are originated from \cite{agarwal2020pc}.

\subsection{MDP construction} \label{apx:mdp_construction}

We consider three MDPs as follows: (i) The true MDP $\cM$. (ii) The optimistic MDP with exploration bonuses $\cM_{b^n}$. $\cM_{b^n}$ replaces the reward function in $\cM$ by $r(s,a)+b^n(s,a)$. (iii) The ($\pi^*$, $\cK^n$)-modified optimistic MDP $\cM^n$. $\cM^n$ is the same as $\cM_{b^n}$ except that, for any $s \notin \cK^n$, $\cM^n$ adds an additional action $a^{\dagger}$ whose reward function and transition distribution are 
\begin{align*}
	r^n(s,a^{\dagger})=1, \quad p^n(s|s,a^{\dagger})=1 .
\end{align*}
In $\cM^n$, we consider a modified version of $\pi^*$, denoted by $\pi^{*,n}$. For any $s \in \cK^n$, $\pi^{*,n}(\cdot|s)=\pi^*(\cdot|s)$. For any $s \notin \cK^n$, $\pi^{*,n}(a^{\dagger}|s)=1$. Thus, in $\cM^n$, under policy $\pi^{*,n}$, once the agent goes into some $s \notin \cK^n$, she will self-loop and keep receiving the reward $1$.

\begin{lemma} \label{lemma:M_n_equal_to_M_b}
	For any phase $n \geq 0$, iteration $t \geq 0$, $s \in \cS$ and $a \neq a^{\dagger}$,
	\begin{align*}
		V_{\cM^n}^{\pi^t}(s) = V_{\cM_{b^n}}^{\pi^t}(s), \quad  Q_{\cM^n}^{\pi^t}(s,a) = Q_{\cM_{b^n}}^{\pi^t}(s,a), \quad  A_{\cM^n}^{\pi^t}(s,a) = A_{\cM_{b^n}}^{\pi^t}(s,a) .
	\end{align*}
\end{lemma}
\begin{proof}
	This lemma follows from the fact that $\cM^n$ is the same as $\cM_{b^n}$ except that $\cM^n$ has an additional action $a^{\dagger}$, but $\pi^t$ never picks $a^{\dagger}$.
\end{proof}

\begin{lemma}[Lemma C.1 in \cite{agarwal2020pc}] \label{lemma:d_star_n_leq_d_star}
	For any phase $n \geq 0$, $s\in \cK^n$ and $a \in \cA$,
	\begin{align*}
		d^{\pi^{*,n}}_{\cM^n}(s,a) \leq d^{\pi^{*}}_{\cM}(s,a) .
	\end{align*}
\end{lemma}

\begin{lemma}[Lemma C.2 in \cite{agarwal2020pc}] \label{lemma:optimistic_M_and_true_M}
	For any phase $n \geq 0$ and iteration $t \geq 0$,
	\begin{align*}
		& V_{\cM^n}^{\pi^{*,n}}(s_{\init}) \geq V_{\cM}^{\pi^{*}}(s_{\init}) , 
		\\
		& V_{\cM^n}^{\pi^t}(s_{\init}) = V_{\cM_{b^n}}^{\pi^t}(s_{\init}) \leq V_{\cM}^{\pi^t}(s_{\init}) + \frac{1}{1-\gamma} \sum_{(s,a) \notin \cK^n} d^{\pi^t}_{s_{\init}}(s,a) .
	\end{align*}
\end{lemma}

\begin{lemma}[Lemma C.3 in \cite{agarwal2020pc}] \label{lemma:sum_occupancy_not explore}
	For any phase $n \geq 0$,
	\begin{align*}
		\sum_{(s,a) \notin \cK^n} d^{\pi^{n+1}}_{s_{\init}}(s,a) \leq \frac{ 1 }{\beta} \ex_{(s,a) \sim d^{\pi^{n+1}}_{s_{\init}}} \mbr{ \phi(s,a)^\top (\hat{\Sigma}^n_{\cover})^{-1} \phi(s,a) }  .
	\end{align*}
	
	Furthermore, it holds that
	\begin{align*}
		\sum_{n=0}^{N-1} \sum_{(s,a) \notin \cK^n} d^{\pi^{n+1}}_{s_{\init}}(s,a) &\leq \frac{2}{\beta} \log\sbr{ \frac{ \det\sbr{\zeta_{\cover} I+\sum_{i=1}^{N} \ex_{(s,a) \sim d^{\pi^{i}}_{s_{\init}}} \mbr{ \phi(s,a) \phi(s,a)^\top } } }{ \det\sbr{\zeta_{\cover} I} } } 
		\\
		&\leq \frac{2 d}{\beta} \log\sbr{ 1+ \frac{N  }{\zeta_{\cover} d} } .
	\end{align*}
\end{lemma}

\subsection{Performance Difference Lemma and Policy Gradient on $\cM^n$}\label{apx:pdl_npg}

\begin{lemma}[Performance Difference Lemma on $\cM^n$] \label{lemma:modified_perf_diff}
	For any phase $n \geq 0$ and iteration $t \geq 0$,
	\begin{align*}
		V_{\cM^n}^{\pi^{*,n}}(s_{\init}) - V_{\cM^n}^{\pi^t}(s_{\init}) &\leq \frac{1}{1-\gamma} \ex_{(s,a) \sim d_{\cM^n;s_{\init}}^{\pi^{*,n}} }\mbr{ A_{\cM_{b^n}}^{\pi^t}(s,a) \cdot \indicator{s\in\cK^n} } .
	\end{align*}
\end{lemma}
\begin{proof}
	For any phase $n \geq 0$ and iteration $t \geq 0$, using the standard performance difference lemma~\cite{kakade2002approximately}, we have
	\begin{align*}
		V_{\cM^n}^{\pi^{*,n}}(s_{\init}) - V_{\cM^n}^{\pi^t}(s_{\init}) &= \frac{1}{1-\gamma} \ex_{(s,a) \sim d_{\cM^n;s_{\init}}^{\pi^{*,n}} }\mbr{ A_{\cM^n}^{\pi^t}(s,a) }
		\\
		&= \frac{1}{1-\gamma} \ex_{(s,a) \sim d_{\cM^n;s_{\init}}^{\pi^{*,n}} }\mbr{ A_{\cM^n}^{\pi^t}(s,a) \cdot \indicator{s\in\cK^n} } \\&\ \quad + \frac{1}{1-\gamma} \ex_{(s,a) \sim d_{\cM^n;s_{\init}}^{\pi^{*,n}} }\mbr{ A_{\cM^n}^{\pi^t}(s,a) \cdot \indicator{s\notin\cK^n} } .
	\end{align*}
	
	In $\cM^n$, for any $s \notin \cK^n$, policy $\pi^{*,n}$ chooses $a^{\dagger}$ deterministically. 
	For any $s \notin \cK^n$, we have 
	\begin{align*}
		A_{\cM^n}^{\pi^t}(s,a^{\dagger}) &= Q_{\cM^n}^{\pi^t}(s,a^{\dagger}) - V_{\cM^n}^{\pi^t}(s)
		\\
		&= 1+\gamma V_{\cM^n}^{\pi^t}(s) - V_{\cM^n}^{\pi^t}(s)
		\\
		&= 1-(1-\gamma) V_{\cM^n}^{\pi^t}(s)
		\\
		&\overset{\textup{(a)}}{\leq} 1-(1-\gamma) \sbr{0+\frac{1}{1-\gamma}} 
		\\
		&= 0 ,
	\end{align*}
	where inequality (a) is due to the facts that for any $s \notin \cK^n$, $\pi^t(\cdot|s)=\unif(\{a \in \cA: b^n(s,a)=\frac{1}{1-\gamma}\})$, and that $V_{\cM^n}^{\pi^t}(s)$ is no smaller than the cumulative reward $0$ plus the exploration bonus $b^n(s,a)=\frac{1}{1-\gamma}$.
	
	Therefore,
	\begin{align*}
		V_{\cM^n}^{\pi^{*,n}}(s_{\init}) - V_{\cM^n}^{\pi^t}(s_{\init}) &\leq \frac{1}{1-\gamma} \ex_{(s,a) \sim d_{\cM^n;s_{\init}}^{\pi^{*,n}} }\mbr{ A_{\cM^n}^{\pi^t}(s,a) \cdot \indicator{s\in\cK^n} } 
		\\
		&\overset{\textup{(b)}}{=} \frac{1}{1-\gamma} \ex_{(s,a) \sim d_{\cM^n;s_{\init}}^{\pi^{*,n}} }\mbr{ A_{\cM_{b^n}}^{\pi^t}(s,a) \cdot \indicator{s\in\cK^n} } ,
	\end{align*}
	where inequality (b) is due to that $A_{\cM^n}^{\pi^t}(s,a)=A_{\cM_{b^n}}^{\pi^t}(s,a)$ for $a \neq a^{\dagger}$ (Lemma~\ref{lemma:M_n_equal_to_M_b}), and $\pi^{*,n}$ never picks $a^{\dagger}$ for any state $s\in\cK^n$.
\end{proof}

Let $W_{A}:=\frac{4}{(1-\gamma)^2}$ and $\eta \leq \frac{1}{W_{A}}$. Then, $|\hat{A}^{\pi^t}_{\cM_{b^n}}(s,a)| \leq W_{A}$ for all $n\geq0$, $t\geq0$ and $(s,a)\in \cS\times\cA$.

\begin{lemma}[Regret for Natural Policy Gradient] \label{lemma:regret_npg}
	For any phase $n \geq 0$ and iteration $t \geq 0$,
	\begin{align*}
		\sum_{t=0}^{T-1} \ex_{(s,a) \sim d_{\cM^n;s_{\init}}^{\pi^{*,n}}}\mbr{ \hat{A}^{\pi^t}_{\cM_{b^n}}(s,a) \cdot \indicator{s \in \cK^n} } \leq \frac{\log(|\cA|)}{\eta} + \eta W_{A}^2 T .
	\end{align*}
\end{lemma}
\begin{proof}
	For any phase $n \geq 0$, iteration $t \geq 0$, $s \in \cK^n$ and $a \in \cA$, we have $b^n(s,a)=0$.
	
	Define 
	\begin{align*}
		D_s:=&\ \sum_{a' \in \cA} \sbr{ \exp \sbr{\phi(s,a')^\top w^t} } ,
		\\
		E_s:=&\ \exp \sbr{ - \eta \ex_{a \sim \pi^t(\cdot|s)}\mbr{ \phi(s,a)^\top\theta^{t}} }
		\\
		=&\ \exp \sbr{ - \eta \ex_{a \sim \pi^t(\cdot|s)}\mbr{ \phi(s,a)^\top\theta^{t} + b^n(s,a)  } }
		\\
		=&\ \exp \sbr{ - \eta \hat{V}^{\pi^t}_{\cM_{b^n}}(s)  } ,
	\end{align*} 
	and we have
	\begin{align*}
		\pi^{t+1}(\cdot|s) &= \frac{ \exp \sbr{\phi(s,\cdot)^\top w^{t+1}} }{\sum_{a' \in \cA} \sbr{\phi(s,a')^\top w^{t+1}} }
		\\
		&= \frac{ \exp \sbr{\phi(s,\cdot)^\top \sbr{w^t+\eta \theta^{t}}} }{\sum_{a' \in \cA} \sbr{\phi(s,a')^\top \sbr{w^t+\eta \theta^{t}} } }
		\\
		&= \frac{ \frac{\exp \sbr{\phi(s,\cdot)^\top w^t}}{D_s} \cdot \exp \sbr{ \eta  \sbr{\phi(s,\cdot)^\top\theta^{t} + b^n(s,\cdot)}   } }{\sum_{a' \in \cA} \sbr{ \frac{\exp \sbr{\phi(s,a')^\top w^t}}{D_s} \cdot \exp \sbr{ \eta  \sbr{\phi(s,a')^\top\theta^{t} + b^n(s,a')}   } } }
		\\
		&= \frac{ \pi^t(\cdot|s) \cdot \exp \sbr{ \eta  \hat{Q}^{\pi^t}_{\cM_{b^n}}(s,\cdot)  } \cdot E_s }{\sum_{a' \in \cA} \sbr{ \pi^t(a'|s) \cdot \exp \sbr{ \eta  \hat{Q}^{\pi^t}_{\cM_{b^n}}(s,a')   } \cdot E_s } }
		\\
		&= \frac{ \pi^t(\cdot|s) \cdot \exp \sbr{ \eta  \hat{A}^{\pi^t}_{\cM_{b^n}}(s,\cdot)  } }{\sum_{a' \in \cA} \sbr{ \pi^t(a'|s) \cdot \exp \sbr{ \eta  \hat{A}^{\pi^t}_{\cM_{b^n}}(s,a') } } }
		\\
		&= \frac{ \pi^t(\cdot|s) \cdot \exp \sbr{ \eta  \hat{A}^{\pi^t}_{\cM_{b^n}}(s,\cdot) } }{\sum_{a' \in \cA} \sbr{ \pi^t(a'|s) \cdot \exp \sbr{ \eta  \hat{A}^{\pi^t}_{\cM_{b^n}}(s,a') } } } .
	\end{align*}
	
	Define $G_s:=\sum_{a' \in \cA} ( \pi^t(a'|s) \cdot \exp  (\eta  \hat{A}^{\pi^t}_{\cM_{b^n}}(s,a') ) )$, and we have
	\begin{align*}
		\log(G_s) &= \log \sbr{ \sum_{a' \in \cA} \sbr{ \pi^t(a'|s) \cdot \exp \sbr{ \eta  \hat{A}^{\pi^t}_{\cM_{b^n}}(s,a') } } }
		\\
		&\overset{\textup{(a)}}{\leq} \log \sbr{ \sum_{a' \in \cA} \sbr{ \pi^t(a'|s) \cdot \sbr{ 1 + \eta  \hat{A}^{\pi^t}_{\cM_{b^n}}(s,a') + \sbr{\eta  \hat{A}^{\pi^t}_{\cM_{b^n}}(s,a')}^2 } } }
		\\
		&\leq \log \sbr{1 + \eta^2 W_{A}^2}
		\\
		&\leq \eta^2 W_{A}^2 ,
	\end{align*}
	where inequality (a) is due to that $\eta \hat{A}^{\pi^t}_{\cM_{b^n}}(s,a') \leq \eta W_{A} \leq 1$ and $\exp(x)\leq 1+x+x^2$ for any $x \leq 1$.

	Thus, for any $s \in \cK^n$, we have
	\begin{align*}
		&\ \quad \kl(\pi^{*,n}(\cdot|s) \| \pi^{t+1}(\cdot|s)) - \kl(\pi^{*,n}(\cdot|s) \| \pi^{t}(\cdot|s)) 
		\\
		&= \ex_{a \sim \pi^{*,n}(\cdot|s)}\mbr{ \log \sbr{\frac{\pi^{*,n}(a|s)}{\pi^{t+1}(a|s)}} } - \ex_{a \sim \pi^{*,n}(\cdot|s)}\mbr{ \log \sbr{\frac{\pi^{*,n}(a|s)}{\pi^{t}(a|s)}} }
		\\
		&= \ex_{a \sim \pi^{*,n}(\cdot|s)}\mbr{ \log \sbr{\frac{\pi^{t}(a|s)}{\pi^{t+1}(a|s)}} }
		\\
		&= \ex_{a \sim \pi^{*,n}(\cdot|s)}\mbr{ \log(G_s) - \eta  \hat{A}^{\pi^t}_{\cM_{b^n}}(s,a)  }
		\\
		&\leq - \eta \ex_{a \sim \pi^{*,n}(\cdot|s)}\mbr{ \hat{A}^{\pi^t}_{\cM_{b^n}}(s,a) } + \eta^2 W_{A}^2 ,
	\end{align*}
	which is equivalent to
	\begin{align*}
		\ex_{a \sim \pi^{*,n}(\cdot|s)}\mbr{ \hat{A}^{\pi^t}_{\cM_{b^n}}(s,a) } \leq  \frac{1}{\eta} \sbr{\kl(\pi^{*,n}(\cdot|s) \| \pi^{t}(\cdot|s)) - \kl(\pi^{*,n}(\cdot|s) \| \pi^{t+1}(\cdot|s))} + \eta W_{A}^2 .
	\end{align*}
	
	Adding $s \sim d_{\cM^n;s_{\init}}^{\pi^{*,n}}$ on both sides and summing over $t=0,\dots,T-1$, we have
	\begin{align*}
		\sum_{t=0}^{T-1} \ex_{(s,a) \sim d_{\cM^n;s_{\init}}^{\pi^{*,n}}}\mbr{ \hat{A}^{\pi^t}_{\cM_{b^n}}(s,a) \cdot \indicator{s \in \cK^n} } &\leq \frac{1}{\eta} \ex_{s \sim d_{\cM^n;s_{\init}}^{\pi^{*,n}}}\mbr{\kl(\pi^{*,n}(\cdot|s) \| \pi^{0}(\cdot|s)) - \kl(\pi^{*,n}(\cdot|s) \| \pi^{T}(\cdot|s))} 
		\\
		&\quad\ + \eta W_{A}^2 T
		\\
		&\leq \frac{\log(|\cA|)}{\eta} + \eta W_{A}^2 T .
	\end{align*}
	
\end{proof}

\subsection{Human Feedback} \label{apx:human_feedback}

For any trajectory $\tau=(s_0,a_0,\dots,s_{H(\tau)},a_{H(\tau)})$, let $H(\tau)$ denote the length of $\tau$, and $\phi(\tau):=\sum_{h=0}^{H(\tau)} \phi(s_h,a_h)$.
For any trajectories $\tau^{(1)}, \tau^{(2)}$, let $\tilde{\phi}^{\tau^{(1)},\tau^{(2)}}:=\sum_{h=0}^{H(\tau^{(1)})} \phi(s^{(1)}_{h},a^{(1)}_{h}) - \sum_{h=0}^{H(\tau^{(2)})} \phi(s^{(2)}_{h},a^{(2)}_{h})$.

For any $(s,a) \in \cS \times \cA$ and policy $\pi$, let $\cO^{\pi}_{s,a}$ be the distribution of the trajectory which is generated by starting at $(s,a)$, executing policy $\pi$ and terminating with probability $1-\gamma$ at each step. For any state-action distribution $\rho$, let $\cO^{\pi}_{\rho}:=\ex_{\rho\sim(s,a)}[\cO^{\pi}_{s,a}]$.

\subsubsection{Trajectory Length and Covariance Matrix Concentration}

To analyze the reward estimation error under human feedback, we first define the concentration events for trajectory length and the coverage and human data covariance matrices.

Define event
\begin{align}
	\cE_{\tau}:=\lbr{ |\tau|\leq  \frac{\log\sbr{\frac{1}{\delta'}}}{1-\gamma} := W_{\tau} , \ \textup{for any trajectory $\tau$ sampled in the algorithm} } . \label{eq:def_event_len_traj}
\end{align}

\begin{lemma}
	It holds that $\Pr[\cE_{\tau}] \geq 1-2N(K+M_{\hf}+TM_{\subsgd})\delta'$.
\end{lemma}
\begin{proof}
	This proof is similar to Eqs.~(94)-(97) in \cite{zanette2021cautiously}.
	
	Let $H$ denote the length of a trajectory which is generated by terminating with probability $1-\gamma$ at each step. Then,
	$H$ is a random variable which satisfies $\Pr[H=t]=\gamma^{t-1} (1-\gamma)$ for  $t=1,2,\dots$.
	
	We have
	\begin{align*}
		\Pr\mbr{H>h} = \sum_{t=h+1}^{\infty} \gamma^{t-1} (1-\gamma)
		= \gamma^{h} \sum_{t=1}^{\infty} \gamma^{t-1} (1-\gamma)
		= \gamma^{h} \sum_{t=1}^{\infty} \gamma^{t-1} (1-\gamma)
		= \gamma^{h} .
	\end{align*}
	
	Let $\delta'=\gamma^{h}$.
	Then,
	\begin{align*}
		h = \frac{\ln(\delta')}{\ln(\gamma)}
		= \frac{-\ln(\delta')}{-\ln(\gamma)}
		\leq \frac{-\ln(\delta')}{-(\gamma-1)}
		\leq \frac{\ln\sbr{\frac{1}{\delta'}}}{1-\gamma} .
	\end{align*}
	
	Thus, we have
	\begin{align*}
		\Pr\mbr{H>\frac{\ln\sbr{\frac{1}{\delta'}}}{1-\gamma}} \leq \delta' .
	\end{align*}
\end{proof}

Let $\zeta_{\cover}:=1$ and $\zeta_{\hf}:=4W_{\tau}^2$.
For any $n\geq0$ and $1\leq i \leq K$, let $(s^n_i,a^n_i)$ denote the $i$-th state-action pair sampled in phase $n$ for constructing the estimated coverage covariance matrix $\hat{\Sigma}^{n}_{\cover}$ (Line~\ref{line:sample_cov} in Algorithm~\ref{alg:pg_rlhf}).

For any phase $n \geq 0$, let
\begin{align*}
	\hat{\Sigma}^{n}_{\cover}&:=\sum_{i=0}^{n}  \sbr{ \frac{1}{K} \sum_{i=1}^{K} \phi(s^n_i,a^n_i)\phi(s^n_i,a^n_i)^\top } + \zeta_{\cover} I ,
	\\
	\Sigma^{n}_{\cover}&:=\sum_{i=0}^{n} \ex_{(s,a)\sim d^{\pi^n}_{s_{\init}}} \mbr{\phi(s,a)\phi(s,a)^\top} + \zeta_{\cover} I 
	\\
	&= (n+1)\ex_{(s,a)\sim \rho^n_{\cover}} \sbr{\phi(s,a)\phi(s,a)^\top} + \zeta_{\cover} I .
	\\
	\hat{\Sigma}^{n}_{\hf} &:=\frac{1}{M_{\hf}} \sum_{i=1}^{M_{\hf}} \sbr{ \phi(\tau^{(1)}_i) - \phi(\tau^{(2)}_i)} \sbr{\phi(\tau^{(1)}_i) - \phi(\tau^{(2)}_i)}^\top + \frac{\zeta_{\hf}}{n} I
	\\
	&= \frac{1}{M_{\hf}} \sum_{i=1}^{M_{\hf}} \tilde{\phi}^{\tau^{(1)}_i,\tau^{(2)}_i} \sbr{ \tilde{\phi}^{\tau^{(1)}_i,\tau^{(2)}_i} }^\top + \frac{\zeta_{\hf}}{n} I , \quad \forall n\geq1
	\\
	\Sigma^{n}_{\hf} &:= \frac{1}{n} \sum_{i=1}^{n} \Bigg(  \ex_{\begin{subarray}{l} \tau^{(1)} \sim \cO^{\pi^i}_{\rho^{i-1}_{\cover}}\\ \tau^{(2)} \sim \cO^{\pi^{\base}}_{s_{\init}} \end{subarray}}  \mbr{ \sbr{\phi(\tau^{(1)}) - \phi(\tau^{(2)})} \sbr{\phi(\tau^{(1)}) - \phi(\tau^{(2)})}^\top } \Bigg) + \frac{\zeta_{\hf}}{n} I 
	\\
	&= \frac{1}{n} \sum_{i=1}^{n} \Bigg(  \ex_{\begin{subarray}{l} \tau^{(1)} \sim \cO^{\pi^i}_{\rho^{i-1}_{\cover}}\\ \tau^{(2)} \sim \cO^{\pi^{\base}}_{s_{\init}} \end{subarray}}  \mbr{ \tilde{\phi}^{\tau^{(1)},\tau^{(2)}} \sbr{\tilde{\phi}^{\tau^{(1)},\tau^{(2)}}}^\top } \Bigg) + \frac{\zeta_{\hf}}{n} I , \quad \forall n\geq1
	\\
	\hat{\Sigma}^{0}_{\hf} &= \Sigma^{0}_{\hf} := \zeta_{\hf} I .
\end{align*}

Define event
\begin{align*}
	\cE_{\cover}:=\Bigg\{&
	\frac{1}{2} \nbr{ \phi(s,a) }_{(\Sigma^{n}_{\cover})^{-1}} \leq \nbr{ \phi(s,a) }_{(\hat{\Sigma}^{n}_{\cover})^{-1}} \leq 2 \nbr{ \phi(s,a) }_{(\Sigma^{n}_{\cover})}^{-1} ,
	\\
	&\frac{1}{2} \nbr{ \phi(s,a) }_{(\Sigma^{n}_{\hf})^{-1}} \leq \nbr{ \phi(s,a) }_{(\hat{\Sigma}^{n}_{\hf})^{-1}} \leq 2 \nbr{ \phi(s,a) }_{(\Sigma^{n}_{\hf})}^{-1} ,\ \forall 0\leq n \leq N-1
	\Bigg\} .
\end{align*}

\begin{lemma}
	Assuming that event $\cE_{\tau}$ holds, we have $\Pr[\cE_{\cover}] \geq 1-2N\delta'$.
\end{lemma}
\begin{proof}
	This lemma follows from Lemma~\ref{lemma:con_matrix_inverse} and the conditions that $K \geq \frac{16 (N+1)^2  \log^2\sbr{\frac{4dN}{\delta'}} }{\zeta_{\cover}^2}$ and $M_{\hf} \geq \frac{16 W_{\tau}^4 \log^2\sbr{\frac{4d}{\delta'}} }{\zeta_{\hf}^2}$.	
\end{proof}

\subsubsection{Reward Estimation Error in Q-value Functions}

Let $W_{\mu}:=1$.

For any $n \geq 0$, recall that
\begin{align*}
	\hat{\mu}^{n} & := \argmin_{\|\mu\|_2 \leq W_{\mu}} \Bigg( - \sum_{i=1}^{M_{\hf}} \log \bigg( \frac{ \indicator{ y_i=1 } }{ 1+ \exp\Big(  \big( \sum_{h=0}^{H(\tau^{(2)}_i)}\phi(s^{(2)}_{i,h},a^{(2)}_{i,h}) - \sum_{h=0}^{H(\tau^{(1)}_i)}\phi(s^{(1)}_{i,h},a^{(1)}_{i,h}) \big)^\top \mu \Big) } 
	\nonumber\\ & \ \quad + \frac{ \indicator{ y_i=0 } }{ 1+ \exp\Big(  \big( \sum_{h=0}^{H(\tau^{(1)}_i)}\phi(s^{(1)}_{i,h},a^{(1)}_{i,h}) - \sum_{h=0}^{H(\tau^{(2)}_i)}\phi(s^{(2)}_{i,h},a^{(2)}_{i,h}) \big)^\top \mu \Big) } \bigg) \Bigg) . 
\end{align*}

\begin{lemma}[MLE, Lemma 5.1 in \cite{zhu2023principled}] \label{lemma:mle}
	For any phase $n \geq 0$, with probability at least $1-\delta'$, we have
	\begin{align*}
		\nbr{ \hat{\mu}^{n} - \mu^* }_{\hat{\Sigma}^{n}_{\hf}} \leq 8 \sqrt{\frac{d+\log\sbr{\frac{1}{\delta'}}}{c_{\mle}^2 M_{\hf}} +  \frac{\zeta_{\hf} W_{\mu}^2}{n} } 
		:= \varepsilon^n_{\hf} .
	\end{align*}
	where $c_{\mle}:=\frac{1}{2+\exp(-2W_{\tau} W_{\mu})+\exp(2W_{\tau} W_{\mu})}$.
	
	In other words, defining event
	\begin{align*}
		\cE_{\mle}=\lbr{ \nbr{ \hat{\mu}^{n} - \mu^* }_{\hat{\Sigma}^{n}_{\hf}} \leq \varepsilon^n_{\hf} ,\ \forall 0\leq n \leq N-1} ,
	\end{align*}
	we have $\Pr[\cE_{\mle}]\geq 1-N\delta'$.
\end{lemma}

\begin{lemma} \label{lemma:Q_decomposition_traj}
	Assume that event $\cE_{\tau}\cap\cE_{\cover}\cap\cE_{\mle}$ holds. Then, for any phase $n \geq 0$, iteration $t \geq 0$ and $(s,a) \in \cS \times \cA$,
	\begin{align*}
		\abr{ Q^{\pi^t}(s,a;\hat{r}^n+b^n) - Q^{\pi^t}(s,a;r+b^n) } &\leq 2\varepsilon^n_{\hf} \ex_{\tau \sim \cO^{\pi^t}_{s,a}} \mbr{\nbr{\sum_{h=0}^{H(\tau)} \phi(s_h,a_h)}_{\sbr{\Sigma_{\hf}^{n}}^{-1}} } := \varsigma^{\pi^t}_{s,a} .
	\end{align*}
\end{lemma}
\begin{proof}
	Since $Q^{\pi^t}(s,a;\hat{r}^n+b^n) = \ex_{\tau \sim \cO^{\pi^t}_{s,a}} [\sum_{h=0}^{H(\tau)} (\hat{r}^n(s_h,a_h)+b^n(s_h,a_h))]$ and $Q^{\pi^t}(s,a;r+b^n) = \ex_{\tau \sim \cO^{\pi^t}_{s,a}} [\sum_{h=0}^{H(\tau)} (r(s_h,a_h)+b^n(s_h,a_h))]$,
	we have
	\begin{align*}
		\abr{ Q^{\pi^t}(s,a;\hat{r}^n+b^n) - Q^{\pi^t}(s,a;r+b^n) } &= \abr{\ex_{\tau \sim \cO^{\pi^t}_{s,a}} \mbr{\sum_{h=0}^{H(\tau)} \sbr{\hat{r}^n(s_h,a_h) - r(s_h,a_h)} }}
		\\
		&\leq \ex_{\tau \sim \cO^{\pi^t}_{s,a}} \mbr{\abr{\sum_{h=0}^{H(\tau)} \sbr{\hat{r}^n(s_h,a_h) - r(s_h,a_h)}}}
		\\
		&= \ex_{\tau \sim \cO^{\pi^t}_{s,a}} \mbr{\abr{\sum_{h=0}^{H(\tau)} \phi(s_h,a_h)^{\top} \sbr{ \hat{\mu}^{n} - \mu^* } } } 
		\\
		&\leq \ex_{\tau \sim \cO^{\pi^t}_{s,a}} \mbr{\nbr{\sum_{h=0}^{H(\tau)} \phi(s_h,a_h)}_{\sbr{\hat{\Sigma}_{\hf}^{n}}^{-1}} \nbr{\hat{\mu}^{n} - \mu^*}_{\hat{\Sigma}_{\hf}^{n}}  } 
		\\
		&\overset{\textup{(a)}}{\leq} 2\varepsilon^n_{\hf} \ex_{\tau \sim \cO^{\pi^t}_{s,a}} \mbr{\nbr{\sum_{h=0}^{H(\tau)} \phi(s_h,a_h)}_{\sbr{\Sigma_{\hf}^{n}}^{-1}} } ,
	\end{align*}
	where inequality (a) is due to the definition of event $\cE_{\cover}$.
\end{proof}

Let
$
\varsigma^{\pi^t}_{\rho^n_{\cover}}:=
\ex_{(s,a) \sim \rho^n_{\cover}} [\varsigma^{\pi^t}_{s,a}] =2\varepsilon^n_{\hf} \ex_{\tau \sim \cO^{\pi^t}_{\rho^n_{\cover}}} [\|\sum_{h=0}^{H(\tau)} \phi(s_h,a_h)\|_{\sbr{\Sigma_{\hf}^{n}}^{-1}} ] 
$, $W_{\theta}:=\frac{2}{(1-\gamma)^2}-\frac{1}{1-\gamma}$
and
$W_{Q}:=\frac{2}{(1-\gamma)^2}$. 

\begin{lemma} \label{lemma:phi_theta_star_minus_theta_mid}
	Assume that event $\cE_{\tau}\cap\cE_{\cover}\cap\cE_{\mle}$ holds. Then, for any phase $n \geq 0$, iteration $t \geq 0$, $s \in \cK^n$ and $a \in \cA$,
	\begin{align*}
		\abr{\phi(s,a)^\top \sbr{\theta^{t}_{*} - \theta^{t}_{\submid}}} \leq \sqrt{32 \beta W_{Q} \varepsilon^n_{\hf} (n+1) \ex_{\tau \sim \cO^{\pi^t}_{\rho^n_{\cover}}} \mbr{\nbr{\sum_{h=0}^{H(\tau)} \phi(s_h,a_h)}_{\sbr{\Sigma_{\hf}^{n}}^{-1}} } } + W_{\theta}\sqrt{8\beta \zeta_{\hf}} .
	\end{align*}
\end{lemma}
\begin{proof}
	For any phase $n \geq 0$ and iteration $t \geq 0$, for any fixed $\theta$ and $(s,a)$, using Lemma~\ref{lemma:tech_sq_diff}, we have
	\begin{align}
		&\ \quad  \sbr{ Q^{\pi^t}(s,a;\hat{r}^n+b^n)- b^n(s,a) - \phi(s,a)^\top \theta }^2 - \sbr{  Q^{\pi^t}(s,a;r+b^n)- b^n(s,a) - \phi(s,a)^\top \theta }^2
		\nonumber\\
		&\leq 4W_{Q} \abr{ Q^{\pi^t}(s,a;\hat{r}^n+b^n) - Q^{\pi^t}(s,a;r+b^n) }
		\nonumber\\
		&\leq 4W_{Q} \varsigma^{\pi^t}_{s,a} , \label{eq:derivation_varsigma1}
	\end{align}
	where $W_Q$ satisfies that $\max\{|Q^{\pi^t}(s,a;\hat{r}^n+b^n)|, |Q^{\pi^t}(s,a;r+b^n)|,  |\phi(s,a)^\top \theta + b^n(s,a)|\} \leq W_{Q}$ for all $n\geq0$, $t\geq 0$ and $(s,a) \in \cS \times \cA$.
	
	Taking $\ex_{(s,a) \sim \rho^{n}_{\cover}}[\cdot]$ on both sides, we have
	\begin{align}
		&\ \quad  \ex_{(s,a) \sim \rho^n_{\cover}} \mbr{ \sbr{ Q^{\pi^t}(s,a;\hat{r}^n+b^n)- b^n(s,a) - \phi(s,a)^\top \theta }^2}
		\\
		&\quad\ - \ex_{(s,a) \sim \rho^n_{\cover}} \mbr{\sbr{  Q^{\pi^t}(s,a;r+b^n)- b^n(s,a) - \phi(s,a)^\top \theta }^2 } 
		\nonumber\\
		&\leq 4W_{Q} \ex_{(s,a) \sim \rho^n_{\cover}} \mbr{\varsigma^{\pi^t}_{s,a}} 
		\nonumber\\
		&= 4W_{Q} \varsigma^{\pi^t}_{\rho^n_{\cover}} . \label{eq:derivation_varsigma2}
	\end{align}
	
	Plugging $\theta^{t}_{*}$ into $\theta$, we have that for any fixed $(s,a)$,
	\begin{align}
		&\ \quad  \ex_{(s,a) \sim \rho^n_{\cover}} \mbr{\sbr{  Q^{\pi^t}(s,a;r+b^n)- b^n(s,a) - \phi(s,a)^\top \theta^{t}_{*} }^2 } 
		\nonumber\\
		&\geq \ex_{(s,a) \sim \rho^n_{\cover}} \mbr{ \sbr{ Q^{\pi^t}(s,a;\hat{r}^n+b^n)- b^n(s,a) - \phi(s,a)^\top \theta^{t}_{*} }^2} - 4W_{Q} \varsigma^{\pi^t}_{\rho^n_{\cover}}
		\nonumber\\
		&\overset{\textup{(a)}}{\geq} \ex_{(s,a) \sim \rho^n_{\cover}} \mbr{ \sbr{ Q^{\pi^t}(s,a;\hat{r}^n+b^n)- b^n(s,a) - \phi(s,a)^\top \theta^{t}_{\submid} }^2} - 4W_{Q} \varsigma^{\pi^t}_{\rho^n_{\cover}} \label{eq:derivation_varsigma3}
	\end{align}
	where inequality (a) is due to the definition of $\theta^{t}_{\submid}$. 
	
	Furthermore, we have
	\begin{align}
		&\ \quad  \ex_{(s,a) \sim \rho^n_{\cover}} \mbr{ \sbr{ Q^{\pi^t}(s,a;r+b^n)- b^n(s,a) - \phi(s,a)^\top \theta^{t}_{\submid} }^2} 
		\nonumber\\
		&\quad\  - \ex_{(s,a) \sim \rho^n_{\cover}} \mbr{\sbr{  Q^{\pi^t}(s,a;r+b^n)- b^n(s,a) - \phi(s,a)^\top \theta^{t}_{*} }^2 } 
		\nonumber\\
		&= \ex_{(s,a) \sim \rho^n_{\cover}} \mbr{ \sbr{ Q^{\pi^t}(s,a;\hat{r}^n+b^n)- b^n(s,a) - \phi(s,a)^\top \theta^{t}_{\submid} }^2} 
		\nonumber\\
		&\quad\ - \ex_{(s,a) \sim \rho^n_{\cover}} \mbr{\sbr{  Q^{\pi^t}(s,a;r+b^n)- b^n(s,a) - \phi(s,a)^\top \theta^{t}_{*} }^2 } 
		\nonumber\\
		& \quad + \ex_{(s,a) \sim \rho^n_{\cover}} \mbr{ \sbr{ Q^{\pi^t}(s,a;r+b^n)- b^n(s,a) - \phi(s,a)^\top \theta^{t}_{\submid} }^2} 
		\nonumber\\
		&\quad\ - \ex_{(s,a) \sim \rho^n_{\cover}} \mbr{\sbr{  Q^{\pi^t}(s,a;\hat{r}^n+b^n)- b^n(s,a) - \phi(s,a)^\top \theta^{t}_{\submid} }^2 } 
		\nonumber\\
		&\overset{\textup{(a)}}{\leq} 4W_{Q} \varsigma^{\pi^t}_{\rho^n_{\cover}} + 4W_{Q} \ex_{(s,a) \sim \rho^n_{\cover}} \mbr{\abr{ Q^{\pi^t}(s,a;\hat{r}^n+b^n) - Q^{\pi^t}(s,a;r+b^n) }}
		\nonumber\\
		&\leq 8 W_{Q} \varsigma^{\pi^t}_{\rho^n_{\cover}} , \label{eq:derivation_varsigma4}
	\end{align}
	where inequality (a) uses Lemma~\ref{lemma:tech_sq_diff}.
	
	On the other hand, it holds that
	\begin{align}
		&\ \quad \ex_{(s,a) \sim \rho^n_{\cover}} \mbr{ \sbr{ Q^{\pi^t}(s,a;r+b^n)- b^n(s,a) - \phi(s,a)^\top \theta^{t}_{\submid} }^2} 
		\nonumber\\
		&\quad\ - \ex_{(s,a) \sim \rho^n_{\cover}} \mbr{\sbr{  Q^{\pi^t}(s,a;r+b^n)- b^n(s,a) - \phi(s,a)^\top \theta^{t}_{*} }^2 }
		\nonumber\\
		&= \ex_{(s,a) \sim \rho^n_{\cover}} \mbr{ \sbr{\phi(s,a)^\top \sbr{\theta^{t}_{*} -  \theta^{t}_{\submid} }}^2 } 
		\nonumber\\
		&\quad\ + 2 \underbrace{\ex_{(s,a) \sim \rho^n_{\cover}} \mbr{ \sbr{ Q^{\pi^t}(s,a;r+b^n)- b^n(s,a) - \phi(s,a)^\top \theta^{t}_{*} } \phi(s,a)^\top \sbr{\theta^{t}_{*} -  \theta^{t}_{\submid} } }}_{\textup{Term $\Gamma$}\ \geq\ 0} , \label{eq:first_order_opt}
	\end{align}
	where Term $\Gamma$ is non-negative due to the the first-order optimality of $\theta^{t}_{*}$.
	
	Thus, we have
	\begin{align*}
		&\ \quad \ex_{(s,a) \sim \rho^n_{\cover}} \mbr{ \sbr{\phi(s,a)^\top \sbr{\theta^{t}_{*} -  \theta^{t}_{\submid} }}^2 } 
		\\
		&\leq \ex_{(s,a) \sim \rho^n_{\cover}} \mbr{ \sbr{ Q^{\pi^t}(s,a;r+b^n)- b^n(s,a) - \phi(s,a)^\top \theta^{t}_{\submid} }^2} 
		\\
		&\quad\ - \ex_{(s,a) \sim \rho^n_{\cover}} \mbr{\sbr{  Q^{\pi^t}(s,a;r+b^n)- b^n(s,a) - \phi(s,a)^\top \theta^{t}_{*} }^2 }
		\\
		&\leq 8 W_{Q} \varsigma^{\pi^t}_{\rho^n_{\cover}} .
	\end{align*}
	
	Since $\ex_{(s,a) \sim \rho^n_{\cover}} [ (\phi(s,a)^\top (\theta^{t}_{*} -  \theta^{t}_{\submid}))^2 ] = \ex_{(s,a) \sim \rho^n_{\cover}} [ (\theta^{t}_{*} -  \theta^{t}_{\submid})^\top \phi(s,a)  \phi(s,a)^\top (\theta^{t}_{*} -  \theta^{t}_{\submid}) ] =  (\theta^{t}_{*} -  \theta^{t}_{\submid})^\top \ex_{(s,a) \sim \rho^n_{\cover}} [\phi(s,a)  \phi(s,a)^\top] (\theta^{t}_{*} -  \theta^{t}_{\submid})$,
	we have
	\begin{align*}
		(\theta^{t}_{*} -  \theta^{t}_{\submid})^\top \ex_{(s,a) \sim \rho^n_{\cover}} [\phi(s,a)  \phi(s,a)^\top] (\theta^{t}_{*} -  \theta^{t}_{\submid}) \leq 8 W_{Q} \varsigma^{\pi^t}_{\rho^n_{\cover}} .
	\end{align*}
	
	Moreover,
	\begin{align*}
		&\ \quad \nbr{ \theta^{t}_{*} -  \theta^{t}_{\submid} }^2_{\Sigma^{n}_{\cover}}
		\\
		&= \sbr{ \theta^{t}_{*} -  \theta^{t}_{\submid} }^\top \sbr{ \sum_{i=0}^{n} \ex_{(s,a) \sim d^{\pi^{i}}_{s_{\init}}} \mbr{\phi(s,a)  \phi(s,a)^\top} + \zeta_{\cover} I} \sbr{ \theta^{t}_{*} -  \theta^{t}_{\submid} }
		\\
		&= \sbr{ \theta^{t}_{*} -  \theta^{t}_{\submid} }^\top \sbr{ \sum_{i=0}^{n} \sum_{(s,a)}  d^{\pi^{i}}_{s_{\init}}(s,a) \cdot \phi(s,a)  \phi(s,a)^\top + \zeta_{\cover} I} \sbr{ \theta^{t}_{*} -  \theta^{t}_{\submid} }
		\\
		&= (n+1) \sbr{ \theta^{t}_{*} -  \theta^{t}_{\submid} }^\top \sbr{  \sum_{(s,a)} \frac{1}{n+1} \sum_{i=0}^{n} d^{\pi^{i}}_{s_{\init}}(s,a) \cdot \phi(s,a)  \phi(s,a)^\top + \frac{\zeta_{\cover}}{n+1} I} \sbr{ \theta^{t}_{*} -  \theta^{t}_{\submid} }
		\\
		&= (n+1) \sbr{ \theta^{t}_{*} -  \theta^{t}_{\submid} }^\top \sbr{  \sum_{(s,a)} \rho^{n}_{\cover}(s,a) \cdot \phi(s,a)  \phi(s,a)^\top + \frac{\zeta_{\cover}}{n+1} I} \sbr{ \theta^{t}_{*} -  \theta^{t}_{\submid} }
		\\
		&\leq 8 (n+1) W_{Q} \varsigma^{\pi^t}_{\rho^n_{\cover}} + 4\zeta_{\cover} W_{\theta}^2 .
	\end{align*}
	
	For any $s \in \cK^n$, using the definitions of $\cK^n$ and event $\cE_{\cover}$, we have
	\begin{align*}
		\frac{1}{\sqrt{2}} \nbr{ \phi(s,a) }_{(\Sigma^{n}_{\cover})^{-1}} \leq \nbr{ \phi(s,a) }_{(\hat{\Sigma}^{n}_{\cover})^{-1}} \leq \sqrt{\beta} .
	\end{align*}
	
	Therefore, we obtain
	\begin{align*}
		\abr{\phi(s,a)^\top \sbr{\theta^{t}_{*} - \theta^{t}_{\submid}}} &\leq \nbr{ \phi(s,a) }_{(\Sigma^{n}_{\cover})^{-1}} \nbr{ \theta^{t}_{*} - \theta^{t}_{\submid} }_{\Sigma^{n}_{\cover}}
		\\
		&\leq \sqrt{2\beta \sbr{8 (n+1) W_{Q} \varsigma^{\pi^t}_{\rho^n_{\cover}} + 4\zeta_{\cover} W_{\theta}^2}}
		\\
		&\leq \sqrt{32 \beta W_{Q} \varepsilon^n_{\hf} (n+1) \ex_{\tau \sim \cO^{\pi^t}_{\rho^n_{\cover}}} \mbr{\nbr{\sum_{h=0}^{H(\tau)} \phi(s_h,a_h)}_{\sbr{\Sigma_{\hf}^{n}}^{-1}} } } + W_{\theta}\sqrt{8\beta \zeta_{\cover}}  .
	\end{align*}
	
\end{proof}

\subsubsection{Elliptical Potential Analysis for Human Data}

\begin{lemma}[Elliptical Potential for the Baseline Policy] \label{lemma:phi_tau_base}
	For any phase $n\geq0$,
	\begin{align*}
		\sum_{n=0}^{N-1} \ex_{\tau \sim \cO^{\pi^{\base}}_{s_{\init}}} \mbr{ \nbr{\phi(\tau)}^2_{\sbr{n\Sigma_{\hf}^{n}}^{-1} } } 
		\leq \frac{2d}{c_{\base}} \log(N) .
	\end{align*}
\end{lemma}
\begin{proof}
	We have
	\begin{align*}
		&\quad\ \sum_{n=0}^{N-1} \ex_{\tau \sim \cO^{\pi^{\base}}_{s_{\init}}} \mbr{ \nbr{\phi(\tau)}^2_{\sbr{n\Sigma_{\hf}^{n}}^{-1} } } 
		\\
		&= \sum_{n=0}^{N-1} \ex_{\tau \sim \cO^{\pi^{\base}}_{s_{\init}}} \mbr{ \phi(\tau)^\top \sbr{ \sum_{i=1}^{n} \ex_{\begin{subarray}{l} \tau^{(1)} \sim \cO^{\pi^i}_{\rho^{i-1}_{\cover}}\\ \tau^{(2)} \sim \cO^{\pi^{\base}}_{s_{\init}} \end{subarray}}  \mbr{ \sbr{\phi(\tau^{(1)}) - \phi(\tau^{(2)})} \sbr{\phi(\tau^{(1)}) - \phi(\tau^{(2)})}^\top } + \zeta_{\hf} I }^{-1} \phi(\tau) }
		\\
		&\overset{(a)}{\leq} \frac{1}{c_{\base}} \sum_{n=0}^{N-1} \ex_{\tau \sim \cO^{\pi^{\base}}_{s_{\init}}} \mbr{ \phi(\tau)^\top \sbr{ \sum_{i=1}^{n} \ex_{\tau^{(2)} \sim \cO^{\pi^{\base}}_{s_{\init}}}  \mbr{ \phi(\tau^{(2)}) \phi(\tau^{(2)})^\top } + \zeta_{\hf} I }^{-1} \phi(\tau) } 
		\\
		&= \frac{1}{c_{\base}} \sum_{n=1}^{N-1} \frac{1}{n}  \trace\sbr{ \sbr{ \ex_{\tau^{(2)} \sim \cO^{\pi^{\base}}_{s_{\init}}}  \mbr{ \phi(\tau^{(2)}) \phi(\tau^{(2)})^\top } + \frac{\zeta_{\hf}}{n} I }^{-1} \ex_{\tau \sim \cO^{\pi^{\base}}_{s_{\init}}} \mbr{ \phi(\tau) \phi(\tau)^\top } } + \frac{W_{\tau}^2}{c_{\base} \zeta_{\hf}}
		\\
		&= \frac{d}{c_{\base}} \sum_{n=1}^{N-1} \frac{1}{n} + \frac{W_{\tau}^2}{c_{\base} \zeta_{\hf}}
		\\
		&\leq \frac{d}{c_{\base}} \sbr{\log(N) + 1} + \frac{1}{c_{\base}} 
		\\
		&\overset{\textup{(b)}}{\leq} \frac{2d}{c_{\base}} \log(N) .
	\end{align*}
	where inequality (a) uses Assumption~\ref{assumption:baseline_policy}, and inequality (b) holds if $\log(N)\geq2$ which can be easily guaranteed in our problem.
\end{proof}

\begin{lemma}[Elliptical Potential for Preference-based Data] \label{lemma:sum_matrix_norm}
	It holds that
	\begin{align*}
		\frac{1}{N} \sum_{n=0}^{N-1} \sbr{ \frac{1}{T} \sum_{t=0}^{T-1} \sqrt{ \ex_{\tau \sim \cO^{\pi^t}_{\rho^n_{\cover}}} \mbr{\nbr{\sum_{h=0}^{H(\tau)} \phi(s_h,a_h)}_{\sbr{\Sigma_{\hf}^{n}}^{-1}} } } }
		\leq 2 d^{\frac{1}{4}} \log^{\frac{1}{4}}\sbr{ 1+ \frac{4 N W_{\tau}^2}{\zeta_{\hf} d} } + \frac{2 d^{\frac{1}{4}} \log^{\frac{1}{4}}(N) }{c_{\base}^{\frac{1}{4}}} .
	\end{align*}
\end{lemma}
\begin{proof}
	For any phase $n \geq 0$, we have
	\begin{align*}
		&\ \quad \frac{1}{T} \sum_{t=0}^{T-1} \sqrt{ \ex_{\tau \sim \cO^{\pi^t}_{\rho^n_{\cover}}} \mbr{\nbr{\sum_{h=0}^{H(\tau)} \phi(s_h,a_h)}_{\sbr{\Sigma_{\hf}^{n}}^{-1}} } }
		\\
		&\leq \frac{1}{T} \sqrt{ T \cdot \sum_{t=0}^{T-1} \ex_{\tau \sim \cO^{\pi^t}_{\rho^n_{\cover}}} \mbr{\nbr{\sum_{h=0}^{H(\tau)} \phi(s_h,a_h)}_{\sbr{\Sigma_{\hf}^{n}}^{-1}} } }
		\\
		&= \frac{1}{T} \sqrt{ T^2 \cdot \frac{1}{T} \sum_{t=0}^{T-1} \ex_{\tau \sim \cO^{\pi^t}_{\rho^n_{\cover}}} \mbr{\nbr{\sum_{h=0}^{H(\tau)} \phi(s_h,a_h)}_{\sbr{\Sigma_{\hf}^{n}}^{-1}} } }
		\\
		&= \sqrt{ \ex_{\tau \sim \cO^{\pi^{n+1}}_{\rho^n_{\cover}}} \mbr{\nbr{\sum_{h=0}^{H(\tau)} \phi(s_h,a_h)}_{\sbr{\Sigma_{\hf}^{n}}^{-1}} } } .
	\end{align*}
	
	We make the convention that $n\Sigma^n_{\hf}:=\zeta_{\hf}I$ for $n=0$. 
	Then, we obtain
	\begin{align*}
		&\ \quad \frac{1}{N} \sum_{n=0}^{N-1} \sbr{ \frac{1}{T} \sum_{t=0}^{T-1} \sqrt{ \ex_{\tau \sim \cO^{\pi^t}_{\rho^n_{\cover}}} \mbr{\nbr{\sum_{h=0}^{H(\tau)} \phi(s_h,a_h)}_{\sbr{\Sigma_{\hf}^{n}}^{-1}} } } }
		\\
		&\leq \frac{1}{N} \sum_{n=0}^{N-1} \sqrt{ \ex_{\tau \sim \cO^{\pi^{n+1}}_{\rho^n_{\cover}}} \mbr{\sqrt{\sbr{\sum_{h=0}^{H(\tau)} \phi(s_h,a_h)}^{\top} \sbr{\Sigma_{\hf}^{n}}^{-1} \sbr{\sum_{h=0}^{H(\tau)} \phi(s_h,a_h)} }} } 
		\\
		&\overset{\textup{(a)}}{\leq} \frac{1}{N} \sum_{n=0}^{N-1} \sqrt{   \sqrt{ \ex_{\tau \sim \cO^{\pi^{n+1}}_{\rho^n_{\cover}}} \mbr{\sbr{\sum_{h=0}^{H(\tau)} \phi(s_h,a_h)}^{\top} \sbr{\Sigma_{\hf}^{n}}^{-1} \sbr{\sum_{h=0}^{H(\tau)} \phi(s_h,a_h)} }} } 
		\\
		&\leq \frac{1}{N} \sqrt{ N \cdot \sum_{n=0}^{N-1}   \sqrt{\ex_{\tau \sim \cO^{\pi^{n+1}}_{\rho^n_{\cover}}} \mbr{\sbr{\sum_{h=0}^{H(\tau)} \phi(s_h,a_h)}^{\top} \sbr{\Sigma_{\hf}^{n}}^{-1} \sbr{\sum_{h=0}^{H(\tau)} \phi(s_h,a_h)} } } } 
		\\
		&\leq \frac{1}{\sqrt{N}} \sqrt{    \sqrt{ N \cdot \sum_{n=0}^{N-1} \ex_{\tau \sim \cO^{\pi^{n+1}}_{\rho^n_{\cover}}} \mbr{\sbr{\sum_{h=0}^{H(\tau)} \phi(s_h,a_h)}^{\top} \sbr{\Sigma_{\hf}^{n}}^{-1} \sbr{\sum_{h=0}^{H(\tau)} \phi(s_h,a_h)} } } } 
		\\
		&= N^{-\frac{1}{4}} \sbr{ \sum_{n=1}^{N-1} \ex_{\tau \sim \cO^{\pi^{n+1}}_{\rho^n_{\cover}}} \mbr{\sbr{\sum_{h=0}^{H(\tau)} \phi(s_h,a_h)}^{\top} \sbr{\Sigma_{\hf}^{n}}^{-1} \sbr{\sum_{h=0}^{H(\tau)} \phi(s_h,a_h)} } }^{\frac{1}{4}} 
		\\
		&\leq \sbr{ \sum_{n=0}^{N-1} \ex_{\tau \sim \cO^{\pi^{n+1}}_{\rho^n_{\cover}}} \mbr{\sbr{\sum_{h=0}^{H(\tau)} \phi(s_h,a_h)}^{\top} \sbr{n\Sigma_{\hf}^{n}}^{-1} \sbr{\sum_{h=0}^{H(\tau)} \phi(s_h,a_h)} } }^{\frac{1}{4}}
		\\
		&= \sbr{ \sum_{n=0}^{N-1} \ex_{\tau \sim \cO^{\pi^{n+1}}_{\rho^n_{\cover}}} \mbr{\nbr{\sum_{h=0}^{H(\tau)} \phi(s_h,a_h)}^2_{\sbr{n\Sigma_{\hf}^{n}}^{-1} } }  }^{\frac{1}{4}} .
	\end{align*}
	where inequality (a) uses the Jensen inequality.

	It holds that
	\begin{align*}
		&\quad \sum_{n=0}^{N-1} \ex_{\tau \sim \cO^{\pi^{n+1}}_{\rho^n_{\cover}}} \mbr{\nbr{\sum_{h=0}^{H(\tau)} \phi(s_h,a_h)}^2_{\sbr{n\Sigma_{\hf}^{n}}^{-1} } } 
		\\
		&= \sum_{n=0}^{N-1} \ex_{\begin{subarray}{l} \tau^{(1)} \sim \cO^{\pi^{n+1}}_{\rho^n_{\cover}}\\ \tau^{(2)} \sim \cO^{\pi^{\base}}_{s_{\init}} \end{subarray}} \mbr{\nbr{ \phi(\tau^{(1)}) - \phi(\tau^{(2)}) + \phi(\tau^{(2)}) }^2_{\sbr{n\Sigma_{\hf}^{n}}^{-1} } }
		\\
		&\leq \sum_{n=0}^{N-1} \Bigg( 2\ex_{\begin{subarray}{l} \tau^{(1)} \sim \cO^{\pi^{n+1}}_{\rho^n_{\cover}}\\ \tau^{(2)} \sim \cO^{\pi^{\base}}_{s_{\init}} \end{subarray}} \mbr{\nbr{ \phi(\tau^{(1)}) - \phi(\tau^{(2)}) }^2_{\sbr{n\Sigma_{\hf}^{n}}^{-1} } } + 2\ex_{\tau^{(2)} \sim \cO^{\pi^{\base}}_{s_{\init}}} \mbr{ \nbr{\phi(\tau^{(2)})}^2_{\sbr{n\Sigma_{\hf}^{n}}^{-1} } } \Bigg)
		\\
		&\overset{\textup{(a)}}{\leq} 2 \sum_{n=1}^{N}  \!\ex_{\begin{subarray}{l} \tau^{(1)} \sim \cO^{\pi^{n}}_{\rho^{n-1}_{\cover}}\\ \tau^{(2)} \sim \cO^{\pi^{\base}}_{s_{\init}} \end{subarray}} \! \Bigg[\sbr{ \phi(\tau^{(1)}) - \phi(\tau^{(2)}) }^{\!\!\top} \!\! \Bigg( \sum_{i=1}^{n-1} \!\ex_{\begin{subarray}{l} \tau^{(1)} \sim \cO^{\pi^i}_{\rho^{i-1}_{\cover}}\\ \tau^{(2)} \sim \cO^{\pi^{\base}}_{s_{\init}} \end{subarray}} \! \mbr{ \sbr{\phi(\tau^{(1)}) - \phi(\tau^{(2)})} \sbr{\phi(\tau^{(1)}) - \phi(\tau^{(2)})}^\top } \!+\! \zeta_{\hf} I \Bigg)^{\!\!-1} \!\!\! \cdot
		\\
		&\quad\ \sbr{ \phi(\tau^{(1)}) - \phi(\tau^{(2)}) } \Bigg]  + \frac{4d}{c_{\base}} \log(N)
		\\
		&\overset{\textup{(b)}}{\leq} 4 d \log\Bigg( \frac{ \det\bigg(\sum_{i=1}^{N} \ex_{\tau^{(1)} \sim \cO^{\pi^i}_{\rho^{i-1}_{\cover}}, \tau^{(2)} \sim \cO^{\pi^{\base}}_{s_{\init}}} \mbr{ \sbr{\phi(\tau^{(1)}) - \phi(\tau^{(2)})} \sbr{\phi(\tau^{(1)}) - \phi(\tau^{(2)})}^\top } + \zeta_{\hf} I \bigg) }{ \det\sbr{ \zeta_{\hf} I } } \Bigg) \!+\! \frac{4d}{c_{\base}} \log(N) 
		\\
		&\leq 4 d \log\sbr{ 1+ \frac{4 N W_{\tau}^2}{\zeta_{\hf} d} } + \frac{4d}{c_{\base}} \log(N) .
	\end{align*}
	Here inequality (a) uses Lemma~\ref{lemma:phi_tau_base} and Assumption~\ref{assumption:baseline_policy}. Inequality (b) follows from the elliptical potential lemma (Lemma~\ref{lemma:tech_abbasi2011}) and the fact that $\zeta_{\hf}:=4W_{\tau}^2$.
	
	Therefore, we have
	\begin{align*}
		&\quad\ \frac{1}{N} \sum_{n=0}^{N-1} \sbr{ \frac{1}{T} \sum_{t=0}^{T-1} \sqrt{ \ex_{\tau \sim \cO^{\pi^t}_{\rho^n_{\cover}}} \mbr{\nbr{\sum_{h=0}^{H(\tau)} \phi(s_h,a_h)}_{\sbr{\Sigma_{\hf}^{n}}^{-1}} } } } 
		\\
		&\leq
		\sbr{ 4 d \log\sbr{ 1+ \frac{4 N W_{\tau}^2}{\zeta_{\hf} d} } + \frac{4d}{c_{\base}} \log(N)  }^{\frac{1}{4}} 
		\\
		&\leq
		2 d^{\frac{1}{4}} \log^{\frac{1}{4}}\sbr{ 1+ \frac{4 N W_{\tau}^2}{\zeta_{\hf} d} } + \frac{2 d^{\frac{1}{4}} \log^{\frac{1}{4}}(N) }{c_{\base}^{\frac{1}{4}}} .
	\end{align*}
\end{proof}

\subsubsection{Discussion on Assumption~\ref{assumption:baseline_policy}} \label{apx:justify_assump_baseline}

Assumption~\ref{assumption:baseline_policy} can be abstracted from the RLHF framework, and serve as a technical condition for an independent mathematical problem.

We provide Lemma~\ref{lemma:justify_assump_baseline} to demonstrate that under Assumption~\ref{assumption:baseline_policy}, one can systematically utilize the elliptical potential lemma~\cite{abbasi2011improved} to obtain a mathematical conclusion that is independent of the RLHF framework.

Our RLHF analysis (Lemmas~\ref{lemma:phi_tau_base} and \ref{lemma:sum_matrix_norm}) is an application of this systematical analytical procedure.

\begin{lemma} \label{lemma:justify_assump_baseline}
	Let $\Phi:=\{\phi \in \R^d: \|\phi\|_2 \leq W_{\phi}\}$. There are random distributions $\cD_1,\dots,\cD_{N}$ and $\cD_{\base}$ over $\Phi$, and a regularization parameter $\zeta \geq W_{\phi}^2$.
	
	Assume that $\cD_{\base}$ satisfies that for any $n \in [N]$,
	\begin{align}
		\ex_{\phi \sim \cD_{n}, \phi' \sim \cD_{\base}}\mbr{(\phi-\phi') (\phi-\phi')^\top} \succeq c_{\base} \ex_{\phi' \sim \cD_{\base}}\mbr{\phi' \phi'^\top} \label{eq:abstract_assump_baseline}
	\end{align}
	for some constant $c_{\base} \in (0,1)$.
	
	Then, we can use the elliptical potential lemma (Lemma~\ref{lemma:tech_abbasi2011})~\cite{abbasi2011improved} to bound
	\begin{align*}
		&\quad\ \sum_{n=1}^{N} \ex_{\phi_n \sim \cD_n, \phi'_n \sim \cD_{\base}} \mbr{ \nbr{\phi_n}^2_{ \sbr{\sum_{i=1}^{n-1} \ex_{\phi_i \sim \cD_i, \phi'_i \sim \cD_{\base}}\mbr{ (\phi_i - \phi'_i) (\phi_i - \phi'_i)^\top } + \zeta I }^{-1} } } 
		\\
		&\leq 2 \sum_{n=1}^{N} \bigg( \ex_{\phi_n \sim \cD_n, \phi'_n \sim \cD_{\base}} \mbr{  \nbr{ \phi_n - \phi'_n }^2_{ \sbr{\sum_{i=1}^{n-1} \ex_{\phi_i \sim \cD_i, \phi'_i \sim \cD_{\base}}\mbr{ (\phi_i - \phi'_i) (\phi_i - \phi'_i)^\top } + \zeta I }^{-1} } } 
		\\
		&\quad\ + \frac{2}{c_{\base}} \ex_{\phi'_n \sim \cD_{\base}}\mbr{ \nbr{ \phi'_n }^2_{ \sbr{\sum_{i=1}^{n-1} \ex_{\phi'_i \sim \cD_{\base}}\mbr{ \phi'_i (\phi'_i)^\top } + \zeta I }^{-1} } } \bigg)
		\\
		&\leq 4 d \log\sbr{1+\frac{N W_{\phi}^2}{\zeta d}} +  \frac{2d}{c_{\base}} \sbr{\log(N) + 2} . 
	\end{align*}
\end{lemma}

\begin{proof}
According to the assumption Eq.~\eqref{eq:abstract_assump_baseline}, we have that for any $i \in [N]$,
\begin{align*}
	\ex_{\phi \sim \cD_{i}, \phi' \sim \cD_{\base}}\mbr{(\phi-\phi') (\phi-\phi')^\top} + \zeta I \succeq c_{\base} \ex_{\phi' \sim \cD_{\base}}\mbr{\phi' \phi'^\top} + \zeta I \succeq c_{\base} \ex_{\phi' \sim \cD_{\base}}\mbr{\phi' \phi'^\top} + c_{\base} \zeta I ,
\end{align*}
which implies that
\begin{align*}
	\sbr{\ex_{\phi \sim \cD_{i}, \phi' \sim \cD_{\base}}\mbr{(\phi-\phi') (\phi-\phi')^\top} + \zeta I}^{-1} \preceq \frac{1}{c_{\base}} \sbr{\ex_{\phi' \sim \cD_{\base}}\mbr{\phi' \phi'^\top} + \zeta I}^{-1} .
\end{align*}
Hence, we have that for any $v \in \R^d$,
\begin{align*}
	v^\top \sbr{\ex_{\phi \sim \cD_{i}, \phi' \sim \cD_{\base}}\mbr{(\phi-\phi') (\phi-\phi')^\top} + \zeta I}^{-1} v 
	\leq 
	\frac{1}{c_{\base}} v^\top \sbr{\ex_{\phi' \sim \cD_{\base}}\mbr{\phi' \phi'^\top} + \zeta I}^{-1} v .
\end{align*}

Furthermore, we have

\begin{align*}
	& \quad\ \sum_{n=1}^{N} \ex_{\phi_n \sim \cD_n, \phi'_n \sim \cD_{\base}} \mbr{ \nbr{\phi_n}^2_{ \sbr{\sum_{i=1}^{n-1} \ex_{\phi_i \sim \cD_i, \phi'_i \sim \cD_{\base}}\mbr{ (\phi_i - \phi'_i) (\phi_i - \phi'_i)^\top } + \zeta I }^{-1} } }
	\\
	&= \sum_{n=1}^{N} \ex_{\phi_n \sim \cD_n, \phi'_n \sim \cD_{\base}} \mbr{ \nbr{ \phi_n - \phi'_n + \phi'_n }^2_{ \sbr{\sum_{i=1}^{n-1} \ex_{\phi_i \sim \cD_i, \phi'_i \sim \cD_{\base}}\mbr{ (\phi_i - \phi'_i) (\phi_i - \phi'_i)^\top } + \zeta I }^{-1} } }
	\\
	&\leq \sum_{n=1}^{N} \bigg( 2\ex_{\phi_n \sim \cD_n, \phi'_n \sim \cD_{\base}} \mbr{ \nbr{ \phi_n - \phi'_n }^2_{ \sbr{\sum_{i=1}^{n-1} \ex_{\phi_i \sim \cD_i, \phi'_i \sim \cD_{\base}}\mbr{ (\phi_i - \phi'_i) (\phi_i - \phi'_i)^\top } + \zeta I }^{-1} } } 
	\\
	&\quad\ + 2\ex_{\phi'_n \sim \cD_{\base}} \mbr{ \nbr{ \phi'_n }^2_{ \sbr{\sum_{i=1}^{n-1} \ex_{\phi_i \sim \cD_i, \phi'_i \sim \cD_{\base}}\mbr{ (\phi_i - \phi'_i) (\phi_i - \phi'_i)^\top } + \zeta I }^{-1} } } \bigg)
	\\
	&\overset{\textup{(a)}}{\leq} 2 \sum_{n=1}^{N} \bigg( \ex_{\phi_n \sim \cD_n, \phi'_n \sim \cD_{\base}} \mbr{  \nbr{ \phi_n - \phi'_n }^2_{ \sbr{\sum_{i=1}^{n-1} \ex_{\phi_i \sim \cD_i, \phi'_i \sim \cD_{\base}}\mbr{ (\phi_i - \phi'_i) (\phi_i - \phi'_i)^\top } + \zeta I }^{-1} } } 
	\\
	&\quad\ + \frac{2}{c_{\base}} \ex_{\phi'_n \sim \cD_{\base}}\mbr{ \nbr{ \phi'_n }^2_{ \sbr{\sum_{i=1}^{n-1} \ex_{\phi'_i \sim \cD_{\base}}\mbr{ \phi'_i (\phi'_i)^\top } + \zeta I }^{-1} } } \bigg)
	\\
	&= 2 \sum_{n=1}^{N} \trace\sbr{ \sbr{\sum_{i=1}^{n-1} \ex_{\phi_i \sim \cD_i, \phi'_n \sim \cD_{\base}}\mbr{ (\phi_i - \phi'_i) (\phi_i - \phi'_i)^\top } + \zeta I }^{-1}  \ex_{\phi_n \sim \cD_n, \phi'_n \sim \cD_{\base}} \mbr{ (\phi_n - \phi'_n) (\phi_n - \phi'_n)^\top  }  }
	\\
	&\quad\ +  \frac{2}{c_{\base}} \sum_{n=2}^{N} \frac{1}{n-1} \cdot \trace\sbr{\sbr{ \ex_{\phi' \sim \cD_{\base}}\mbr{ \phi' (\phi')^\top } + \frac{\zeta}{n-1} I }^{-1} \ex_{\phi'_n \sim \cD_{\base}} \mbr{ \phi'_n (\phi'_n)^\top } }  + \frac{2W_{\phi}^2}{c_{\base}\zeta}
	\\
	&\overset{\textup{(b)}}{\leq} 4 d \log\sbr{1+\frac{N W_{\phi}^2}{\zeta d}} +  \frac{2d}{c_{\base}} \sum_{n=2}^{N} \frac{1}{n-1} + \frac{2W_{\phi}^2}{c_{\base}\zeta}
	\\
	&\leq 4 d \log\sbr{1+\frac{N W_{\phi}^2}{\zeta d}} +  \frac{2d}{c_{\base}} \sbr{\log(N) + 2} ,
\end{align*}
where inequality (a) uses Assumption~\ref{assumption:baseline_policy}, and inequality (b) applies the elliptical potential lemma (Lemma~\ref{lemma:tech_abbasi2011})~\cite{abbasi2011improved}.

\end{proof}

\subsection{Proof of Theorem~\ref{thm:ub_pgrlhf}}\label{apx:main_thm_proof_linear}

For any phase $n=0,\dots,N-1$ and iteration $t=0,\dots,T-1$, define
\begin{align*}
	\theta^{t}_{*} &:= \argmin_{\|\theta\|_2 \leq W_{\theta}} \ex_{(s,a) \sim \rho^n_{\cover}} \mbr{ \sbr{ \phi(s,a)^\top \theta - \sbr{ Q^{\pi^t}(s,a;r+b^n)- b^n(s,a) } }^2 } ,
	\\
	\theta^{t}_{\submid} &:= \argmin_{\|\theta\|_2 \leq W_{\theta}} \ex_{(s,a) \sim \rho^n_{\cover}} \mbr{ \sbr{ \phi(s,a)^\top \theta - \sbr{ Q^{\pi^t}(s,a;\hat{r}^n+b^n)- b^n(s,a) } }^2 } ,
	\\
	\theta^{t} &\overset{\subsgd}{\approx} \argmin_{\|\theta\|_2 \leq W_{\theta}} \ex_{(s,a) \sim \rho^n_{\cover}} \mbr{ \sbr{ \phi(s,a)^\top \theta - \sbr{ Q^{\pi^t}(s,a;\hat{r}^n+b^n)- b^n(s,a) } }^2 } .
\end{align*}

For any $n\geq0$, $t\geq0$ and $(s,a) \in \cS \times \cA$, let $\bar{b}^{n,t}(s,a) := b^n(s,a) - \ex_{a' \sim \pi^t(\cdot|s)} \mbr{ b^n(s,a') }$ and $\bar{\phi}^{t}(s,a) := \phi(s,a) - \ex_{a' \sim \pi^t(\cdot|s)} \mbr{ \phi(s,a') }$.

\begin{proof}[Proof of Theorem~\ref{thm:ub_pgrlhf}]
	Using Lemma~\ref{lemma:modified_perf_diff}, we have that for any phase $n=0,\dots,N-1$ and iteration $t=0,\dots,T-1$,
	\begin{align}
		&\ \quad V_{\cM^n}^{\pi^{*,n}}(s_{\init}) - V_{\cM^n}^{\pi^t}(s_{\init}) 
		\nonumber\\
		&\leq \frac{1}{1-\gamma} \ex_{(s,a) \sim d_{\cM^n;s_{\init}}^{\pi^{*,n}} }\mbr{ A_{\cM_{b^n}}^{\pi^t}(s,a) \cdot \indicator{s\in\cK^n} }
		\nonumber\\
		&= \frac{1}{1-\gamma} \ex_{(s,a) \sim d_{\cM^n;s_{\init}}^{\pi^{*,n}}} \bigg[  \hat{A}^{\pi^t}_{\cM_{b^n}}(s,a) \cdot \indicator{s \in \cK^n}  
		\nonumber\\&\ \quad + \underbrace{\sbr{ A_{\cM_{b^n}}^{\pi^t}(s,a)  - \sbr{ \bar{\phi}^t(s,a)^\top \theta^{t}_{*} + \bar{b}^{n,t}(s,a)  } } \cdot \indicator{s \in \cK^n} }_{\textup{Term 1}}
		\nonumber\\&\ \quad +  \underbrace{  \bar{\phi}^t(s,a)^\top \sbr{\theta^{t}_{*} - \theta^{t}_{\submid}} \cdot \indicator{s \in \cK^n}}_{\textup{Term 2}}
		+  \underbrace{  \bar{\phi}^t(s,a)^\top \sbr{\theta^{t}_{\submid} - \theta^{t} } \cdot \indicator{s \in \cK^n}}_{\textup{Term 3}} 
		\bigg] . \label{eq:regret_decomposition}
	\end{align}
	

	Following the proof of Lemma D.1 in \cite{agarwal2020pc}, we can bound Terms 1 and 3 as follows.
	\begin{align}
		\textup{Term 1} &= \ex_{(s,a) \sim d_{\cM^n;s_{\init}}^{\pi^{*,n}} }\mbr{ \sbr{ A_{\cM_{b^n}}^{\pi^t}(s,a)  - \sbr{ \bar{\phi}^t(s,a)^\top \theta^{t}_{*} + \bar{b}^{n,t}(s,a)  } } \cdot \indicator{s \in \cK^n}  } 
		\nonumber\\
		&= \ex_{(s,a) \sim d_{\cM^n;s_{\init}}^{\pi^{*,n}} }\mbr{ \sbr{ Q_{\cM_{b^n}}^{\pi^t}(s,a)  - \sbr{ \phi(s,a)^\top \theta^{t}_{*} + b(s,a)  } } \cdot \indicator{s \in \cK^n}  } 
		\nonumber\\
		&\quad\ + \ex_{s \sim d_{\cM^n;s_{\init}}^{\pi^{*,n}}, a'\sim \pi^t(\cdot|s) }\mbr{ \sbr{ Q_{\cM_{b^n}}^{\pi^t}(s,a')  - \sbr{ \phi(s,a)^\top \theta^{t}_{*} + b(s,a')  } } \cdot \indicator{s \in \cK^n}  } 
		\nonumber\\
		&\overset{\textup{(a)}}{\leq} \sqrt{\ex_{(s,a) \sim d_{s_{\init}}^{\pi^{*}} } \mbr{ \sbr{Q_{\cM_{b^n}}^{\pi^t}(s,a)  - \sbr{ \phi(s,a)^\top \theta^{t}_{*} + b^{n}(s,a) } }^2 } } 
		\nonumber\\&\quad\ 
		+ \sqrt{\ex_{s \sim d_{s_{\init}}^{\pi^{*}}, a' \sim \pi^t(\cdot|s) } \mbr{ \sbr{ Q_{\cM_{b^n}}^{\pi^t}(s,a') - \sbr{ \phi(s,a')^\top \theta^{t}_{*} + b^{n}(s,a') } }^2 } }
		\nonumber\\
		&\leq 2\sqrt{ |\cA| \ex_{(s,a) \sim d_{s_{\init}}^{\star} } \mbr{ \sbr{Q_{\cM_{b^n}}^{\pi^t}(s,a)  - \sbr{ \phi(s,a)^\top \theta^{t}_{*} + b^{n}(s,a) } }^2 } } 
		\nonumber\\
		&\leq 2 \sqrt{|\cA| \varepsilon_{\bias}} , \label{eq:term_bias}
	\end{align}
	where inequality (a) uses Lemma~\ref{lemma:d_star_n_leq_d_star}.

	Define the Q-value function fitting error as
	\begin{align*}
		\varepsilon_{\stat}&:= 8W_{Q}^2 \sqrt{ \frac{ \log\sbr{\frac{1}{\delta'}} }{M_{\subsgd}} } .
	\end{align*}
	
	With probability at least $1-2NT\delta'$, 
	\begin{align}
		\textup{Term 3} &= \ex_{(s,a) \sim d^{\pi^*}_{s_{\init}}} \mbr{ \bar{\phi}^t(s,a)^\top \sbr{\theta^{t}_{\submid} - \theta^{t} } \cdot \indicator{s \in \cK^n} }
		\nonumber\\
		&\leq 2 \sqrt{\beta \zeta_{\cover} W_{\theta}^2 + \beta (n+1) \varepsilon_{\stat}} 
		\nonumber\\
		&= 2 \sqrt{ \beta \zeta_{\cover} W_{\theta}^2 + 8 \beta W_{Q}^2  (n+1) \sqrt{ \frac{ \log\sbr{\frac{1}{\delta'}} }{M_{\subsgd}} } }
		\nonumber\\
		&\leq 2 W_{\theta} \sqrt{ \beta \zeta_{\cover} } + 4 W_{Q} \sqrt{ \beta  (n+1) } \sbr{ \frac{ \log\sbr{\frac{1}{\delta'}} }{M_{\subsgd}} }^{\frac{1}{4}}
		. \label{eq:term_stat}
	\end{align}

	Define event 
	\begin{align*}
		\cE_{\theta}:=\lbr{\textup{Term 3}\leq 2 \sqrt{\beta \zeta_{\cover} W_{\theta}^2 + \beta (n+1) \varepsilon_{\stat}}} .
	\end{align*}
	Then, $\Pr[\cE_{\theta}]\geq 1-2NT\delta'$.
	
	Now we have $\Pr[\cE_{\theta}\cap\cE_{\tau}\cap\cE_{\mle}\cap\cE_{\cover}]\geq 1-4 \cdot 2N(K+M_{\hf}+TM_{\subsgd}) \cdot 2\delta' \geq 1-\delta$.
	In the following, we assume that event $\cE_{\theta}\cap\cE_{\tau}\cap\cE_{\mle}\cap\cE_{\cover}$ holds, and derive the suboptimality guarantee.

	Applying Lemma~\ref{lemma:phi_theta_star_minus_theta_mid}, Term 2 can be bounded as follows.
	\begin{align}
		\textup{Term 2} &= \ex_{(s,a) \sim d_{\cM^n;s_{\init}}^{\pi^{*,n}} }\mbr{ \bar{\phi}^t(s,a)^\top \sbr{\theta^{t}_{*} - \theta^{t}_{\submid} } \cdot \indicator{s \in \cK^n} }
		\nonumber\\
		&\leq \ex_{(s,a) \sim d_{\cM^n;s_{\init}}^{\pi^{*,n}} }\mbr{ \abr{\phi(s,a)^\top \sbr{\theta^{t}_{*} - \theta^{t}_{\submid} }} \cdot \indicator{s \in \cK^n} } 
		\nonumber\\
		&\quad\ + \ex_{s \sim d_{\cM^n;s_{\init}}^{\pi^{*,n}}, a' \sim \pi^t(\cdot|s) }\mbr{ \abr{\phi(s,a')^\top \sbr{\theta^{t}_{*} - \theta^{t}_{\submid} }} \cdot \indicator{s \in \cK^n} }  
		\nonumber\\
		&\leq 16\sqrt{ \beta W_{Q} \varepsilon^n_{\hf} (n+1) \ex_{\tau \sim \cO^{\pi^t}_{\rho^n_{\cover}}} \mbr{\nbr{\sum_{h=0}^{H(\tau)} \phi(s_h,a_h)}_{\sbr{\Sigma_{\hf}^{n}}^{-1}} } } + 8W_{\theta}\sqrt{\beta \zeta_{\cover}} 
		. \label{eq:term_hf}
	\end{align}

	Plugging the connection result between $\cM^n$ and $\cM$ (Lemma~\ref{lemma:optimistic_M_and_true_M}) into the suboptimality decomposition (Eq.~\eqref{eq:regret_decomposition}), we have
	\begin{align*}
		\ \quad V^{\pi^{*}}(s_{\init}) - V^{\pi^t}(s_{\init})  \leq \textup{RHS in Eq.~\eqref{eq:regret_decomposition}} + \frac{1}{1-\gamma} \sum_{(s,a) \notin \cK^n} d^{\pi^t}_{s_{\init}}(s,a) .
	\end{align*}

	Summing over $t=0,\dots,T-1$ and dividing $T$, we have
	\begin{align}
		&\ \quad V^{\pi^{*}}(s_{\init}) - V^{\pi^{n+1}}(s_{\init}) 
		\nonumber\\
		&= \frac{1}{T} \sum_{t=0}^{T-1} \sbr{V^{\pi^{*}}(s_{\init}) - V^{\pi^t}(s_{\init}) }
		\nonumber\\
		&\leq \frac{1}{T} \sum_{t=0}^{T-1} \textup{RHS in Eq.~\eqref{eq:regret_decomposition}} + \frac{1}{1-\gamma} \sum_{(s,a) \notin \cK^n} d^{\pi^{n+1}}_{s_{\init}}(s,a) 
		\nonumber\\
		&\overset{\textup{(a)}}{\leq} 
		\frac{\log(|\cA|)}{ (1-\gamma) \eta T} + \frac{\eta W_{A}^2}{1-\gamma}  
		+ \frac{2 \sqrt{|\cA| \varepsilon_{\bias}}}{1-\gamma} 
		+ \frac{4 W_{Q} \sqrt{ \beta  (n+1) }}{1-\gamma} \sbr{ \frac{ \log\sbr{\frac{1}{\delta'}} }{M_{\subsgd}} }^{\frac{1}{4}} + \frac{10 W_{\theta} {\sqrt{\beta \zeta_{\cover}}}}{16-\gamma}
		\nonumber\\
		&\ \quad + \frac{16 \sqrt{\beta W_{Q} }}{1-\gamma} \cdot 2 \sbr{\frac{ (n+1)^2 \sbr{d+\log\sbr{\frac{1}{\delta'}}} }{c_{\mle}^2 M_{\hf}} +  2(n+1)\zeta_{\hf} W_{\mu}^2 }^{\frac{1}{4}} \cdot 
		\nonumber\\
		&\ \quad \frac{1}{T} \sum_{t=0}^{T-1} \sqrt{ \ex_{\tau \sim \cO^{\pi^t}_{\rho^n_{\cover}}} \mbr{\nbr{\sum_{h=0}^{H(\tau)} \phi(s_h,a_h)}_{\sbr{\Sigma_{\hf}^{n}}^{-1}} } }
		+ \frac{1}{1-\gamma} \sum_{(s,a) \notin \cK^n} d^{\pi^{n+1}}_{s_{\init}}(s,a) , \label{eq:regrer_decomposition_n}
	\end{align}
	where inequality (a) combines the natural policy gradient regret (Lemma~\ref{lemma:regret_npg}) and Terms 1-3 (Eqs.~\eqref{eq:term_bias}-\eqref{eq:term_hf}). 
	
	Summing over $n=0,\dots,N-1$ and dividing $N$, we have
	\begin{align}
		&\ \quad  V^{\pi^{*}}(s_{\init}) - V^{\pi^{\out}}(s_{\init}) 
		\nonumber\\
		&= \frac{1}{N} \sum_{n=0}^{N-1} \sbr{ V^{\pi^{*}}(s_{\init}) - V^{\pi^{n+1}}(s_{\init}) }
		\nonumber\\
		&\leq \frac{\log(|\cA|)}{ (1-\gamma) \eta T} + \frac{\eta W_{A}^2}{1-\gamma}  
		+ \frac{2 \sqrt{|\cA| \varepsilon_{\bias} }}{1-\gamma} 
		+
		\frac{8 W_{Q} \sqrt{ \beta  N }}{1-\gamma} \sbr{ \frac{ \log\sbr{\frac{1}{\delta'}} }{M_{\subsgd}} }^{\frac{1}{4}}  + \frac{10W_{\theta} \sqrt{\beta \zeta_{\cover} }}{1-\gamma}
		\nonumber\\
		&\ \quad +  \frac{32\sqrt{\beta W_{Q} }}{1-\gamma} \cdot 2 \sbr{\frac{ 4 N^2 \sbr{d+\log\sbr{\frac{1}{\delta'}}} }{c_{\mle}^2 M_{\hf}} +  4N\zeta_{\hf} W_{\mu}^2 }^{\frac{1}{4}} \cdot
		\nonumber\\
		&\ \quad
		\frac{1}{N} \sum_{n=0}^{N-1} \frac{1}{T} \sum_{t=0}^{T-1} \sqrt{ \ex_{\tau \sim \cO^{\pi^t}_{\rho^n_{\cover}}} \mbr{\nbr{\sum_{h=0}^{H(\tau)} \phi(s_h,a_h)}_{\sbr{\Sigma_{\hf}^{n}}^{-1}} } }
		\nonumber\\
		&\ \quad
		+ \frac{1}{(1-\gamma) N} \sum_{n=0}^{N-1} \sum_{(s,a) \notin \cK^n} d^{\pi^{n+1}}_{s_{\init}}(s,a) 
		\nonumber\\
		&\overset{\textup{(a)}}{\leq} \frac{\log(|\cA|)}{ (1-\gamma) \eta T} + \frac{\eta W_{A}^2}{1-\gamma}  
		+ \frac{2 \sqrt{|\cA| \varepsilon_{\bias} }}{1-\gamma}
		+
		\frac{8 W_{Q} \sqrt{ \beta  N }}{1-\gamma} \sbr{ \frac{ \log\sbr{\frac{1}{\delta'}} }{M_{\subsgd}} }^{\frac{1}{4}}  + \frac{10 W_{\theta} \sqrt{\beta \zeta_{\cover} }}{1-\gamma}
		\nonumber\\
		&\ \quad + \frac{256\sqrt{\beta W_{Q} }}{1-\gamma} \cdot  \sbr{\frac{  N^2 \sbr{d+\log\sbr{\frac{1}{\delta'}}} }{c_{\mle}^2 M_{\hf}} +  N\zeta_{\hf} W_{\mu}^2 }^{\frac{1}{4}} \cdot \Bigg( d^{\frac{1}{4}} \log^{\frac{1}{4}}\sbr{ 1+ \frac{4 N W_{\tau}^2}{\zeta_{\hf} d} } + \frac{ d^{\frac{1}{4}} \log^{\frac{1}{4}}(N) }{c_{\base}^{\frac{1}{4}}}  \Bigg)
		\nonumber\\
		&\ \quad  
		+ \frac{1}{(1-\gamma) N} \cdot \frac{2 d}{\beta} \log\sbr{ 1+ \frac{N  }{\zeta_{\cover} d} } 
		\nonumber\\
		&\leq \frac{2 \sqrt{|\cA| \varepsilon_{\bias} }}{1-\gamma}
		+ \frac{W_{A} \sqrt{\log(|\cA|)}}{(1-\gamma)\sqrt{T}} 
		+ 
		\frac{8 W_{Q} \sqrt{ \beta  N }}{1-\gamma} \cdot \frac{ \log^{\frac{1}{4}}\sbr{\frac{1}{\delta'}} }{ (M_{\subsgd})^{\frac{1}{4}} }   + \frac{10 W_{\theta} \sqrt{\beta \zeta_{\cover} }}{1-\gamma}
		\nonumber\\
		&\ \quad + \frac{2 \cdot 256 \sqrt{\beta W_{Q} }}{1-\gamma} \cdot \frac{d^{\frac{1}{4}} \log^{\frac{1}{4}}\sbr{ 5N } }{c_{\base}^{\frac{1}{4}}}  \cdot  \sbr{\frac{ 2 N^2 d \log\sbr{\frac{1}{\delta'}} }{c_{\mle}^2 M_{\hf}}  +  N\zeta_{\hf} W_{\mu}^2 }^{\frac{1}{4}} 
		+ \frac{2d}{(1-\gamma) N \beta} \log\sbr{ 2N   }  ,	\label{eq:apx_suboptimality}
	\end{align}
	where inequality (a) uses Lemmas~\ref{lemma:sum_matrix_norm}, \ref{lemma:sum_occupancy_not explore}, and the fact $\eta:=\frac{\sqrt{\log(|\cA|)}}{W_{A}\sqrt{T}}$.
	
	In addition, due to the condition of concentration event $\cE_{\cover}$, we should guarantee $K \geq \frac{16 (N+1)^2  \log^2\sbr{\frac{4dN}{\delta'}} }{\zeta_{\cover}^2}$ and $M_{\hf} \geq \frac{16 W_{\tau}^4 \log^2\sbr{\frac{4d}{\delta'}} }{\zeta_{\hf}^2}$.
	
	Recall that $W_{\tau}:=\frac{\log(\frac{1}{\delta'})}{1-\gamma}$, $W_{\mu}:=1$, $W_{A}:=\frac{4}{(1-\gamma)^2}$, $W_{\theta}:=\frac{2}{(1-\gamma)^2}-\frac{1}{1-\gamma}$, $W_{Q}:=\frac{2}{(1-\gamma)^2}$, $c_{\mle}:=\frac{1}{2+\exp(-2W_{\tau} W_{\mu})+\exp(2W_{\tau} W_{\mu})}$, $\xi:=\frac{W_{\theta}}{(W_{Q}+W_{\theta})\sqrt{T}}$, $\zeta_{\cover}:=1$ and $\zeta_{\hf}:=4W_{\tau}^2=\frac{4\log^2(\frac{1}{\delta'})}{(1-\gamma)^2}$.
	
	We set 
	\begin{align}
		T&:=\frac{6^2 W_{A}^2\log(|\cA|)}{(1-\gamma)^2\varepsilon^2} ,
		\nonumber\\
		\eta&:=\frac{\sqrt{\log(|\cA|)}}{W_{A}\sqrt{T}} = \frac{(1-\gamma)\varepsilon}{ 6 W_{A}^2},
		\nonumber\\
		\beta&:= \frac{(1-\gamma)^5 \varepsilon^5 c_{\base}}{5000 \cdot 6^5 \cdot 2^4 \cdot 256^4 W_Q^2 W_{\mu}^2 \zeta_{\hf} d^2 } \log^{-2}\sbr{ \frac{800 \cdot 256^2 d^3 W_{Q} W_{\mu} \sqrt{10 \zeta_{\hf}} }{ (1-\gamma)^{4.5} \sqrt{c_{\base}} } } 
		= \tilde{O}\sbr{\frac{(1-\gamma)^5 \varepsilon^5 c_{\base}}{ W_Q^2 W_{\mu}^2 d^2 \zeta_{\hf}  } } ,
		\nonumber\\
		N&:= \frac{6 \cdot 10d}{(1-\gamma)\varepsilon\beta}\log\sbr{ \frac{6 \cdot4d}{(1-\gamma)\varepsilon\beta} } 
		= \tilde{O} \sbr{  \frac{d^3 W_{Q}^2 W_{\mu}^2 \zeta_{\hf} }{ (1-\gamma)^6 \varepsilon^6 c_{\base} } } ,
		\nonumber\\
		M_{\subsgd}&:=1200  \cdot \underbrace{ \frac{6^4 \cdot 8^4 \cdot W_Q^4 \beta^2 N^2}{(1-\gamma)^4 \varepsilon^4} }_{:=L_1} \log^2\sbr{ \frac{L_1 L_2^2 L_3}{\delta} } 
		= \tilde{O} \sbr{ \frac{W_{Q}^4 d^2}{ (1-\gamma)^6 \varepsilon^6 } } ,
		\nonumber\\
		M_{\hf}&:= 1200\cdot   \underbrace{ \frac{6^4 \cdot 2^5 \cdot 256^4 \beta^2 W_Q^2 N^2 d^2 \log(5N) }{(1-\gamma)^4 \varepsilon^4 c_{\mle}^2 c_{\base}} }_{:=L_3}  \log^2\sbr{ \frac{L_1 L_2^2 L_3}{\delta} } 
		= \tilde{O} \sbr{ \frac{W_{Q}^2 d^4}{ (1-\gamma)^6 \varepsilon^6 c_{\mle}^2 c_{\base} } } ,
		\nonumber\\
		K &:= \frac{64 N^2   }{\zeta_{\cover}^2} \cdot \log^2\underbrace{ \sbr{  \frac{4dN \cdot 12N(K+1+T) M_{\hf} M_{\subsgd}}{\delta}} }_{:=L_2} 
		= \tilde{O} \sbr{  \frac{d^6 W_{Q}^4 W_{\mu}^4 \zeta_{\hf}^2 }{ (1-\gamma)^{12} \varepsilon^{12} c_{\base}^2 } } ,
		\nonumber\\
		\delta'&:= \frac{\delta}{12N(K+1+T) M_{\hf} M_{\subsgd}} . \label{eq:set_parameter_linear}
	\end{align}

	Then, we have 
	\begin{align*}
		V^{\pi^{*}}(s_{\init}) - V^{\pi^{\out}}(s_{\init}) \leq \varepsilon + \frac{2 \sqrt{|\cA| \varepsilon_{\bias} }}{1-\gamma} .
	\end{align*}
	
	Finally, the number of samples is bounded by
	\begin{align}
		&\quad\ \tilde{O}\sbr{  N \sbr{K+ M_{\hf} + TM_{\subsgd}} \cdot \frac{1}{1-\gamma} }
		\nonumber\\
		&= \tilde{O} \Bigg( \frac{W_Q^2 W_{\mu}^2 \zeta_{\hf} d^3}{ (1-\gamma)^6 \varepsilon^6 c_{\base} } \cdot \bigg( \frac{W_Q^4 W_{\mu}^4 \zeta_{\hf}^2 d^6}{ (1-\gamma)^{12} \varepsilon^{12} c_{\base}^2 } + \frac{W_Q^2d^4}{(1-\gamma)^6 \varepsilon^6 c_{\mle}^2 c_{\base}} + \frac{W_A^2}{(1-\gamma)^2 \varepsilon^2} \cdot \frac{W_Q^4d^2}{(1-\gamma)^6 \varepsilon^6} \bigg) \cdot \frac{1}{1-\gamma} \Bigg)
		\nonumber\\
		&= \tilde{O} \Bigg(  \frac{W_Q^6 W_{\mu}^6 \zeta_{\hf}^3 d^9}{ (1-\gamma)^{19} \varepsilon^{18} c_{\base}^3 } \Bigg) \label{eq:apx_samples}
		.
	\end{align}
	
\end{proof}

\section{Proofs for PO-RLHF with Neural Function Approximation}


In this section, we provide the proofs for algorithm $\algnnrlhf$. 

\paragraph{Definitions for Neural Function Approximation.}

We first introduce or recall some definitions.

Let $\cS_R:=\{ w \in \R^{md}: \nbr{w - w^0}_2 \leq R \}$ and $\cU_R:=\{ \mu \in \R^{md}: \nbr{\mu - \mu^0}_2 \leq R \}$.

For any $w \in \R^{md}$, recall that
\begin{align*}
	[\psi_{w}]_{\ell}(s,a) &:= \frac{b_{\ell}}{\sqrt{m}}   \cdot \indicator{ \phi(s,a)^\top [w]_{\ell} >0 } \phi(s,a) \ \in \R^{d} , \quad \forall \ell \in [m],
	\\
	\psi_{w}(s,a) &:= \mbr{ [\psi_{w}]_1(s,a); \dots; [\psi_{w}]_m(s,a) } \ \in \R^{md} .
\end{align*}

Here  $\underline{c} \leq \|[w^0]_{\ell}\|_2 \leq \bar{c}$ for all $\ell \in [m]$ for some constants $\underline{c},\bar{c}>0$.

Recall the Q-network, policy network and reward network as follow:
\begin{align*}
	f(s,a;\theta) &:= \frac{1}{\sqrt{m}} \sum_{\ell=1}^{m} b_{\ell} \cdot \indicator{ \phi(s,a)^\top [\theta]_{\ell} >0 } \phi(s,a)^\top [\theta]_{\ell} = \psi_{\theta}(s,a)^\top \theta ,
	\\
	\pi_{\alpha,w}(a|s) &:= \frac{ \exp\sbr{\alpha f(s,a;w)} }{ \sum_{a' \in \cA} \exp\sbr{\alpha f(s,a';w)} } = \frac{ \exp(\alpha \psi_{w}(s,a)^\top w) }{ \sum_{a' \in \cA} \exp\sbr{\alpha \psi_{w}(s,a')^\top w} } ,
	\\
	h(s,a;\mu) &:= \frac{1}{\sqrt{m}} \sum_{\ell=1}^{m} b'_{\ell} \cdot \indicator{ \phi(s,a)^\top [\mu]_{\ell} >0 } \phi(s,a)^\top [\mu]_{\ell} = \psi_{\mu}(s,a)^\top \mu .
\end{align*}
For any $t\geq0$, we use $\pi^t$ and $\pi_{\alpha^t,w^t}$ interchangeably.

Let
\begin{align*}
	f_0(s,a;w) &:= \frac{1}{\sqrt{m}} \sum_{\ell=1}^{m} b_{\ell} \cdot \indicator{ \phi(s,a)^\top [w^{0}]_{\ell} >0 } \phi(s,a)^\top [w]_{\ell} = \psi_{w^0}(s,a)^\top w ,
	\\
	h_0(s,a;\mu) &:= \frac{1}{\sqrt{m}} \sum_{\ell=1}^{m} b'_{\ell} \cdot \indicator{ \phi(s,a)^\top [\mu^{0}]_{\ell} >0 } \phi(s,a)^\top [\mu]_{\ell} = \psi_{\mu^0}(s,a)^\top w .
\end{align*}

Define the neural kernel spaces as
\begin{align*}
	\cF^{w}_{R,\infty}:=\lbr{f(s,a)=f(s,a;w^0)+\int \indicator{ \phi(s,a)^\top w >0 } \phi(s,a)^\top \nu^{w}(w) dp^{w}(w) :\ \nbr{ \nu^{w}(w) }_{\infty} \leq \frac{R}{\sqrt{d}} } ,
	\\
	\cF^{\mu}_{R,\infty}:=\lbr{f(s,a)=h(s,a;\mu^0)+\int \indicator{ \phi(s,a)^\top \mu >0 } \phi(s,a)^\top \nu^{\mu}(\mu) dp^{\mu}(\mu) :\ \nbr{ \nu^{\mu}(\mu) }_{\infty} \leq \frac{R}{\sqrt{d}} } ,
\end{align*}
Here $\nu^{w}:\R^d \rightarrow \R^d$ and $f(s,a;w^0)$ parameterize the element of $\cF^{w}_{R,\infty}$, and $p^{w}:\R^d \rightarrow \R$ is the density function of the initialization distribution of $w^0$. 
Similarly, $\nu^{\mu}:\R^d \rightarrow \R^d$ and $h(s,a;\mu^0)$ parameterize the element of $\cF^{\mu}_{R,\infty}$, and $p^{\mu}:\R^d \rightarrow \R$ is the density function of the initialization distribution of $\mu^0$. In this work, for simplicity, we set the initialization distribution for $w^0$ and $\mu^0$ as $\cD_{\init}$.

Define
\begin{align*}
	\cF^{\mu}_{R,m}&:=\lbr{ \frac{1}{\sqrt{m}} \sum_{\ell=1}^{m} b'_{\ell} \cdot \indicator{ \phi(s,a)^\top [\mu^0]_{\ell} >0 } \phi(s,a)^\top [\mu]_{\ell} :\ \nbr{ \mu-\mu^0 }_2 \leq R } ,
	\\
	\bar{\cF}^{\mu}_{R,m}&:=\lbr{ \frac{1}{\sqrt{m}} \sum_{\ell=1}^{m} b'_{\ell} \cdot \indicator{ \phi(s,a)^\top [\mu^0]_{\ell} >0 } \phi(s,a)^\top [\mu]_{\ell} :\ \nbr{ [\mu]_{\ell} - [\mu^0]_{\ell} }_{\infty} \leq \frac{R}{\sqrt{md}} } .
\end{align*}

$\bar{\cF}^{\mu}_{R,m}$ is the subset of $\cF^{\mu}_{R,m}$.

Let $\mu^{\supproj}_{r} \in \cU_R$ be the parameter such that 
\begin{align*}
	\proj_{\cF_{R,m}} r(s,a)=\psi_0(s,a)^\top \mu^{\supproj}_{r} .
\end{align*}

\paragraph{Covariance Matrix Concentration.}

Next, we define the concentration event for the coverage and human data covariance matrices.

For any trajectory $\tau=(s_0,a_0,\dots,s_{H(\tau)},a_{H(\tau)})$ and $\mu \in \cU_R$, let $\psi_{\mu}(\tau):=\sum_{h=0}^{H(\tau)} \psi_{\mu}(s_h,a_h)$.
For any trajectories $\tau^{(1)}, \tau^{(2)}$  and $\mu \in \cU_R$, let $\tilde{\psi}_{\mu}^{\tau^{(1)},\tau^{(2)}}:=\sum_{h=0}^{H(\tau^{(1)})}\psi_{\mu}(s^{(1)}_{h},a^{(1)}_{h}) - \sum_{h=0}^{H(\tau^{(2)})}\psi_{\mu}(s^{(2)}_{h},a^{(2)}_{h})$.

For any $n\geq0$ and $t\geq0$, let $(s^n_i,a^n_i)$ denote the $i$-th state-action pair sampled in phase $n$ for constructing the estimated coverage covariance matrix $\hat{\Sigma}^{\nn,n}_{\cover}$ (Line~\ref{line:nn_sample_cov} in Algorithm~\ref{alg:neural_pg_rlhf}).

For any phase $n \geq 0$, define
\begin{align*}
	\hat{\Sigma}^{\nn,n}_{\cover}&:=\sum_{i=0}^{n}  \sbr{ \frac{1}{K} \sum_{i=1}^{K} \psi_0(s^n_i,a^n_i)\psi_0(s^n_i,a^n_i)^\top } + \zeta_{\cover} I ,
	\\
	\Sigma^{\nn,n}_{\cover}&:=\sum_{i=0}^{n} \ex_{(s,a)\sim d^{\pi^i}_{s_{\init}}} \mbr{\psi_0(s,a)\psi_0(s,a)^\top} + \zeta_{\cover} I 
	\\
	&= (n+1)\ex_{(s,a)\sim \rho^n_{\cover}} \sbr{\psi_0(s,a)\psi_0(s,a)^\top} + \zeta_{\cover} I .
	\\
	\hat{\Sigma}^{\nn,n}_{\hf} &:=\frac{1}{M_{\hf}} \sum_{i=1}^{M_{\hf}} \sbr{ \psi_0(\tau^{(1)}_i) - \psi_0(\tau^{(2)}_i)} \sbr{\psi_0(\tau^{(1)}_i) - \psi_0(\tau^{(2)}_i)}^\top + \frac{\zeta_{\hf}}{n} I
	\\
	&= \frac{1}{M_{\hf}} \sum_{i=1}^{M_{\hf}} \tilde{\psi}^{\tau^{(1)}_i,\tau^{(2)}_i}_0 \sbr{ \tilde{\psi}^{\tau^{(1)}_i,\tau^{(2)}_i}_0 }^\top + \frac{\zeta_{\hf}}{n} I , \quad \forall n\geq1
	\\
	\Sigma^{\nn,n}_{\hf} &:= \frac{1}{n} \sum_{i=1}^{n} \Bigg(  \ex_{\begin{subarray}{l} \tau^{(1)} \sim \cO^{\pi^i}_{\rho^{i-1}_{\cover}}\\ \tau^{(2)} \sim \cO^{\pi^{\base}}_{s_{\init}} \end{subarray}} \mbr{ \sbr{\psi_0(\tau^{(1)}) - \psi_0(\tau^{(2)})} \sbr{\psi_0(\tau^{(1)}) - \psi_0(\tau^{(2)})}^\top } \Bigg) + \frac{\zeta_{\hf}}{n} I 
	\\
	&= \frac{1}{n} \sum_{i=1}^{n} \Bigg(  \ex_{\begin{subarray}{l} \tau^{(1)} \sim \cO^{\pi^i}_{\rho^{i-1}_{\cover}}\\ \tau^{(2)} \sim \cO^{\pi^{\base}}_{s_{\init}} \end{subarray}}  \mbr{ \tilde{\psi}^{\tau^{(1)},\tau^{(2)}}_0 \sbr{\tilde{\psi}^{\tau^{(1)},\tau^{(2)}}_0}^\top } \Bigg) + \frac{\zeta_{\hf}}{n} I , \quad \forall n\geq1
	\\
	\hat{\Sigma}^{\nn,n}_{\hf}&=\Sigma^{\nn,n}_{\hf} := \zeta_{\hf} I .
\end{align*}

Recall  $W_{\tau}:=\frac{\log\sbr{\frac{1}{\delta'}}}{1-\gamma}$ and the definition of event $\cE_{\tau}$ (Eq.~\eqref{eq:def_event_len_traj}).

Define event
\begin{align*}
	\cE^{\nn}_{\cover}:=\Bigg\{&
	\frac{1}{2} \nbr{ \psi_0(s,a) }_{(\Sigma^{\nn,n}_{\cover})^{-1}} \leq \nbr{ \psi_0(s,a) }_{(\hat{\Sigma}^{\nn,n}_{\cover})^{-1}} \leq 2 \nbr{ \psi_0(s,a) }_{(\Sigma^{\nn,n}_{\cover})}^{-1} ,
	\\
	&\frac{1}{2} \nbr{ \psi_0(s,a) }_{(\Sigma^{\nn,n}_{\hf})^{-1}} \leq \nbr{ \psi_0(s,a) }_{(\hat{\Sigma}^{\nn,n}_{\hf})^{-1}} \leq 2 \nbr{ \psi_0(s,a) }_{(\Sigma^{\nn,n}_{\hf})}^{-1} ,\ \forall 0\leq n \leq N-1
	\Bigg\} .
\end{align*}

\begin{lemma}
	Assuming that event $\cE_{\tau}$ holds, then we have $\Pr[\cE^{\nn}_{\cover}] \geq 1-2N\delta'$.
\end{lemma}
\begin{proof}
	This lemma follows from Lemma~\ref{lemma:con_matrix_inverse} and the condition that $K \geq \frac{16 (N+1)^2  \log^2\sbr{\frac{4dN}{\delta'}} }{\zeta_{\cover}^2}$ and $M_{\hf} \geq \frac{16 W_{\tau}^4 \log^2\sbr{\frac{4dN}{\delta'}} }{\zeta_{\hf}^2}$.	
\end{proof}

\subsection{Neural Function Approximation}

In the following, we present useful technical lemmas for neural function approximation.
Lemmas~\ref{lemma：distance_psi_0_psi_theta}-\ref{lemma:distance_mu_proj_r} borrow the ideas from prior neural network theory works~\cite{rahimi2008weighted,cai2019neural,wang2019neural,xu2021crpo}. 

For brevity of presentation, Lemmas~\ref{lemma：distance_psi_0_psi_theta} and \ref{lemma:nn_universal_ub} are written with parameter $w$ and function $f$, but it works for parameters $w,\theta,\mu$ and their corresponding functions $f,h$. 

For ease of notation, we simplify the notations $\psi_{\theta^0}$ and $\psi_{\mu^0}$ as $\psi_{0}$, which can be easily recovered from the context.

\begin{lemma} \label{lemma：distance_psi_0_psi_theta}
	For any $w,w' \in \R^{md}$ such that $\|w-w^0\|_2 \leq R$ and $\|w'-w^0\|_2 \leq R$,\yihan{figure out the distribution for $(s,a)$}
	\begin{align*}
		\ex_{\rho}\mbr{\abr{\psi_{0}(s,a)^\top w' - \psi_{w}(s,a)^\top w'}^2} &\leq \frac{4c_{\scale} R^3}{\underline{c} \sqrt{m}} ,
		\\
		\ex_{\rho}\mbr{\nbr{\psi_{0}(s,a) - \psi_{w}(s,a)}_2^2} &\leq \frac{c_{\scale} R}{\underline{c} \sqrt{m}} .
	\end{align*}
\end{lemma}
\begin{proof}
	We prove the first statement as follows.
	\begin{align*}
		&\quad\ \abr{\psi_{0}(s,a)^\top w' - \psi_{w}(s,a)^\top w'}
		\\
		&= \frac{1}{\sqrt{m}} \sum_{\ell=1}^{m} b_{\ell} \cdot \sbr{\indicator{ \phi(s,a)^\top [w^0]_{\ell} >0 } - \indicator{ \phi(s,a)^\top [w]_{\ell} >0 }} \phi(s,a)^\top [w']_{\ell}
		\\
		&\leq  \frac{1}{\sqrt{m}} \sum_{\ell=1}^{m}  \abr{\indicator{ \phi(s,a)^\top [w^0]_{\ell} >0 } - \indicator{ \phi(s,a)^\top [w]_{\ell} >0 }} \abr{\phi(s,a)^\top [w']_{\ell}} .
	\end{align*}
	
	Since $\abr{\indicator{ \phi(s,a)^\top [w^0]_{\ell} >0 } - \indicator{ \phi(s,a)^\top [w]_{\ell} >0 }}$ implies
	\begin{align*}
		\abr{ \phi(s,a)^\top [w^0]_{\ell} } \leq \abr{ \phi(s,a)^\top [w]_{\ell} - \phi(s,a)^\top [w^0]_{\ell} } \leq \nbr{\phi(s,a)}_2 \nbr{ [w^0]_{\ell} - [w]_{\ell} }_2 ,
	\end{align*}
	we have
	\begin{align}
		\abr{\indicator{ \phi(s,a)^\top [w^0]_{\ell} >0 } - \indicator{ \phi(s,a)^\top [w]_{\ell} >0 }} \leq \indicator{ \abr{ \phi(s,a)^\top [w^0]_{\ell} } \leq
			\nbr{\phi(s,a)}_2 \nbr{ [w^0]_{\ell} - [w]_{\ell} }_2 } . \label{eq:indicator_ub}
	\end{align}
	
	Hence, we have
	\begin{align*}
		&\quad\ \abr{\psi_{0}(s,a)^\top w' - \psi_{w}(s,a)^\top w'}
		\\
		&\leq \frac{1}{\sqrt{m}} \sum_{\ell=1}^{m}  \indicator{ \abr{ \phi(s,a)^\top [w^0]_{\ell} } \leq
			\nbr{\phi(s,a)}_2 \nbr{ [w^0]_{\ell} - [w]_{\ell} }_2 } \abr{\phi(s,a)^\top [w']_{\ell}}
		\\
		&\leq \frac{1}{\sqrt{m}} \sum_{\ell=1}^{m}  \indicator{ \abr{ \phi(s,a)^\top [w^0]_{\ell} } \leq \nbr{\phi(s,a)}_2 \nbr{ [w^0]_{\ell} - [w]_{\ell} }_2 } \sbr{ \abr{\phi(s,a)^\top [w^0]_{\ell}} + \abr{\phi(s,a)^\top \sbr{[w']_{\ell} - [w^0]_{\ell}}}}
		\\
		&\overset{\textup{(a)}}{\leq} \frac{1}{\sqrt{m}} \sum_{\ell=1}^{m}  \indicator{ \abr{ \phi(s,a)^\top [w^0]_{\ell} } \leq \nbr{\phi(s,a)}_2 \nbr{ [w^0]_{\ell} - [w]_{\ell} }_2 } \cdot 
		\\
		&\quad\ \sbr{ \nbr{\phi(s,a)}_2 \nbr{ [w^0]_{\ell} - [w]_{\ell} }_2 + \nbr{\phi(s,a)}_2 \nbr{[w']_{\ell} - [w^0]_{\ell}}_2 } ,
	\end{align*}
	where inequality (a) is due to $\indicator{|x|\leq y}|x| \leq \indicator{|x|\leq y}y$.
	
	Using the Cauchy-Schwartz inequality, we have
	\begin{align*}
		&\quad\ \abr{\psi_{0}(s,a)^\top w' - \psi_{w}(s,a)^\top w'}^2
		\\
		&\leq \frac{1}{m} \sum_{\ell=1}^{m}  \indicator{ \abr{ \phi(s,a)^\top [w^0]_{\ell} } \leq \nbr{\phi(s,a)}_2 \nbr{ [w^0]_{\ell} - [w]_{\ell} }_2 } \cdot 
		\\
		&\quad\ \sum_{\ell=1}^{m} \sbr{ 2\nbr{\phi(s,a)}_2^2 \nbr{ [w^0]_{\ell} - [w]_{\ell} }_2^2 + 2\nbr{\phi(s,a)}_2^2 \nbr{[w']_{\ell} - [w^0]_{\ell}}_2^2 }
		\\
		&\leq \frac{4 R^2}{m} \sum_{\ell=1}^{m}  \indicator{ \abr{ \phi(s,a)^\top [w^0]_{\ell} } \leq \nbr{\phi(s,a)}_2 \nbr{ [w^0]_{\ell} - [w]_{\ell} }_2 } 
		\\
		&\leq \frac{4 R^2}{m} \sum_{\ell=1}^{m}  \indicator{ \abr{ \phi(s,a)^\top [w^0]_{\ell} } \leq  \nbr{ [w^0]_{\ell} - [w]_{\ell} }_2 } .
	\end{align*}
	
	Therefore, we have
	\begin{align*}
		\ex_{\rho}\mbr{\abr{\psi_{0}(s,a)^\top w' - \psi_{w}(s,a)^\top w'}^2} &\leq \frac{4 R^2}{m}  \sum_{\ell=1}^{m}  \ex_{\rho}\mbr{ \indicator{ \abr{ \phi(s,a)^\top [w^0]_{\ell} } \leq  \nbr{ [w^0]_{\ell} - [w]_{\ell} }_2 } }
		\\
		&\overset{\textup{(a)}}{\leq} \frac{4c_{\scale} R^2}{m}  \sum_{\ell=1}^{m}    \frac{  \nbr{ [w^0]_{\ell} - [w]_{\ell} }_2}{ \nbr{[w^0]_{\ell}}_2  } 
		\\
		&\leq \frac{4c_{\scale} R^2}{m}  \sqrt{  \sum_{\ell=1}^{m} \nbr{ [w^0]_{\ell} - [w]_{\ell} }_2^2 } \sqrt{ \sum_{\ell=1}^{m} \frac{ 1 }{ \nbr{[w^0]_{\ell}}_2^2  } } 
		\\
		&\leq \frac{4c_{\scale} R^3}{\underline{c} \sqrt{m}} ,
	\end{align*}
	where inequality (a) uses Assumption~\ref{assumption:scale_theta_0}.
	
	Next, we prove the second statement using the similar argument.
	\begin{align*}
		&\quad\ \nbr{\psi_{0}(s,a) - \psi_{w}(s,a)}_2^2
		\\
		&= \sum_{\ell=1}^{m} \frac{b_{\ell}^2}{m} \cdot \sbr{\indicator{ \phi(s,a)^\top [w^0]_{\ell} >0 } - \indicator{ \phi(s,a)^\top [w]_{\ell} >0 }}^2 \nbr{\phi(s,a)}_2^2 
		\\
		&\overset{\textup{(a)}}{\leq}  \frac{1}{m} \sum_{\ell=1}^{m}  \indicator{ \abr{ \phi(s,a)^\top [w^0]_{\ell} } \leq  \nbr{ [w^0]_{\ell} - [w]_{\ell} }_2 } ,
	\end{align*} 
	where inequality (a) uses Eq.~\eqref{eq:indicator_ub}.
	
	Taking $\ex_{\rho}[\cdot]$, we have
	\begin{align*}
		\ex_{\rho}\mbr{ \nbr{\psi_{0}(s,a) - \psi_{w}(s,a)}_2^2 } &\leq  \frac{1}{m} \sum_{\ell=1}^{m}  \ex_{\rho}\mbr{ \indicator{ \abr{ \phi(s,a)^\top [w^0]_{\ell} } \leq  \nbr{ [w^0]_{\ell} - [w]_{\ell} }_2 }  }
		\\
		&\leq \frac{c_{\scale}}{m}  \sum_{\ell=1}^{m} \frac{  \nbr{ [w^0]_{\ell} - [w]_{\ell} }_2}{ \nbr{[w^0]_{\ell}}_2  } 
		\\
		&\leq \frac{c_{\scale}}{m}  \sqrt{  \sum_{\ell=1}^{m} \nbr{ [w^0]_{\ell} - [w]_{\ell} }_2^2 } \sqrt{ \sum_{\ell=1}^{m} \frac{ 1 }{ \nbr{[w^0]_{\ell}}_2^2  } } 
		\\
		&\leq \frac{c_{\scale} R}{m} \sqrt{ \sum_{\ell=1}^{m}  \frac{ 1 }{ \nbr{[w^0]_{\ell}}_2^2  } }
		\\
		&\leq \frac{c_{\scale} R}{\underline{c} \sqrt{m}} .
	\end{align*}
\end{proof}

\begin{lemma} \label{lemma:nn_universal_ub}
	For any $w \in \cS_R$ and $(s,a) \in \cS \times \cA$,
	\begin{align*}
		\|\psi_{w}(s,a)\|_2 &\leq 1 ,
		\\
		\|w\|_2 &\leq \sqrt{m} \bar{c} + R ,
		\\
		|f(s,a;w)| &\leq \sqrt{m} \bar{c} + R .
	\end{align*}
\end{lemma}
\begin{proof}
	We have
	\begin{align*}
		\|\psi_{w}(s,a)\|_2 &=
		\sqrt{\sum_{\ell=1}^{m}  \|[\psi_{w}(s,a)]_{\ell}\|_2^2}
		= \sqrt{\sum_{\ell=1}^{m} \frac{b_{\ell}^2}{m} \cdot \indicator{ \phi(s,a)^\top [w]_{\ell} >0 } \|\phi(s,a)\|_2^2}
		\leq 1 .
	\end{align*}
	
	In addition,
	\begin{align*}
		\nbr{w^0}_2
		= \sqrt{ \sum_{\ell=1}^{m} \nbr{[w^0]_{\ell}}_2^2}  
		\leq \sqrt{m} \bar{c} .
	\end{align*}
	
	Then,
	\begin{align*}
		\nbr{w}_2 &\leq  \nbr{w^0}_2 +  \nbr{w-w^0}_2 
		\\
		&\leq \sqrt{m} \bar{c} + R .
	\end{align*}
	
	Furthermore, 
	\begin{align*}
		|f(s,a;w)| = |\psi_{w}(s,a)^\top w|
		\leq  \nbr{\psi_{w}(s,a)}_2 \nbr{w}_2
		\leq \sqrt{m} \bar{c} + R .
	\end{align*}
\end{proof}

\begin{lemma}[Projection Error for $\bar{\cF}^{\mu}_{R,m}$~\cite{rahimi2008weighted}] \label{lemma:projection_error_rahimi}
	Let $h \in \cF^{\mu}_{R,\infty}$. For any $\delta'>0$, with probability at least $1-\delta'$,
	\begin{align*}
		\nbr{\proj_{\bar{\cF}_{R,m}}h - h }_{\rho} \leq \frac{R \sbr{1+\sqrt{2\log\sbr{\frac{1}{\delta'}}} } }{\sqrt{m}} ,
	\end{align*}
	where $\rho$ is a distribution over $\cS \times \cA$.
	
\end{lemma}

\begin{lemma}[Distance between $r(s,a)$ and $\psi_0(s,a)^\top \mu^{\supproj}_{r}$] \label{lemma:distance_mu_proj_r}
	Assume that event $\cE_{\init}$ holds. Then,
	\begin{align*}
		\nbr{\psi_0^\top \mu^{\supproj}_{r} - r }_{\rho} \leq 4R \sqrt{\frac{ \log\sbr{\frac{1}{\delta'}}}{m} } .
	\end{align*}
\end{lemma}
\begin{proof}
	Recall that $r \in \cF^{\mu}_{R,\infty}$ and $\proj_{\cF^{\mu}_{R,m}} r(s,a)=\psi_0(s,a)^\top \mu^{\supproj}_{r}$.
	Since $\bar{\cF}^{\mu}_{R,m}$ is a subset of $\cF^{\mu}_{R,m}$, we have
	\begin{align*}
		\nbr{\psi_0^\top \mu^{\supproj}_{r} - r }_{\rho}
		=\nbr{\proj_{\cF_{R,m}}r - r }_{\rho} 
		\leq \nbr{\proj_{\bar{\cF}_{R,m}}r - r }_{\rho} 
		\leq 4R \sqrt{\frac{ \log\sbr{\frac{1}{\delta'}}}{m} } .
	\end{align*}
\end{proof}

Define event
\begin{align*}
	\cE_{\init}:\Bigg\{ &\nbr{\psi_0^\top \mu^{\supproj}_{r} - r }_{d^{\pi^t}_{\rho^{n}_{\cover}}} \leq 4R \sqrt{\frac{ \log\sbr{\frac{1}{\delta'}}}{m} } , \ \forall t \in [T], \ \forall n \in [N]  \Bigg\} .
\end{align*}

\begin{lemma}
	It holds that $\Pr[\cE_{\init}]\geq NT\delta'$.
\end{lemma}
\begin{proof}
	This lemma follows from Lemma~\ref{lemma:distance_mu_proj_r} and a union bound.
\end{proof}

\subsection{Neural Neural Policy Gradient}

Let $W^{\nn}_{\theta}:= \sqrt{m} \bar{c} + R$. 
According to Remark 28 in \cite{agarwal2021theory}, since $\|\psi_0(s,a)\|_2 \leq 1$, $\log(\pi_{\alpha,w})$ is a smooth function with smoothness parameter $W_{S}=1$.

\begin{lemma}[Neural Neural Policy Gradient] \label{lemma:nn_regret_npg}
	For any phase $n \geq 0$ and iteration $t \geq 0$,
	\begin{align*}
		\sum_{t=0}^{T-1} \ex_{(s,a) \sim d_{\cM^n;s_{\init}}^{\pi^{*,n}}}\mbr{ \sbr{ \bar{\psi}^t_{w^t}(s,a)^\top \theta^{t} + \bar{b}^{n,t}(s,a) } \cdot \indicator{s \in \cK^n} } \leq \frac{\log(|\cA|)}{\eta} + \eta W_{S} (W^{\nn}_{\theta})^2 T .
	\end{align*}
\end{lemma}
\begin{proof}
	Following the analysis in \cite{agarwal2021theory}, according to the $W_{S}$-smoothness of $\log(\pi_{\alpha^t,w^t})$, we have
	\begin{align*}
		&\log\sbr{\pi_{\alpha^{t+1},w^{t+1}}(a|s)}-\log\sbr{\pi_{\alpha^{t},w^{t}}(a|s)} 
		\\
		\geq\ & \nabla_{w}\log\sbr{\pi_{\alpha^{t},w^{t}}(a|s)}^\top \sbr{\alpha^{t+1}w^{t+1}-\alpha^{t}w^{t}} - W_{S}\nbr{ \alpha^{t+1}w^{t+1}-\alpha^{t}w^{t} }_2^2 .
	\end{align*}

	For any $s \in \cK^n$, we have
	\begin{align*}
		&\ \quad  \kl(\pi^{*,n}(\cdot|s) \| \pi^{t}(\cdot|s)) - \kl(\pi^{*,n}(\cdot|s) \| \pi^{t+1}(\cdot|s))  
		\\
		&= \ex_{a \sim \pi^{*,n}(\cdot|s)}\mbr{ \log \sbr{\frac{\pi^{*,n}(a|s)}{\pi^{t}(a|s)}} } - \ex_{a \sim \pi^{*,n}(\cdot|s)}\mbr{ \log \sbr{\frac{\pi^{*,n}(a|s)}{\pi^{t+1}(a|s)}} }
		\\
		&= \ex_{a \sim \pi^{*,n}(\cdot|s)}\mbr{ \log \sbr{\frac{\pi^{t+1}(a|s)}{\pi^{t}(a|s)}} }
		\\
		&\geq \ex_{a \sim \pi^{*,n}(\cdot|s)}\mbr{ \nabla_{w}\log\sbr{\pi_{\alpha^{t},w^{t}}(a|s)}^\top \sbr{\alpha^{t+1}w^{t+1}-\alpha^{t}w^{t}} - W_{S}\nbr{ \alpha^{t+1}w^{t+1}-\alpha^{t}w^{t} }_2^2 } 
		\\
		&= \eta \ex_{a \sim \pi^{*,n}(\cdot|s)}\mbr{ \bar{\psi}^t_{w^t}(s,a)^\top  \theta^{t}}  - \eta^2 W_{S}\nbr{ \theta^{t} }_2^2  ,
	\end{align*}
	which is equivalent to
	\begin{align*}
		\ex_{a \sim \pi^{*,n}(\cdot|s)}\mbr{ \bar{\psi}^t_{w^t}(s,a)^\top  \theta^{t} } \leq  \frac{1}{\eta} \sbr{\kl(\pi^{*,n}(\cdot|s) \| \pi^{t}(\cdot|s)) - \kl(\pi^{*,n}(\cdot|s) \| \pi^{t+1}(\cdot|s))} + \eta W_{S}\nbr{ \theta^{t} }_2^2 .
	\end{align*}
	
	For any phase $n \geq 0$, $s \in \cK^n$ and $a \in \cA$, we have $b^n(s,a)=0$, and then $\bar{b}^{n,t}(s,a) := b^n(s,a) - \ex_{a' \sim \pi^t(\cdot|s)} \mbr{ b^n(s,a') }=0$.
	
	Adding $s \sim d_{\cM^n;s_{\init}}^{\pi^{*,n}}$ on both sides and summing over $t=0,\dots,T-1$, we have
	\begin{align*}
		&\quad\ \sum_{t=0}^{T-1} \ex_{(s,a) \sim d_{\cM^n;s_{\init}}^{\pi^{*,n}}}\mbr{ \sbr{\bar{\psi}^t_{w^t}(s,a)^\top  \theta^{t} + \bar{b}^{n,t}(s,a) }\cdot \indicator{s \in \cK^n} }
		\\
		&= \sum_{t=0}^{T-1} \ex_{(s,a) \sim d_{\cM^n;s_{\init}}^{\pi^{*,n}}}\mbr{ \bar{\psi}^t_{w^t}(s,a)^\top  \theta^{t} \cdot \indicator{s \in \cK^n} } 
		\\
		&\leq \frac{1}{\eta} \ex_{s \sim d_{\cM^n;s_{\init}}^{\pi^{*,n}}}\mbr{\kl(\pi^{*,n}(\cdot|s) \| \pi^{0}(\cdot|s)) - \kl(\pi^{*,n}(\cdot|s) \| \pi^{T}(\cdot|s))} + \eta W_{S} (W^{\nn}_{\theta})^2 T
		\\
		&\leq \frac{\log(|\cA|)}{\eta} + \eta W_{S} (W^{\nn}_{\theta})^2 T .
	\end{align*}
\end{proof}

\begin{algorithm}[t]
	\caption{Q-network Training via Projected SGD (with the objective Eq.~\eqref{eq:apx_nn_objective_npg})}
	\label{alg:nn_sgd_Q_fitting}
	\begin{algorithmic}[1]
		\STATE {\bfseries Input:} $f(s,a;w^0)$, $\xi_{\theta}$.
		\FOR{$i=0,\dots,M^{\theta}_{\subsgd}-1$}
		\STATE $g^{t,i} \leftarrow 2 \sbr{ f(s_i,a_i;\theta^{t,i}) - \sbr{\hat{Q}^{\pi^t}(s_i,a_i;\hat{r}^n+b^n)-b^n(s_i,a_i)} } \nabla_{\theta}f(s_i,a_i;\theta^{t,i})$, where $(s_i,a_i) \sim \rho^n_{\cover}$ and $\hat{Q}^{\pi^t}(s_i,a_i;\hat{r}^n+b^n)$ is estimated by Monte Carlo sampling
		\STATE $\tilde{\theta}^{t,i+1} := \theta^{t,i} - \xi_{\theta} g^{t,i}$
		\STATE $\theta^{t,i+1} \leftarrow \proj_{\cU_R} ( \tilde{\theta}^{t,i+1} )$
		\ENDFOR
		\STATE {\bfseries return} $\theta^t = \sum_{i=0}^{M^{\theta}_{\subsgd}-1} \theta^{t,i}$
	\end{algorithmic}
\end{algorithm}

\subsection{Q-value Function Fitting}

For any fixed phase $n=0,\dots,N-1$ and fixed iteration $t=0,\dots,T-1$, define
\begin{align*}
	F^{\hat{r}^n}(\theta)&:= \ex_{(s,a)\sim\rho^n_{\cover}} \mbr{ \sbr{f_0(s,a;\theta) - \sbr{Q^{\pi^t}(s,a;\hat{r}^n+b^n)-b^n(s,a)} }^2 } ,
	\\
	\theta^{t,\hat{r}^n}_{\submid}&:= \argmin_{\theta \in \cS_R} F^{\hat{r}^n}(\theta) .
\end{align*}

Then,
\begin{align*}
	\nabla_{\theta} F^{\hat{r}^n}(\theta) &:= \ex_{(s,a)\sim\rho^n_{\cover}} \mbr{ 2\sbr{f_0(s,a;\theta) - \sbr{Q^{\pi^t}(s,a;\hat{r}^n+b^n)-b^n(s,a)} } \nabla_{\theta} f_0(s,a;\theta) } .
\end{align*}

Furthermore, for any $i=0,\dots,M^{\theta}_{\subsgd}-1$, define
\begin{align*}
	g^{t,i} &:= 2 \sbr{ f(s_i,a_i;\theta^{t,i}) - \sbr{\hat{Q}^{\pi^t}(s_i,a_i;\hat{r}^n+b^n)-b^n(s_i,a_i)} } \nabla_{\theta}f(s_i,a_i;\theta^{t,i}) ,
	\\
	\tilde{\theta}^{t,i+1} &:=  \theta^{t,i} - \xi_{\theta} g^{t,i} ,
	\\
	\bar{g}^{t,i} &:= \ex_{(s,a)\sim\rho^n_{\cover}} \mbr{ 2 \sbr{ f(s,a;\theta^{t,i}) - \sbr{Q^{\pi^t}(s,a;\hat{r}^n+b^n)-b^n(s,a)} } \nabla_{\theta}f(s,a;\theta^{t,i}) } ,
\end{align*}
and it holds that
\begin{align*}
	\theta^{i+1} = \proj_{\cS_R} (\tilde{\theta}^{i+1}) .
\end{align*}

Let $W^{\nn}_{\nabla F}:=\frac{4}{(1-\gamma)^2} + \frac{4(\sqrt{m} \bar{c}+R)}{1-\gamma}$.


Define event
\begin{align*}
	\cE^{\nn}_{\theta}:= \Bigg\{  \bigg|\sum_{i=0}^{M^{\theta}_{\subsgd}-1} (g^{t,i})^\top \sbr{ \theta^{t,i} - \theta^{t,\hat{r}^n}_{\submid} }  -& \sum_{i=0}^{M^{\theta}_{\subsgd}-1} (\bar{g}^{t,i})^\top \sbr{ \theta^{t,i} - \theta^{t,\hat{r}^n}_{\submid} } \bigg| \leq 2 W^{\nn}_{\nabla F} R \sqrt{M^{\theta}_{\subsgd} \log\sbr{\frac{1}{\delta'}}} ,
	\nonumber\\
	& \forall 0\leq n \leq N-1, \forall 0\leq t \leq T-1, \forall 0\leq i \leq M^{\theta}_{\subsgd}-1  \Bigg\} . 
\end{align*}

\begin{lemma}
	It holds that $\Pr[\cE^{\nn}_{\theta}] \geq 1-2NT\delta'$.
\end{lemma}
\begin{proof}
	This lemma can be obtained by using the Azuma-Hoeffding inequality and the union bound.
\end{proof}

Let $W^{\nn}_{f}:= \sqrt{m}\bar{c}+R$,  $W^{\nn}_{Q}:= \frac{\sqrt{m}\bar{c}+R}{1-\gamma}+\frac{2}{(1-\gamma)^2}$ and $\xi_{\theta}:=\frac{R }{ W^{\nn}_{\nabla F} \sqrt{M^{\theta}_{\subsgd}}}$.

Below we give the guarantee for the projected SGD of Q-network training, which is described in algorithm~\ref{alg:nn_sgd_Q_fitting}.

\begin{lemma}[SGD for Q-value Function Fitting] \label{lemma:nn_sgd_Q}
	Assume that event $\cE^{\nn}_{\theta}$ holds. Then, for any phase $n\geq0$ and iteration $t\geq0$,
	\begin{align*}
		F^{\hat{r}^n}(\theta^{t})  -  F^{\hat{r}^n}(\theta^{t,\hat{r}^n}_{\submid}) 
		& \leq
		4 W^{\nn}_{\nabla F} R \sqrt{ \frac{\log\sbr{\frac{1}{\delta'}}}{M^{\theta}_{\subsgd}} }  + \frac{ 12R^2(W^{\nn}_{f} + W^{\nn}_{Q}) \sqrt{c_{\scale} R}}{\sqrt{\underline{c}} m^{\frac{1}{4}}}  :=\varepsilon^{\nn}_{\stat} .
	\end{align*}
\end{lemma}

\begin{proof}
	Fix phase $n$ and iteration $t$. 
	For any $i=0,\dots,M^{\theta}_{\subsgd}-1$, since $F^{\hat{r}^n}(\theta)$ is convex with respect to $\theta$, we have
	\begin{align*}
		F^{\hat{r}^n}(\theta^{t,i}) -  F^{\hat{r}^n}(\theta^{t,\hat{r}^n}_{\submid}) &\leq  \nabla_{\theta} F^{\hat{r}^n}(\theta^{t,i})^\top \sbr{ \theta^{t,i} - \theta^{t,\hat{r}^n}_{\submid} } 
		\\
		&= (g^{t,i})^\top \sbr{ \theta^{t,i} - \theta^{t,\hat{r}^n}_{\submid} } + \sbr{\nabla_{\theta} F^{\hat{r}^n}(\theta^{t,i})-g^{t,i}}^\top \sbr{ \theta^{t,i} - \theta^{t,\hat{r}^n}_{\submid} } 
		\\
		&=  \frac{1}{\xi_{\theta}}(\theta^{t,i} - \tilde{\theta}^{t+1,i})^\top \sbr{ \theta^{t,i} - \theta^{t,\hat{r}^n}_{\submid} } + \sbr{\nabla_{\theta} F^{\hat{r}^n}(\theta^{t,i})-g^{t,i}}^\top \sbr{ \theta^{t,i} - \theta^{t,\hat{r}^n}_{\submid} } 
		\\
		&=  \frac{1}{2\xi_{\theta}} \sbr{ \nbr{\theta^{t,i} - \tilde{\theta}^{t,i+1}}_2^2 + \nbr{\theta^{t,i} - \theta^{t,\hat{r}^n}_{\submid}}_2^2 - \nbr{\tilde{\theta}^{t,i+1} - \theta^{t,\hat{r}^n}_{\submid}}_2^2 }  \\& \quad +  \sbr{\nabla_{\theta} F^{\hat{r}^n}(\theta^{t,i})-g^{t,i}}^\top \sbr{ \theta^{t,i} - \theta^{t,\hat{r}^n}_{\submid} } 
		\\
		&\leq  \frac{\xi_{\theta}}{2}  \nbr{g^{t,i}}_2^2 + \frac{1}{2\xi_{\theta}} \sbr{ \nbr{\theta^{t,i} - \theta^{t,\hat{r}^n}_{\submid}}_2^2 - \nbr{\theta^{t,i+1} - \theta^{t,\hat{r}^n}_{\submid}}_2^2 }  \\& \quad +  \sbr{\nabla_{\theta} F^{\hat{r}^n}(\theta^{t,i})-g^{t,i}}^\top \sbr{ \theta^{t,i} - \theta^{t,\hat{r}^n}_{\submid} } 
	\end{align*}
	
	Summing $i=0,\dots,M^{\theta}_{\subsgd}-1$ and dividing $M^{\theta}_{\subsgd}$, we have
	\begin{align}
		&\quad  F^{\hat{r}^n}(\theta^{t})  -  F^{\hat{r}^n}(\theta^{t,\hat{r}^n}_{\submid})
		\nonumber\\
		&=  F\sbr{ \frac{1}{M^{\theta}_{\subsgd}} \sum_{i=0}^{M^{\theta}_{\subsgd}-1} \theta^{t,i} }  -  F^{\hat{r}^n}(\theta^{t,\hat{r}^n}_{\submid})
		\nonumber\\
		&\overset{\textup{(a)}}{\leq} \frac{1}{M^{\theta}_{\subsgd}} \sum_{i=0}^{M^{\theta}_{\subsgd}-1}	 F^{\hat{r}^n}(\theta^{t,i})  -  F^{\hat{r}^n}(\theta^{t,\hat{r}^n}_{\submid}) 
		\nonumber\\
		&\leq  \frac{\xi_{\theta}}{2 M^{\theta}_{\subsgd}} \sum_{i=0}^{M^{\theta}_{\subsgd}-1} \nbr{g^{t,i}}_2^2 + \frac{1}{2\xi_{\theta} M^{\theta}_{\subsgd}} \sbr{ \nbr{\theta^{t,0} - \theta^{t,\hat{r}^n}_{\submid}}_2^2 - \nbr{\theta^{t,M^{\theta}_{\subsgd}} - \theta^{t,\hat{r}^n}_{\submid}}_2^2 }  \nonumber\\& \quad +  \frac{1}{M^{\theta}_{\subsgd}} \sum_{i=0}^{M^{\theta}_{\subsgd}-1} \sbr{\nabla_{\theta} F^{\hat{r}^n}(\theta^{t,i}) - \bar{g}^{t,i} + \bar{g}^{t,i} -g^{t,i}}^\top \sbr{ \theta^{t,i} - \theta^{t,\hat{r}^n}_{\submid} } 
		\nonumber\\
		&\leq  \frac{\xi_{\theta}}{2 M^{\theta}_{\subsgd}} \sum_{i=0}^{M^{\theta}_{\subsgd}-1} \nbr{g^{t,i}}_2^2  + \frac{R^2}{2\xi_{\theta} M^{\theta}_{\subsgd}} +  \frac{1}{M^{\theta}_{\subsgd}} \sum_{i=0}^{M^{\theta}_{\subsgd}-1} \sbr{ \bar{g}^{t,i}-g^{t,i} }^\top \sbr{ \theta^{t,i} - \theta^{t,\hat{r}^n}_{\submid} }  \nonumber\\&\quad\ +  \frac{2R}{M^{\theta}_{\subsgd}} \sum_{i=0}^{M^{\theta}_{\subsgd}-1} \nbr{\nabla_{\theta} F^{\hat{r}^n}(\theta^{t,i}) - \bar{g}^{t,i}}_2  , \label{eq:F_0_theta_t-F_0_theta_t_mid}
	\end{align}
	where inequality (a) uses the Jensen inequality.
	
	For any $i\geq0$, let $\cH_i$ be all histories of steps $0,\dots,i$, and we make the convention that $\cH_{i-1}=\emptyset$ for $i=0$. Let $\ex_i[\cdot|\cH_{i-1}]$ denote the expectation with respect to the randomness at step $i$ conditioning on all histories of steps $0,\dots,i-1$. 
	Then, for any $i\geq0$, we have $\ex_i[\nabla_{\theta} \hat{F}^i(\theta^{t,i})^\top \sbr{ \theta^{t,i} - \theta^{t,\hat{r}^n}_{\submid} }| \cH_{i-1}]=\nabla_{\theta} F^{\hat{r}^n}(\theta^{t,i})^\top \sbr{ \theta^{t,i} - \theta^{t,\hat{r}^n}_{\submid} }$.
	
	According to the definition of event $\cE^{\nn}_{\theta}$, we have
	\begin{align}
		&\quad\ \abr{\sum_{i=0}^{M^{\theta}_{\subsgd}-1} (g^{t,i})^\top \sbr{ \theta^{t,i} - \theta^{t,\hat{r}^n}_{\submid} } - \sum_{i=0}^{M^{\theta}_{\subsgd}-1} (\bar{g}^{t,i})^\top \sbr{ \theta^{t,i} - \theta^{t,\hat{r}^n}_{\submid} } }
		\nonumber\\
		&\leq 2 W^{\nn}_{\nabla F} R \sqrt{M^{\theta}_{\subsgd} \log\sbr{\frac{1}{\delta'}}} . 
		\label{eq:g_0-hat_g_0}
	\end{align}

	Then, we have
	\begin{align}
		&\quad\ \nbr{ \nabla_{\theta} F^{\hat{r}^n}(\theta^{t,i}) - \bar{g}^{t,i} }_2
		\nonumber\\
		&= \Big\|  \ex_{(s,a)\sim\rho^n_{\cover}} \Big[ 2\sbr{f_0(s,a;\theta^{t,i}) - \sbr{Q^{\pi^t}(s,a;\hat{r}^n+b^n)-b^n(s,a)} } \nabla_{\theta} f_0(s,a;\theta^{t,i}) 
		\nonumber\\
		&\quad\ - 2 \sbr{ f(s,a;\theta^{t,i}) - \sbr{Q^{\pi^t}(s,a;\hat{r}^n+b^n)-b^n(s,a)} } \nabla_{\theta}f(s,a;\theta^{t,i}) \Big] \Big\|_2
		\nonumber\\
		&\leq 2 \ex_{(s,a)\sim\rho^n_{\cover}} \Big[ \nbr{ f_0(s,a;\theta^{t,i}) \nabla_{\theta} f_0(s,a;\theta^{t,i}) - f(s,a;\theta^{t,i}) \nabla_{\theta}f(s,a;\theta^{t,i}) }_2
		\nonumber\\
		&\quad\ +  W^{\nn}_{Q} \nbr{\nabla_{\theta}f_0(s,a;\theta^{t,i})  - \nabla_{\theta}f(s,a;\theta^{t,i}) }_2 \Big] 
		\nonumber\\
		&\leq 2 \ex_{(s,a)\sim\rho^n_{\cover}} \Big[ \nbr{ f_0(s,a;\theta^{t,i}) \nabla_{\theta} f_0(s,a;\theta^{t,i}) - f_0(s,a;\theta^{t,i}) \nabla_{\theta}f(s,a;\theta^{t,i}) }_2
		\nonumber\\
		&\quad\ +\nbr{ f_0(s,a;\theta^{t,i}) \nabla_{\theta} f(s,a;\theta^{t,i}) - f(s,a;\theta^{t,i}) \nabla_{\theta}f(s,a;\theta^{t,i}) }_2
		\nonumber\\
		&\quad\ +  W^{\nn}_{Q} \nbr{\nabla_{\theta}f_0(s,a;\theta^{t,i})  - \nabla_{\theta}f(s,a;\theta^{t,i}) }_2 \Big] 
		\nonumber\\
		&\leq 2  \ex_{(s,a)\sim\rho^n_{\cover}} \Big[ \abr{ f_0(s,a;\theta^{t,i})  - f(s,a;\theta^{t,i}) } 
		\nonumber\\
		&\quad\ + (W^{\nn}_{f} + W^{\nn}_{Q}) \nbr{ \psi_0(s,a)  - \psi_{\theta^{t,i}}(s,a) }_2 \Big] 
		.
		\label{eq:hat_g_0-g_i}
	\end{align}
	
	Plugging Eqs.~\eqref{eq:g_0-hat_g_0} and \eqref{eq:hat_g_0-g_i} into Eq.~\eqref{eq:F_0_theta_t-F_0_theta_t_mid}, we have
	\begin{align*}
		F^{\hat{r}^n}(\theta^{t})  -  F^{\hat{r}^n}(\theta^{t,\hat{r}^n}_{\submid}) 
		& \leq
		4 W^{\nn}_{\nabla F} R \sqrt{ \frac{\log\sbr{\frac{1}{\delta'}}}{M^{\theta}_{\subsgd}} } + \frac{4R}{M^{\theta}_{\subsgd}} \sum_{i=0}^{M^{\theta}_{\subsgd}-1}    \ex_{(s,a)\sim\rho^n_{\cover}} \Big[ \abr{ f_0(s,a;\theta^{t,i})  - f(s,a;\theta^{t,i}) } 
		\nonumber\\
		&\quad\ + (W^{\nn}_{f} + W^{\nn}_{Q}) \nbr{ \psi_0(s,a)  - \psi_{\theta^{t,i}}(s,a) }_2 \Big] 
		\\
		& \overset{\textup{(a)}}{\leq} 4 W^{\nn}_{\nabla F} R \sqrt{ \frac{\log\sbr{\frac{1}{\delta'}}}{M^{\theta}_{\subsgd}} } + 4R \sbr{ \frac{2 \sqrt{c_{\scale}R^3}}{ \sqrt{\underline{c}} m^{\frac{1}{4}}} +  \frac{ (W^{\nn}_{f} + W^{\nn}_{Q}) \sqrt{c_{\scale} R}}{\sqrt{\underline{c}} m^{\frac{1}{4}}} } 
		\\
		& \leq 4 W^{\nn}_{\nabla F} R \sqrt{ \frac{\log\sbr{\frac{1}{\delta'}}}{M^{\theta}_{\subsgd}} } +  \frac{ 12R^2(W^{\nn}_{f} + W^{\nn}_{Q}) \sqrt{c_{\scale} R}}{\sqrt{\underline{c}} m^{\frac{1}{4}}}  , 
	\end{align*}		
	where inequality (a) uses Assumption~\ref{assumption:scale_theta_0}.
	
\end{proof}

\subsection{Human Feedback} \label{apx:nn_human_feedback}

Recall that for any trajectories $\tau^{(1)}, \tau^{(2)}$  and $\mu \in \cU_R$, let $\tilde{\psi}_{\mu}^{\tau^{(1)},\tau^{(2)}}:=\sum_{h=0}^{H(\tau^{(1)})}\psi_{\mu}(s^{(1)}_{h},a^{(1)}_{h}) - \sum_{h=0}^{H(\tau^{(2)})}\psi_{\mu}(s^{(2)}_{h},a^{(2)}_{h})$,   $\tilde{h}(\tau^{(1)},\tau^{(2)};\mu):= \sum_{h=0}^{H(\tau^{(1)})} h(s^{(1)}_{h},a^{(1)}_{h};\mu) - \sum_{h=0}^{H(\tau^{(2)})} h(s^{(2)}_{h},a^{(2)}_{h};\mu)$ and $\tilde{r}(\tau^{(1)},\tau^{(2)}):=\sum_{h=0}^{H(\tau^{(1)})} r(s^{(1)}_{h},a^{(1)}_{h}) - \sum_{h=0}^{H(\tau^{(2)})} r(s^{(2)}_{h},a^{(2)}_{h})$.

For any fixed phase $n=0,\dots,N-1$, define the approximated MLE objective function and its optimal solution as follows:
\begin{align*}
	L(\mu)&:= \frac{1}{M_{\hf}} \sum_{i=1}^{M_{\hf}} \Bigg( - \log \Bigg( \frac{ \indicator{y_i=1} }{ 1+ \exp\sbr{   \sum_{h=0}^{H(\tau^{(2)}_i)} f_0(s^{(2)}_{i,h},a^{(2)}_{i,h};\mu) - \sum_{h=0}^{H(\tau^{(1)}_i)} f_0(s^{(1)}_{i,h},a^{(1)}_{i,h};\mu)  } } 
	\\
	& \quad\ + \frac{ \indicator{ y_i=0 } }{ 1+ \exp\sbr{   \sum_{h=0}^{H(\tau^{(1)}_i)}f_0(s^{(1)}_{i,h},a^{(1)}_{i,h};\mu) - \sum_{h=0}^{H(\tau^{(2)}_i)}f_0(s^{(2)}_{i,h},a^{(2)}_{i,h};\mu)  } } \Bigg) \Bigg)
	\\
	&= \frac{1}{M_{\hf}} \sum_{i=1}^{M_{\hf}} \sbr{ - \log \sbr{ \frac{ \indicator{y_i=1} }{ 1+ \exp\sbr{  -( \tilde{\psi}_0^{\tau^{(1)}_i,\tau^{(2)}_i} )^\top \mu } } + \frac{ \indicator{ y_i=0 } }{ 1+ \exp\sbr{  (\tilde{\psi}_0^{\tau^{(1)}_i,\tau^{(2)}_i})^\top \mu } } } } ,
	\\
	\mu^{*}_{\mle}&:= \argmin_{\mu \in \cU_R} L(\mu) .
\end{align*}

Then, it holds that
\begin{align*}
	\nabla_{\mu} L(\mu) &= \frac{1}{M_{\hf}} \sum_{i=1}^{M_{\hf}} \Bigg( \underbrace{\Bigg(- \frac{ \indicator{y_i=1} \exp\sbr{  -( \tilde{\psi}_0^{\tau^{(1)}_i,\tau^{(2)}_i} )^\top \mu } }{ 1+ \exp\sbr{  -( \tilde{\psi}_0^{\tau^{(1)}_i,\tau^{(2)}_i} )^\top \mu } } + \frac{ \indicator{ y_i=0 } }{ 1+ \exp\sbr{  - (\tilde{\psi}_0^{\tau^{(1)}_i,\tau^{(2)}_i})^\top \mu } } \Bigg)}_{:=q^i_0(\mu)} \tilde{\psi}_0^{\tau^{(1)}_i,\tau^{(2)}_i} \Bigg) ,
	\\
	\nabla_{\mu}^2 L(\mu) &\!=\! \!\frac{1}{M_{\hf}}\! \sum_{i=1}^{M_{\hf}}\! \sbr{ \sbr{ \frac{ \indicator{y_i=1} \exp\sbr{  -( \tilde{\psi}_0^{\tau^{(1)}_i,\tau^{(2)}_i} )^\top \mu } }{ \sbr{1+ \exp\sbr{  -( \tilde{\psi}_0^{\tau^{(1)}_i,\tau^{(2)}_i} )^\top \mu }}^2 } \!+\! \frac{ \indicator{ y_i=0 } (\tilde{\psi}_0^{\tau^{(1)}_i,\tau^{(2)}_i})^\top \mu }{ \sbr{1+ \exp\sbr{  (\tilde{\psi}_0^{\tau^{(1)}_i,\tau^{(2)}_i})^\top \mu }}^2 } } \!\tilde{\psi}_0^{\tau^{(1)}_i,\tau^{(2)}_i}\!\! (\tilde{\psi}_0^{\tau^{(1)}_i,\tau^{(2)}_i})^\top\!\! } \!.
\end{align*}

For any $j = 0,\dots,M^{\mu}_{\subsgd}-1$, define
\begin{align*}
	z^j &:=  \sbr{- \frac{ \indicator{y_j=1} \exp\sbr{  -\tilde{h}(\tau^{(1)}_j,\tau^{(2)}_j; \mu^j) } }{ 1+ \exp\sbr{  -\tilde{h}(\tau^{(1)}_j,\tau^{(2)}_j; \mu^j) } } + \frac{ \indicator{ y_j=0 } }{ 1+ \exp\sbr{  -\tilde{h}(\tau^{(1)}_j,\tau^{(2)}_j; \mu^j) } } }  \nabla_{\mu} \tilde{h}(\tau^{(1)}_j,\tau^{(2)}_j; \mu^j) ,
	\\
	\bar{z}^j &:= \frac{1}{M_{\hf}} \sum_{i=1}^{M_{\hf}} \Bigg( \underbrace{ \Bigg(- \frac{ \indicator{y_i=1} \exp\sbr{  -\tilde{h}(\tau^{(1)}_i,\tau^{(2)}_i; \mu^j) } }{ 1+ \exp\sbr{  -\tilde{h}(\tau^{(1)}_i,\tau^{(2)}_i; \mu^j) } } + \frac{ \indicator{ y_i=0 } }{ 1+ \exp\sbr{  -\tilde{h}(\tau^{(1)}_i,\tau^{(2)}_i; \mu^j) } } \Bigg) }_{:=q^i(\mu^j)} \nabla_{\mu} \tilde{h}(\tau^{(1)}_i,\tau^{(2)}_i; \mu^j) \Bigg) ,
	\\
	\tilde{\mu}^{j+1} &:= \mu^{j}-\xi_{\mu} z^j ,
\end{align*}
where $(\tau^{(1)}_j,\tau^{(2)}_j,y_j)$ is uniformly drawn from $\{(\tau^{(1)}_i,\tau^{(2)}_i,y_i)\}_{i=1}^{M_{\hf}}$. 

Then, we have
\begin{align*}
	\mu^{j+1} = \proj_{\cU_R}(\tilde{\mu}^{j+1}) .
\end{align*}

\begin{algorithm}[t]
	\caption{Reward Network Training via Projected SGD (with the objective Eq.~\eqref{eq:apx_nn_mle_mu_r})}
	\label{alg:nn_sgd_reward}
	\begin{algorithmic}[1]
		\STATE {\bfseries Input:} $h(s,a;\mu^0)$, $\xi_{\mu}$.
		\FOR{$j=0,\dots,M^{\mu}_{\subsgd}-1$}
		\STATE $z^j \leftarrow \sbr{- \frac{ \indicator{y_j=1} \exp\sbr{  -\tilde{h}(\tau^{(1)}_j,\tau^{(2)}_j; \mu^j) } }{ 1+ \exp\sbr{  -\tilde{h}(\tau^{(1)}_j,\tau^{(2)}_j; \mu^j) } } + \frac{ \indicator{ y_j=0 } }{ 1+ \exp\sbr{  -\tilde{h}(\tau^{(1)}_j,\tau^{(2)}_j; \mu^j) } } }  \nabla_{\mu} \tilde{h}(\tau^{(1)}_j,\tau^{(2)}_j; \mu^j)$
		\STATE $\tilde{\mu}^{j+1} \leftarrow \mu^{j}-\xi_{\mu} z^j$
		\STATE $\mu^{i+1} \leftarrow \proj_{\cU_R} ( \tilde{\mu}^{j+1} )$
		\ENDFOR
		\STATE {\bfseries return} $\hat{\mu}^n = \sum_{j=0}^{M^{\mu}_{\subsgd}-1} \mu^j$
	\end{algorithmic}
\end{algorithm}

Define event
\begin{align*}
	\cE^{\nn}_{\mu}:=\Bigg\{& \bigg| \sum_{j=0}^{M^{\mu}_{\subsgd}-1} \nabla_{\mu} \hat{L}^j(\mu^j)^\top \sbr{ \mu^{j} -  \mu^{*}_{\mle} } - \sum_{j=0}^{M^{\mu}_{\subsgd}-1} \nabla_{\mu} L(\mu^{j})^\top \sbr{ \mu^{j} - \mu^{*}_{\mle} } \bigg|
	\\
	&\leq 8W_{\tau} R \sqrt{ M^{\mu}_{\subsgd} \log\sbr{\frac{1}{\delta'}} } ,\ \forall 0\leq n \leq N-1, \forall 0\leq j\leq M^{\mu}_{\subsgd}-1 \Bigg\} .
\end{align*}

\begin{lemma}
	It holds that $\Pr[\cE^{\nn}_{\mu}] \geq 1-2N\delta'$.
\end{lemma}
\begin{proof}
	This lemma can be obtained by applying the Azuma-Hoeffding inequality and the union bound.
\end{proof}

Let $\xi_{\mu}:=\frac{R}{W_{\tau} \sqrt{M^{\mu}_{\subsgd}}}$.

Below we provide the guarantee for the projected SGD of reward training, which is illustrated in algorithm~\ref{alg:nn_sgd_reward}.

\begin{lemma}[SGD for the Reward Model]\label{lemma:sgd_mu}
	Assume that event $\cE_{\init} \cap \cE_{\tau}\cap\cE^{\nn}_{\mu}$ holds. Then, for any phase $n$,
	\begin{align*}
		&\quad\ L(\mu^{n}) - L(\mu^{*}_{\mle})
		\nonumber\\
		&\leq  17W_{\tau} R \sqrt{ \frac{\log\sbr{\frac{1}{\delta'}}}{M^{\mu}_{\subsgd}}  }
		+  \frac{2R}{M^{\mu}_{\subsgd}} \sum_{j=0}^{M^{\mu}_{\subsgd}-1} \Bigg( \frac{1}{M_{\hf}} \sum_{i=1}^{M_{\hf}} \bigg( 2\nbr{ \tilde{\psi}_0^{\tau^{(1)}_i,\tau^{(2)}_i} - \tilde{\psi}_{\mu^j}^{\tau^{(1)}_i,\tau^{(2)}_i} }_2 
		\\
		&\quad\ + 4 W_{\tau}  \abr{ \tilde{h}_0(\tau^{(1)}_i,\tau^{(2)}_i; \mu^j) - \tilde{h}(\tau^{(1)}_i,\tau^{(2)}_i; \mu^j) } \bigg) \Bigg)  := \varepsilon^{\nn,n}_{\subsgd} .
	\end{align*}
	Furthermore,
	\begin{align*}
		\ex_{\begin{subarray}{l} \{\tau^{(1)}_i\}_{i=1}^{M_{\hf}} \sim \cO^{n}_{\hf} \\ \{\tau^{(2)}_i\}_{i=1}^{M_{\hf}} \sim \cO^{\pi^{\base}}_{s_{\init}} \end{subarray}}\mbr{\varepsilon^{\nn,n}_{\subsgd}} &\leq 17W_{\tau} R \sqrt{ \frac{\log\sbr{\frac{1}{\delta'}}}{M^{\mu}_{\subsgd}}  }
		+  \frac{40 R^2 W_{\tau}  \sqrt{c_{\scale} R}}{ (1-\gamma) \sqrt{ \underline{c}} m^{\frac{1}{4}} } .
	\end{align*}
\end{lemma}
\begin{proof}
	For any $j=0,\dots,M^{\mu}_{\subsgd}-1$,
	\begin{align*}
		L(\mu^{j}) - L(\mu^{*}_{\mle}) &\leq  \nabla_{\mu} L(\mu^{j})^\top \sbr{ \mu^{j} - \mu^{*}_{\mle} } 
		\\
		&=  (z^j)^\top \sbr{ \mu^{j} - \mu^{*}_{\mle} } + \sbr{\nabla_{\mu} L(\mu^{j}) - z^j }^\top \sbr{ \mu^{j} - \mu^{*}_{\mle} } 
		\\
		&=  \frac{1}{\xi_{\mu}} \sbr{ \mu^{j} - \tilde{\mu}^{j+1} }^\top \sbr{ \mu^{j} - \mu^{*}_{\mle} }  +  \sbr{\nabla_{\mu} L(\mu^{j}) - z^j }^\top \sbr{ \mu^{j} - \mu^{*}_{\mle} } 
		\\
		&=  \frac{1}{2\xi_{\mu}} \sbr{\nbr{ \mu^{j} - \tilde{\mu}^{j+1} }_2^2 + \nbr{ \mu^{j} - \mu^{*}_{\mle} }_2^2 - \nbr{ \tilde{\mu}^{j+1} - \mu^{*}_{\mle} }_2^2 }  \\&\quad\ +  \sbr{\nabla_{\mu} L(\mu^{j}) - z^j }^\top \sbr{ \mu^{j} - \mu^{*}_{\mle} } 
		\\
		&\leq  \frac{\xi_{\mu}}{2} \nbr{ z^j }_2^2 + \frac{1}{2\xi_{\mu}} \sbr{ \nbr{ \mu^{j} - \mu^{*}_{\mle} }_2^2 - \nbr{ \mu^{j+1} - \mu^{*}_{\mle} }_2^2 }  \\&\quad\ +  \sbr{\nabla_{\mu} L(\mu^{j}) - z^j }^\top \sbr{ \mu^{j} - \mu^{*}_{\mle} }  .
	\end{align*}
	
	Summing $j=0,\dots,M^{\mu}_{\subsgd}-1$ and dividing $M^{\mu}_{\subsgd}$, we have
	\begin{align}
		&\quad\ L(\mu^{n}) - L(\mu^{*}_{\mle})
		\nonumber\\
		&= L\sbr{ \frac{1}{M^{\mu}_{\subsgd}} \sum_{j=0}^{M^{\mu}_{\subsgd}-1} \mu^j } - L(\mu^{*}_{\mle})
		\nonumber\\
		&\overset{\textup{(a)}}{\leq} \frac{1}{M^{\mu}_{\subsgd}} \sum_{j=0}^{M^{\mu}_{\subsgd}-1} \sbr{L(\mu^{j}) - L(\mu^{*}_{\mle})} 
		\nonumber\\
		&\leq  \frac{\xi_{\mu}}{2 M^{\mu}_{\subsgd}} \sum_{j=0}^{M^{\mu}_{\subsgd}-1} \nbr{ z^j }_2^2 + \frac{1}{2\xi_{\mu} M^{\mu}_{\subsgd}} \sbr{ \nbr{ \mu^{0} - \mu^{*}_{\mle} }_2^2 - \nbr{ \mu^{M^{\mu}_{\subsgd}} - \mu^{*}_{\mle} }_2^2 }  \nonumber\\&\quad\ +   \frac{1}{M^{\mu}_{\subsgd}} \sum_{j=0}^{M^{\mu}_{\subsgd}-1} \sbr{\nabla_{\mu} L(\mu^{j}) - \bar{z}^j + \bar{z}^j - z^j }^\top \sbr{ \mu^{j} - \mu^{*}_{\mle} }
		\nonumber\\
		&\leq  \frac{\xi_{\mu}}{2 M^{\mu}_{\subsgd}} \sum_{j=0}^{M^{\mu}_{\subsgd}-1} \nbr{ z^j }_2^2 + \frac{R^2}{2\xi_{\mu} M^{\mu}_{\subsgd}}    
		+   \frac{1}{M^{\mu}_{\subsgd}} \sum_{j=0}^{M^{\mu}_{\subsgd}-1} \sbr{ \bar{z}^j - z^j }^\top \sbr{ \mu^{j} - \mu^{*}_{\mle} }
		\nonumber\\
		&\quad\ +  \frac{2R}{M^{\mu}_{\subsgd}} \sum_{j=0}^{M^{\mu}_{\subsgd}-1} \nbr{ \nabla_{\mu} L(\mu^j) - \bar{z}^j }_2 , \label{eq:mu_sgd_decomposition}
	\end{align}
	where inequality (a) uses the Jensen inequality.
	
	We have
	\begin{align}
		\nbr{z^j}_2 &\leq 2 \nbr{ \tilde{\psi}^{\tau^{(1)}_j,\tau^{(2)}_j}_{\mu^j} }_2
		\nonumber\\
		&= 2 \nbr{ \sum_{h=0}^{H(\tau^{(1)}_j)}\psi_{\mu^j}(s^{(1)}_{j,h},a^{(1)}_{j,h}) - \sum_{h=0}^{H(\tau^{(2)}_j)}\psi_{\mu^j}(s^{(2)}_{j,h},a^{(2)}_{j,h}) }_2
		\nonumber\\
		&\leq 4 W_{\tau}  . \label{eq:ub_z_j}
	\end{align}
	
	For any $j\geq0$, let $\cH_j$ be all histories of steps $0,\dots,j$, and we make the convention that $\cH_{j-1}=\emptyset$ for $j=0$. Let $\ex_j[\cdot|\cH_{j-1}]$ denote the expectation with respect to the randomness at step $j$ (i.e., $(\tau^{(1)}_j,\tau^{(2)}_j,y_j) \sim \unif( \{(\tau^{(1)}_i,\tau^{(2)}_i,y_i)\}_{i=1}^{M_{\hf}})$) conditioning on all histories of steps $0,\dots,j-1$. Then, for any $j\geq0$, we have $\ex_j[\nabla_{\mu} \hat{L}^j(\mu^j)^\top \sbr{ \mu^{j} -  \mu^{*}_{\mle} }| \cH_{j-1}]=\nabla_{\mu} L(\mu^{j})^\top \sbr{ \mu^{j} - \mu^{*}_{\mle} }$.
	
	According to the definition of $\cE^{\nn}_{\mu}$, we have
	\begin{align}
		&\quad\ \abr{ \sum_{j=0}^{M^{\mu}_{\subsgd}-1} (z^j)^\top \sbr{ \mu^{j} -  \mu^{*}_{\mle} } - \sum_{j=0}^{M^{\mu}_{\subsgd}-1} (\bar{z}^j)^\top \sbr{ \mu^{j} - \mu^{*}_{\mle} } } 
		\nonumber\\
		&\leq 8W_{\tau} R \sqrt{ M^{\mu}_{\subsgd} \log\sbr{\frac{1}{\delta'}} } . \label{eq:azuma_nable_L}
	\end{align}
	
	
	We have
	\begin{align*}
		\abr{ q^i_0(\mu^j) - q^i(\mu^j) } &\leq \Bigg| - \frac{ \indicator{y_i=1} \exp\sbr{  -( \tilde{\psi}_0^{\tau^{(1)}_i,\tau^{(2)}_i} )^\top \mu^j } }{ 1+ \exp\sbr{  -( \tilde{\psi}_0^{\tau^{(1)}_i,\tau^{(2)}_i} )^\top \mu^j } } + \frac{ \indicator{ y_i=0 } }{ 1+ \exp\sbr{  - (\tilde{\psi}_0^{\tau^{(1)}_i,\tau^{(2)}_i})^\top \mu^j } } 
		\\
		&\quad\ + \frac{ \indicator{y_i=1} \exp\sbr{  -\tilde{h}(\tau^{(1)}_i,\tau^{(2)}_i; \mu^j) } }{ 1+ \exp\sbr{  -\tilde{h}(\tau^{(1)}_i,\tau^{(2)}_i; \mu^j) } } - \frac{ \indicator{ y_i=0 } }{ 1+ \exp\sbr{  -\tilde{h}(\tau^{(1)}_i,\tau^{(2)}_i; \mu^j) } }
		\Bigg|
		\\
		&\leq \abr{ \frac{   \exp\sbr{  -\tilde{h}(\tau^{(1)}_i,\tau^{(2)}_i; \mu^j) } }{ 1+ \exp\sbr{  -\tilde{h}(\tau^{(1)}_i,\tau^{(2)}_i; \mu^j) } } - \frac{  \exp\sbr{  -( \tilde{\psi}_0^{\tau^{(1)}_i,\tau^{(2)}_i} )^\top \mu^j } }{ 1+ \exp\sbr{  -( \tilde{\psi}_0^{\tau^{(1)}_i,\tau^{(2)}_i} )^\top \mu^j } } }
		\\
		&\quad\ + \abr{ \frac{ 1 }{ 1+ \exp\sbr{  - (\tilde{\psi}_0^{\tau^{(1)}_i,\tau^{(2)}_i})^\top \mu^j } }  - - \frac{ 1 }{ 1+ \exp\sbr{  -\tilde{h}(\tau^{(1)}_i,\tau^{(2)}_i; \mu^j) } } } 
		\\
		&\overset{\textup{(a)}}{\leq} 2\abr{ \tilde{h}_0(\tau^{(1)}_i,\tau^{(2)}_i; \mu^j) - \tilde{h}(\tau^{(1)}_i,\tau^{(2)}_i; \mu^j) } ,
	\end{align*}
	where inequality (a) uses the fact that the derivative of functions $\frac{\exp(x)}{1+\exp(x)}$ and $\frac{1}{1+\exp(x)}$ lies in $(0,1)$.
	
	
	Then, it holds that
	\begin{align}
		&\quad\ \nbr{ \nabla_{\mu} L(\mu^j) - \bar{z}^j }_2 
		\\
		&= \nbr{ \frac{1}{M_{\hf}} \sum_{i=1}^{M_{\hf}} \sbr{q^i_0(\mu^j) \tilde{\psi}_0^{\tau^{(1)}_i,\tau^{(2)}_i} - q^i(\mu^j) \tilde{\psi}_{\mu^j}^{\tau^{(1)}_i,\tau^{(2)}_i} } }_2
		\nonumber\\
		&\leq \frac{1}{M_{\hf}} \sum_{i=1}^{M_{\hf}} \nbr{ q^i_0(\mu^j) \tilde{\psi}_0^{\tau^{(1)}_i,\tau^{(2)}_i} - q^i_0(\mu^j) \tilde{\psi}_{\mu^j}^{\tau^{(1)}_i,\tau^{(2)}_i} + q^i_0(\mu^j) \tilde{\psi}_{\mu^j}^{\tau^{(1)}_i,\tau^{(2)}_i} - q^i(\mu^j) \tilde{\psi}_{\mu^j}^{\tau^{(1)}_i,\tau^{(2)}_i}  }_2
		\nonumber\\
		&\leq \frac{1}{M_{\hf}} \sum_{i=1}^{M_{\hf}} \sbr{ 2\nbr{ \tilde{\psi}_0^{\tau^{(1)}_i,\tau^{(2)}_i} - \tilde{\psi}_{\mu^j}^{\tau^{(1)}_i,\tau^{(2)}_i} }_2 + 2 W_{\tau} \abr{ q^i_0(\mu^j) - q^i(\mu^j) } }
		\nonumber\\
		&\leq \frac{1}{M_{\hf}} \sum_{i=1}^{M_{\hf}} \sbr{ 2\nbr{ \tilde{\psi}_0^{\tau^{(1)}_i,\tau^{(2)}_i} - \tilde{\psi}_{\mu^j}^{\tau^{(1)}_i,\tau^{(2)}_i} }_2 + 4 W_{\tau} \abr{ \tilde{h}_0(\tau^{(1)}_i,\tau^{(2)}_i; \mu^j) - \tilde{h}(\tau^{(1)}_i,\tau^{(2)}_i; \mu^j) } } . \label{eq:distance_nable_L_z^j}
	\end{align}
	
	Plugging Eqs.~\eqref{eq:ub_z_j}-\eqref{eq:distance_nable_L_z^j} into Eq.~\eqref{eq:mu_sgd_decomposition}, we have
	\begin{align*}
		&\quad\ L(\mu^{n}) - L(\mu^{*}_{\mle})
		\nonumber\\
		&\leq  8 \xi_{\mu} W_{\tau}^2  + \frac{R^2}{2\xi_{\mu} M^{\mu}_{\subsgd}}    
		+ 8W_{\tau} R \sqrt{ \frac{\log\sbr{\frac{1}{\delta'}}}{M^{\mu}_{\subsgd}}  }
		\nonumber\\
		&\quad\ +  \frac{2R}{M^{\mu}_{\subsgd}} \sum_{j=0}^{M^{\mu}_{\subsgd}-1} \Bigg( \frac{1}{M_{\hf}} \sum_{i=1}^{M_{\hf}} \bigg( 2\nbr{ \tilde{\psi}_0^{\tau^{(1)}_i,\tau^{(2)}_i} - \tilde{\psi}_{\mu^j}^{\tau^{(1)}_i,\tau^{(2)}_i} }_2 
		\\
		&\quad\ + 4 W_{\tau}  \abr{ \tilde{h}_0(\tau^{(1)}_i,\tau^{(2)}_i; \mu^j) - \tilde{h}(\tau^{(1)}_i,\tau^{(2)}_i; \mu^j) } \bigg) \Bigg)
		\nonumber\\
		&\leq  17W_{\tau} R \sqrt{ \frac{\log\sbr{\frac{1}{\delta'}}}{M^{\mu}_{\subsgd}}  } +  \frac{2R}{M^{\mu}_{\subsgd}} \sum_{j=0}^{M^{\mu}_{\subsgd}-1} \Bigg( \frac{1}{M_{\hf}} \sum_{i=1}^{M_{\hf}} \bigg( 2\nbr{ \tilde{\psi}_0^{\tau^{(1)}_i,\tau^{(2)}_i} - \tilde{\psi}_{\mu^j}^{\tau^{(1)}_i,\tau^{(2)}_i} }_2 
		\\
		&\quad\ + 4 W_{\tau}  \abr{ \tilde{h}_0(\tau^{(1)}_i,\tau^{(2)}_i; \mu^j) - \tilde{h}(\tau^{(1)}_i,\tau^{(2)}_i; \mu^j) } \bigg) \Bigg) := \varepsilon^{\nn,n}_{\subsgd} .
	\end{align*}
	
	In addition, we have
	\begin{align*}
		&\quad\ \ex_{\begin{subarray}{l} \{\tau^{(1)}_i\}_{i=1}^{M_{\hf}} \sim \cO^{n}_{\hf} ,\ \{\tau^{(2)}_i\}_{i=1}^{M_{\hf}} \sim \cO^{\pi^{\base}}_{s_{\init}} \end{subarray}}\mbr{\varepsilon^{\nn,n}_{\subsgd}} 
		\\
		&\leq 17W_{\tau} R \sqrt{ \frac{\log\sbr{\frac{1}{\delta'}}}{M^{\mu}_{\subsgd}}  } +  2R  \Bigg( \ex_{\begin{subarray}{l} \tau^{(1)} \sim \cO^{n}_{\hf}\\ \tau^{(2)} \sim \cO^{\pi^{\base}}_{s_{\init}} \end{subarray}}  \Bigg[ 2 \bigg( \sum_{h=0}^{H(\tau^{(1)})}\nbr{\psi_{0}(s^{(1)}_{h},a^{(1)}_{h})-\psi_{\mu^j}(s^{(1)}_{h},a^{(1)}_{h})}_2 
		\\
		&\quad\ + \sum_{h=0}^{H(\tau^{(2)})}\nbr{\psi_{0}(s^{(2)}_{h},a^{(2)}_{h})-\psi_{\mu^j}(s^{(2)}_{h},a^{(2)}_{h})}_2  \bigg)
		+ 4 W_{\tau}  \bigg( \sum_{h=0}^{H(\tau^{(1)})}\abr{h_{0}(s^{(1)}_{h},a^{(1)}_{h};\mu^j)-h(s^{(1)}_{h},a^{(1)}_{h};\mu^j)} 
		\\
		&\quad\ +  \sum_{h=0}^{H(\tau^{(2)})}\abr{h_{0}(s^{(2)}_{h},a^{(2)}_{h};\mu^j)-h(s^{(2)}_{h},a^{(2)}_{h};\mu^j)} \bigg) \Bigg] \Bigg) 
		\\
		&\leq 17W_{\tau} R \sqrt{ \frac{\log\sbr{\frac{1}{\delta'}}}{M^{\mu}_{\subsgd}}  }
		+  \frac{2R}{1-\gamma}  \Bigg( \ex_{\begin{subarray}{l} s^{(1)} \sim d^{n}_{\hf}\\ a^{(2)} \sim d_{\base} \end{subarray}}  \Bigg[ 2 \bigg( \nbr{\psi_{0}(s^{(1)},a^{(1)})-\psi_{\mu^j}(s^{(1)},a^{(1)})}_2 
		\\
		&\quad\ + \nbr{\psi_{0}(s^{(2)},a^{(2)})-\psi_{\mu^j}(s^{(2)},a^{(2)})}_2 \bigg) 
		+ 4 W_{\tau}   \bigg( \abr{h_{0}(s^{(1)},a^{(1)};\mu^j)-h(s^{(1)},a^{(1)};\mu^j)} 
		\\
		&\quad\ +  \abr{h_{0}(s^{(2)},a^{(2)};\mu^j)-h(s^{(2)},a^{(2)};\mu^j)} \bigg) \Bigg] \Bigg) 
		\\
		&\overset{\textup{(a)}}{\leq} 17W_{\tau} R \sqrt{ \frac{\log\sbr{\frac{1}{\delta'}}}{M^{\mu}_{\subsgd}}  }
		+  \frac{2R}{1-\gamma}  \sbr{\frac{4 \sqrt{c_{\scale} R}}{\sqrt{\underline{c}} m^{\frac{1}{4}}}  +  \frac{16 W_{\tau}  \sqrt{c_{\scale} R^3}}{ \sqrt{\underline{c}} m^{\frac{1}{4}} }}
		\\
		&\leq 17W_{\tau} R \sqrt{ \frac{\log\sbr{\frac{1}{\delta'}}}{M^{\mu}_{\subsgd}}  }
		+ \frac{40 R^2 W_{\tau}  \sqrt{c_{\scale} R}}{ (1-\gamma) \sqrt{ \underline{c}} m^{\frac{1}{4}} } ,
	\end{align*}
	where inequality (a) uses Assumption~\ref{assumption:scale_theta_0}.
\end{proof}


Let $c^{\nn}_{\mle}:= (2+\exp(-2W_{\tau}(\sqrt{m}\bar{c}+R))+\exp(2W_{\tau}(\sqrt{m}\bar{c}+R)))^{-1}$.

\begin{lemma}[MLE] \label{lemma:nn_mle}
	Assume that event $\cE_{\init}\cap\cE_{\tau}\cap\cE^{\nn}_{\mu}$ holds. Then, for any phase $n\geq0$, we have that with probability at least $1-2\delta'$,
	\begin{align*}
		\nbr{\mu^n-\mu^{\supproj}_{r}}_{\hat{\Sigma}^{\nn,n}_{\hf}} 
		&\leq \frac{1}{2c^{\nn}_{\mle}} \sqrt{\frac{5md \log \sbr{\frac{1}{\delta'}}}{M_{\hf}}} + \frac{3}{2c^{\nn}_{\mle}} \sqrt{\sum_{i=1}^{M_{\hf}} (\ex_{y_i}[V_i])^2} \sbr{ \frac{\log \sbr{\frac{1}{\delta'}}}{M_{\hf}}  }^{\frac{1}{4}}  \\&\quad\ + \sqrt{\frac{\varepsilon^{\nn,n}_{\subsgd}}{c^{\nn}_{\mle}} } + 2R\sqrt{\frac{\zeta_{\hf}}{n}} :=\varepsilon^{\nn,n}_{\mle} .
	\end{align*}

	In other words, defining event
	\begin{align*}
		\cE^{\nn}_{\mle}:=\lbr{ \nbr{\mu^n-\mu^{\supproj}_{r}}_{\hat{\Sigma}^{\nn,n}_{\hf}} \leq \varepsilon^{\nn,n}_{\mle} ,\ \forall 0\leq n\leq N-1 } ,
	\end{align*}
	we have $\Pr[\cE^{\nn}_{\mle}] \geq 1-2N\delta'$.
	
	Furthermore, we have
	\begin{align*}
		\ex_{\begin{subarray}{l} \{\tau^{(1)}_i\}_{i=1}^{M_{\hf}} \sim \cO^{n}_{\hf} \\ \{\tau^{(2)}_i\}_{i=1}^{M_{\hf}} \sim \cO^{\pi^{\base}}_{s_{\init}} \end{subarray}} \mbr{\varepsilon^{\nn,n}_{\mle}}
		&\leq \frac{1}{2c^{\nn}_{\mle}} \sqrt{\frac{5md \log \sbr{\frac{1}{\delta'}}}{M_{\hf}}} + \frac{1}{\sqrt{c^{\nn}_{\mle}} } \sbr{ 17W_{\tau} R \sqrt{ \frac{\log\sbr{\frac{1}{\delta'}}}{M^{\mu}_{\subsgd}}  }}^{\frac{1}{2}}
		\\
		&\quad\ + 2R\sqrt{\frac{\zeta_{\hf}}{n}} + \frac{19 c_{\scale}^{\frac{1}{4}} R^{\frac{5}{4}} M_{\hf}^{\frac{1}{4}}   \sqrt{W_{\tau} \exp(4W_{\tau})} }{\underline{c}^{\frac{1}{4}} c^{\nn}_{\mle} \sqrt{1-\gamma}  m^{\frac{1}{8}} } \log \sbr{\frac{1}{\delta'}} .
	\end{align*}
	
	Here we make the convention that $\frac{\zeta_{\hf}}{n}:=\zeta_{\hf}$.

	
\end{lemma}
\begin{proof}
	Since
	\begin{align*}
		\nabla_{\mu}^2 L(\mu) \!&=\! \frac{1}{M_{\hf}} \sum_{i=1}^{M_{\hf}} \sbr{ \sbr{ \frac{ \indicator{y_i=1} \exp\sbr{  -( \tilde{\psi}_0^{\tau^{(1)}_i,\tau^{(2)}_i} )^\top \mu } }{ \sbr{1+ \exp\sbr{  -( \tilde{\psi}_0^{\tau^{(1)}_i,\tau^{(2)}_i} )^\top \mu }}^2 } \!+\! \frac{ \indicator{ y_i=0 } (\tilde{\psi}_0^{\tau^{(1)}_i,\tau^{(2)}_i})^\top \mu }{ \sbr{1+ \exp\sbr{  (\tilde{\psi}_0^{\tau^{(1)}_i,\tau^{(2)}_i})^\top \mu }}^2 } } \! \tilde{\psi}_0^{\tau^{(1)}_i,\tau^{(2)}_i} (\tilde{\psi}_0^{\tau^{(1)}_i,\tau^{(2)}_i})^\top \!\! }
		\\
		&\succeq  \frac{c^{\nn}_{\mle}}{M_{\hf}} \sum_{i=1}^{M_{\hf}}  \tilde{\psi}_0^{\tau^{(1)}_i,\tau^{(2)}_i} (\tilde{\psi}_0^{\tau^{(1)}_i,\tau^{(2)}_i})^\top ,
	\end{align*}
	we have that for any $\Delta \in \R^d$, 
	\begin{align*}
		&\quad\ L(\mu^{\supproj}_{r}+\Delta)-L(\mu^{\supproj}_{r})-(\mu^{\supproj}_{r})^\top \Delta 
		\\
		&\geq \Delta^\top \nabla_{\mu}^2 L(\mu^{\supproj}_{r}) \Delta
		\\
		&\geq c^{\nn}_{\mle} \Delta^\top \sbr{\frac{1}{M_{\hf}} \sum_{i=1}^{M_{\hf}} \sbr{ \psi_0(\tau^{(1)}_i) - \psi_0(\tau^{(2)}_i)} \sbr{\psi_0(\tau^{(1)}_i) - \psi_0(\tau^{(2)}_i)}^\top} \Delta .
	\end{align*}

	Using Lemma~\ref{lemma:sgd_mu} and the definition of $\mu^{*}_{\mle}$, we have
	\begin{align*}
		L(\mu^n) \leq L(\mu^{*}_{\mle}) + \varepsilon^{\nn,n}_{\subsgd} \leq L(\mu^{\supproj}_{r}) + \varepsilon^{\nn,n}_{\subsgd} .
	\end{align*}
	
	Then,
	\begin{align*}
		&\quad\ c^{\nn}_{\mle} \sbr{\mu^n-\mu^{\supproj}_{r}}^\top \sbr{\frac{1}{M_{\hf}} \sum_{i=1}^{M_{\hf}} \sbr{ \psi_0(\tau^{(1)}_i) - \psi_0(\tau^{(2)}_i)} \sbr{\psi_0(\tau^{(1)}_i) - \psi_0(\tau^{(2)}_i)}^\top} \sbr{\mu^n-\mu^{\supproj}_{r}} 
		\\
		&\leq L(\mu^n) - L(\mu^{\supproj}_{r}) - \nabla_{\mu} L(\mu^{\supproj}_{r})^\top \sbr{\mu^n-\mu^{\supproj}_{r}} 
		\\
		&\leq - \nabla_{\mu} L(\mu^{\supproj}_{r})^\top \sbr{\mu^n-\mu^{\supproj}_{r}} + \varepsilon^{\nn,n}_{\subsgd} ,
	\end{align*}
	which implies 
	\begin{align*}
		c^{\nn}_{\mle} \nbr{\mu^n-\mu^{\supproj}_{r}}_{\hat{\Sigma}^{\nn,n}_{\hf}}^2 \leq  \nbr{\nabla_{\mu} L(\mu^{\supproj}_{r})}_{\sbr{\hat{\Sigma}^{\nn,n}_{\hf}}^{-1}} \nbr{\mu^n-\mu^{\supproj}_{r}}_{\hat{\Sigma}^{\nn,n}_{\hf}} + \varepsilon^{\nn,n}_{\subsgd} +  \frac{4c^{\nn}_{\mle}\zeta_{\hf} R^2}{n}  .
	\end{align*}
	
	By analysis for  quadratic functions, we have
	\begin{align}
		\nbr{\mu^n-\mu^{\supproj}_{r}}_{\hat{\Sigma}^{\nn,n}_{\hf}} \leq \frac{1}{2c^{\nn}_{\mle}} \nbr{\nabla_{\mu} L(\mu^{\supproj}_{r})}_{\sbr{\hat{\Sigma}^{\nn,n}_{\hf}}^{-1}} + \sqrt{\frac{\varepsilon^{\nn,n}_{\subsgd}}{c^{\nn}_{\mle}} } + 2R\sqrt{\frac{\zeta_{\hf}}{n} } . \label{eq:Delta_norm_quadratic_function}
	\end{align}
	
	Let
	\begin{align*}
		V_i &= - \frac{ \indicator{y_i=1} \exp\sbr{  -( \tilde{\psi}_0^{\tau^{(1)}_i,\tau^{(2)}_i} )^\top \mu^{\supproj}_{r} } }{ 1+ \exp\sbr{  -( \tilde{\psi}_0^{\tau^{(1)}_i,\tau^{(2)}_i} )^\top \mu^{\supproj}_{r} } } + \frac{ \indicator{ y_i=0 } }{ 1+ \exp\sbr{  - (\tilde{\psi}_0^{\tau^{(1)}_i,\tau^{(2)}_i})^\top \mu^{\supproj}_{r} } } , \ \forall i \in [M_{\hf}] ,
		\\
		V &= [V_1,\dots,V_{M_{\hf}}]^\top  \ \in \R^{M_{\hf}} ,
		\\
		X &= [(\psi_0^{\tau^{(1)}_1,\tau^{(2)}_1})^\top ;\dots; (\psi_0^{\tau^{(1)}_{M_{\hf}},\tau^{(2)}_{M_{\hf}}})^\top ]  \ \in \R^{M_{\hf} \times d} ,
		\\
		X^\top &= [\psi_0^{\tau^{(1)}_1,\tau^{(2)}_1} ,\dots, \psi_0^{\tau^{(1)}_{M_{\hf}},\tau^{(2)}_{M_{\hf}}}]  \ \in \R^{d \times M_{\hf}} ,
	\end{align*}
	and then
	\begin{align*}
		\nabla_{\mu} L(\mu^{\supproj}_{r}) &= \frac{1}{M_{\hf}} \sum_{i=1}^{M_{\hf}} V_i \tilde{\psi}_0^{\tau^{(1)}_i,\tau^{(2)}_i} = \frac{1}{M_{\hf}} X^\top V ,
		\\
		\hat{\Sigma}^{\nn,n}_{\hf} &= \frac{1}{M_{\hf}} X^\top X + \frac{\zeta_{\hf}}{n} I .
	\end{align*}

	For any $i \in [M_{\hf}]$, we have $|V_i|\leq 1$ and
	\begin{align*}
		\ex_{y_i}[V_i] &= \ex_{y_i} \mbr{ - \frac{ \indicator{y_i=1} \exp\sbr{  -( \tilde{\psi}_0^{\tau^{(1)}_i,\tau^{(2)}_i} )^\top \mu^{\supproj}_{r} } }{ 1+ \exp\sbr{  -( \tilde{\psi}_0^{\tau^{(1)}_i,\tau^{(2)}_i} )^\top \mu^{\supproj}_{r} } } + \frac{ \indicator{ y_i=0 } }{ 1+ \exp\sbr{  - (\tilde{\psi}_0^{\tau^{(1)}_i,\tau^{(2)}_i})^\top \mu^{\supproj}_{r} } } }
		\\
		&= - \frac{ 1 }{ 1+ \exp\sbr{  - \tilde{r}(\tau^{(1)}_i,\tau^{(2)}_i)  } } \cdot \frac{ \exp\sbr{  -( \tilde{\psi}_0^{\tau^{(1)}_i,\tau^{(2)}_i} )^\top \mu^{\supproj}_{r} } }{ 1+ \exp\sbr{  -( \tilde{\psi}_0^{\tau^{(1)}_i,\tau^{(2)}_i} )^\top \mu^{\supproj}_{r} } } 
		\\
		&\quad\ + \frac{ \exp\sbr{  - \tilde{r}(\tau^{(1)}_i,\tau^{(2)}_i)  } }{ 1+ \exp\sbr{  - \tilde{r}(\tau^{(1)}_i,\tau^{(2)}_i)  } } \cdot \frac{ 1 }{ 1+ \exp\sbr{  - (\tilde{\psi}_0^{\tau^{(1)}_i,\tau^{(2)}_i})^\top \mu^{\supproj}_{r} } } 
		\\
		&= \frac{ \exp\sbr{  - \tilde{r}(\tau^{(1)}_i,\tau^{(2)}_i)  } - \exp\sbr{  -( \tilde{\psi}_0^{\tau^{(1)}_i,\tau^{(2)}_i} )^\top \mu^{\supproj}_{r} } }{ \sbr{1+ \exp\sbr{  - \tilde{r}(\tau^{(1)}_i,\tau^{(2)}_i)  }} \sbr{1+ \exp\sbr{  -( \tilde{\psi}_0^{\tau^{(1)}_i,\tau^{(2)}_i} )^\top \mu^{\supproj}_{r} }} } .
	\end{align*}
	
	Then,
	\begin{align*}
		|\ex_{y_i}[V_i]| &\leq \exp(2W_{\tau}) \abr{ ( \tilde{\psi}_0^{\tau^{(1)}_i,\tau^{(2)}_i} )^\top \mu^{\supproj}_{r} - \tilde{r}(\tau^{(1)}_i,\tau^{(2)}_i) } 
		\\
		&= \exp(2W_{\tau}) \Bigg| \sbr{\sum_{h=0}^{H(\tau^{(1)}_i)}f_0(s^{(1)}_{i,h},a^{(1)}_{i,h};\mu^{\supproj}_{r}) - \sum_{h=0}^{H(\tau^{(2)}_i)}f_0(s^{(2)}_{i,h},a^{(2)}_{i,h};\mu^{\supproj}_{r})}  \\&\quad\ - \sbr{\sum_{h=0}^{H(\tau^{(1)}_i)}r(s^{(1)}_{i,h},a^{(1)}_{i,h}) - \sum_{h=0}^{H(\tau^{(2)}_i)}r(s^{(2)}_{i,h},a^{(2)}_{i,h})} \Bigg|  
		\\
		&\leq \exp(2W_{\tau}) \Bigg( \sum_{h=0}^{H(\tau^{(1)}_i)} \abr{f_0(s^{(1)}_{i,h},a^{(1)}_{i,h};\mu^{\supproj}_{r}) - r(s^{(1)}_{i,h},a^{(1)}_{i,h}) }  \\&\quad\ + \sum_{h=0}^{H(\tau^{(2)}_i)} \abr{f_0(s^{(2)}_{i,h},a^{(2)}_{i,h};\mu^{\supproj}_{r}) - r(s^{(2)}_{i,h},a^{(2)}_{i,h})}  \Bigg) .
	\end{align*}
	
	Let $D:=\frac{1}{M_{\hf}^2} X ( \hat{\Sigma}^{\nn,n}_{\hf} )^{-1} X^\top=\frac{1}{M_{\hf}^2} X ( \frac{1}{M_{\hf}} X^\top X +\frac{\zeta_{\hf}}{n} I )^{-1} X^\top \ \in \R^{M_{\hf} \times M_{\hf}}$.
	
	Then, 
	\begin{align*}
		\nbr{\nabla_{\mu} L(\mu^{\supproj}_{r})}_{\sbr{\hat{\Sigma}^{\nn,n}_{\hf}}^{-1}}^2 &= \nabla_{\mu} L(\mu^{\supproj}_{r})^\top \sbr{ \hat{\Sigma}^{\nn,n}_{\hf}}^{-1} \nabla_{\mu} L(\mu^{\supproj}_{r})
		\\
		&= \frac{1}{M_{\hf}^2} V^\top X \sbr{ \hat{\Sigma}^{\nn,n}_{\hf} }^{-1} X^\top V 
		\\
		&= V^\top D V .
	\end{align*}
	
	
	Since $D$ is positive semi-definite, let $\lambda_1 \geq \dots \geq \lambda_{M_{\hf}}\geq 0$ denote the eigenvalues of $D$.
	
	We bound $\trace(D)$, $\nbr{D}$, $\trace\sbr{D \ex[V] \ex[V]^\top}$ and $\frac{\nbr{D}^2}{\trace\sbr{D^2}}$ as follows.
	\begin{align*}
		\trace(D) = \trace \sbr{ \frac{1}{M_{\hf}^2} X \sbr{ \frac{1}{M_{\hf}} X^\top X +\frac{\zeta_{\hf}}{n} I }^{-1} X^\top }
		= \frac{1}{M_{\hf}} \trace \sbr{ \sbr{ X^\top X + \frac{M_{\hf} \zeta_{\hf}}{n} I }^{-1} X^\top X }
		\leq \frac{d}{M_{\hf}} ,
	\end{align*}
	\begin{align*}
		\nbr{D} = \nbr{ \frac{1}{M_{\hf}^2} X \sbr{ \frac{1}{M_{\hf}} X^\top X +\frac{\zeta_{\hf}}{n} I }^{-1} X^\top }
		= \nbr{ \frac{1}{M_{\hf}} X \sbr{ X^\top X + \frac{M_{\hf} \zeta_{\hf}}{n} I }^{-1} X^\top }
		\leq \frac{1}{M_{\hf}} ,
	\end{align*}
	\begin{align*}
		\trace\sbr{D \ex[V] \ex[V]^\top} = \trace\sbr{\ex[V]^\top D \ex[V] }
		= \ex[V]^\top D \ex[V] 
		\leq \nbr{\ex[V]}_2^2 \nbr{D} 
		\leq \frac{\sum_{i=1}^{M_{\hf}} (\ex_{y_i}[V_i])^2}{M_{\hf}} ,
	\end{align*}
	and 
	\begin{align*}
		\frac{\nbr{D}^2}{\trace\sbr{D^2}} \overset{\textup{(a)}}{\leq} \frac{M_{\hf} \nbr{D}^2}{\sbr{\trace\sbr{D}}^2}
		= \frac{M_{\hf} \lambda_1^2(D)}{\sbr{\sum_{i=1}^{M_{\hf}} \lambda_i(D)}^2}
		\leq \frac{M_{\hf} \lambda_1^2(D)}{\sum_{i=1}^{M_{\hf}} \lambda_i^2(D)}
		\leq M_{\hf} ,
	\end{align*}
	where inequality (a) is due to $\trace\sbr{D^2} \geq \frac{(\trace\sbr{D})^2}{M_{\hf}}$.
	
	Let $\delta' \leq \frac{1}{e}$. According to Lemma~\ref{lemma:con_Ax}, we have that with probability at least $1-2\delta'$,
	\begin{align}
		\nbr{\nabla_{\mu} L(\mu^{\supproj}_{r})}_{\sbr{\hat{\Sigma}^{\nn,n}_{\hf}}^{-1}}^2 &\leq \trace\sbr{D} + 2 \sqrt{ \sbr{\trace\sbr{D}}^2 \log \sbr{\frac{1}{\delta'}}} +2\nbr{D} \log \sbr{\frac{1}{\delta'}} 
		\nonumber\\
		&\quad\ + \trace \sbr{ D \ex[V]\ex[V]^\top } \sbr{ 1+2 \sqrt{ M_{\hf} \log \sbr{\frac{1}{\delta'}} } }
		\nonumber\\
		&\leq \frac{md}{M_{\hf}} + \frac{2md}{M_{\hf}} \sqrt{ \log \sbr{\frac{1}{\delta'}}} + \frac{2}{M_{\hf}} \log \sbr{\frac{1}{\delta'}} + \frac{\sum_{i=1}^{M_{\hf}} (\ex_{y_i}[V_i])^2}{M_{\hf}} \sbr{ 1+2 \sqrt{ M_{\hf} \log \sbr{\frac{1}{\delta'}} } }
		\nonumber\\
		&\leq \frac{5md \log \sbr{\frac{1}{\delta'}}}{M_{\hf}} + 3 \sbr{\sum_{i=1}^{M_{\hf}} (\ex_{y_i}[V_i])^2} \sqrt{ \frac{\log \sbr{\frac{1}{\delta'}}}{M_{\hf}}  } . \label{eq:nabla_L_upper_bound}
	\end{align}
	
	Plugging Eq.~\eqref{eq:nabla_L_upper_bound} into Eq.~\eqref{eq:Delta_norm_quadratic_function}, we have
	\begin{align*}
		&\quad\ \nbr{\mu^n-\mu^{\supproj}_{r}}_{\hat{\Sigma}^{\nn,n}_{\hf}}
		\\
		&\leq \frac{1}{2c^{\nn}_{\mle}} \sqrt{\frac{5md \log \sbr{\frac{1}{\delta'}}}{M_{\hf}} + 3 \sbr{\sum_{i=1}^{M_{\hf}} (\ex_{y_i}[V_i])^2} \sqrt{ \frac{\log \sbr{\frac{1}{\delta'}}}{M_{\hf}}  } } + \sqrt{\frac{\varepsilon^{\nn,n}_{\subsgd}}{c^{\nn}_{\mle}} } + 2R\sqrt{\frac{\zeta_{\hf}}{n}}
		\\
		&\leq \frac{1}{2c^{\nn}_{\mle}} \sqrt{\frac{5md \log \sbr{\frac{1}{\delta'}}}{M_{\hf}}} + \frac{3}{2c^{\nn}_{\mle}} \sqrt{\sum_{i=1}^{M_{\hf}} (\ex_{y_i}[V_i])^2} \sbr{ \frac{\log \sbr{\frac{1}{\delta'}}}{M_{\hf}}  }^{\frac{1}{4}}  + \sqrt{\frac{\varepsilon^{\nn,n}_{\subsgd}}{c^{\nn}_{\mle}} } + 2R\sqrt{\frac{\zeta_{\hf}}{n}} :=\varepsilon^{\nn,n}_{\mle} .
	\end{align*}
	
	Next, we handle the term $\ex_{y_i}[V_i]$. For any $i \in [M_{\hf}]$,
	\begin{align*}
		\ex_{\begin{subarray}{l} \tau^{(1)}_i \sim \cO^{n}_{\hf}\\ \tau^{(2)}_i \sim \cO^{\pi^{\base}}_{s_{\init}} \end{subarray}} \mbr{(\ex_{y_i}[V_i])^2} &\leq \exp(4W_{\tau}) W_{\tau} \ex_{\ex_{\begin{subarray}{l} \tau^{(1)}_i \sim \cO^{n}_{\hf}\\ \tau^{(2)}_i \sim \cO^{\pi^{\base}}_{s_{\init}} \end{subarray}}} \Bigg[   \sum_{h=0}^{H(\tau^{(1)}_i)} \sbr{f_0(s^{(1)}_{i,h},a^{(1)}_{i,h};\mu^{\supproj}_{r}) - r(s^{(1)}_{i,h},a^{(1)}_{i,h}) }^2  \\&\quad\ + \sum_{h=0}^{H(\tau^{(2)}_i)} \sbr{f_0(s^{(2)}_{i,h},a^{(2)}_{i,h};\mu^{\supproj}_{r}) - r(s^{(2)}_{i,h},a^{(2)}_{i,h})}^2 \Bigg] 
		\\
		&= \frac{\exp(4W_{\tau}) W_{\tau}}{1-\gamma} \ex_{\ex_{\begin{subarray}{l} (s^{(1)},a^{(1)}) \sim d^{n}_{\hf}\\ (s^{(2)},a^{(2)}) \sim d_{\base} \end{subarray}}} \Bigg[    \sbr{f_0(s^{(1)}_{h},a^{(1)}_{h};\mu^{\supproj}_{r}) - r(s^{(1)}_{h},a^{(1)}_{h}) }^2  \\&\quad\ + \sbr{f_0(s^{(2)}_{h},a^{(2)}_{h};\mu^{\supproj}_{r}) - r(s^{(2)}_{h},a^{(2)}_{h})}^2 \Bigg] 
		\\
		&\leq \frac{32R^2W_{\tau}\exp(4W_{\tau}) }{(1-\gamma)m} \log\sbr{\frac{1}{\delta'}}
		.
	\end{align*}
	
	Then, we have
	\begin{align*}
		\ex_{\begin{subarray}{l} \{\tau^{(1)}_i\}_{i=1}^{M_{\hf}} \sim \cO^{n}_{\hf} \\ \{\tau^{(2)}_i\}_{i=1}^{M_{\hf}} \sim \cO^{\pi^{\base}}_{s_{\init}} \end{subarray}} \mbr{\varepsilon^{\nn,n}_{\mle}}
		&\leq \frac{1}{2c^{\nn}_{\mle}} \sqrt{\frac{5md \log \sbr{\frac{1}{\delta'}}}{M_{\hf}}} + \frac{12 R M_{\hf}^{\frac{1}{4}} }{c^{\nn}_{\mle}} \sqrt{\frac{W_{\tau}\exp(4W_{\tau})  }{(1-\gamma)m}
		} \log \sbr{\frac{1}{\delta'}}   \\
		&\quad\ + \sqrt{\frac{\ex_{\begin{subarray}{l} \{\tau^{(1)}_i\}_{i=1}^{M_{\hf}} \sim \cO^{n}_{\hf} ,\ \{\tau^{(2)}_i\}_{i=1}^{M_{\hf}} \sim \cO^{\pi^{\base}}_{s_{\init}} \end{subarray}} \mbr{\varepsilon^{\nn,n}_{\subsgd}}}{c^{\nn}_{\mle}} } + 2R\sqrt{\frac{\zeta_{\hf}}{n}} 
		\\
		&\leq \frac{1}{2c^{\nn}_{\mle}} \sqrt{\frac{5md \log \sbr{\frac{1}{\delta'}}}{M_{\hf}}} + \frac{12 R M_{\hf}^{\frac{1}{4}} }{c^{\nn}_{\mle}} \sqrt{\frac{W_{\tau}\exp(4W_{\tau})  }{(1-\gamma)m}
		} \log \sbr{\frac{1}{\delta'}}   \\
		&\quad\ + \frac{1}{\sqrt{c^{\nn}_{\mle}} } \sbr{ 17W_{\tau} R \sqrt{ \frac{\log\sbr{\frac{1}{\delta'}}}{M^{\mu}_{\subsgd}}  }}^{\frac{1}{2}}
		+ \frac{1}{\sqrt{c^{\nn}_{\mle}} } \sbr{\frac{40 R^2 W_{\tau}  \sqrt{c_{\scale} R}}{ (1-\gamma) \sqrt{ \underline{c}} m^{\frac{1}{4}} }}^{\frac{1}{2}}
		+ 2R\sqrt{\frac{\zeta_{\hf}}{n}} 
		\\
		&\leq \frac{1}{2c^{\nn}_{\mle}} \sqrt{\frac{5md \log \sbr{\frac{1}{\delta'}}}{M_{\hf}}} + \frac{1}{\sqrt{c^{\nn}_{\mle}} } \sbr{ 17W_{\tau} R \sqrt{ \frac{\log\sbr{\frac{1}{\delta'}}}{M^{\mu}_{\subsgd}}  }}^{\frac{1}{2}}
		\\
		&\quad\ + 2R\sqrt{\frac{\zeta_{\hf}}{n}} + \frac{19 c_{\scale}^{\frac{1}{4}} R^{\frac{5}{4}} M_{\hf}^{\frac{1}{4}}   \sqrt{W_{\tau} \exp(4W_{\tau})} }{\underline{c}^{\frac{1}{4}} c^{\nn}_{\mle} \sqrt{1-\gamma}  m^{\frac{1}{8}} } \log \sbr{\frac{1}{\delta'}}
		.
	\end{align*}
	
\end{proof}


\begin{lemma}\label{lemma:nn_varsigma}
	Assume that event $\cE_{\init}\cap\cE_{\tau}\cap\cE^{\nn}_{\mu}\cap\cE^{\nn}_{\mle}\cap\cE^{\nn}_{\cover}$ holds. Then, for any phase $n \geq 0$ and iteration $t \geq 0$,
	\begin{align*}
		&\quad\ \ex_{(s,a) \sim \rho^n_{\cover},\ \{\tau^{(1)}_i\}_{i=1}^{M_{\hf}} \sim \cO^{n}_{\hf} ,\ \{\tau^{(2)}_i\}_{i=1}^{M_{\hf}}} \mbr{ \abr{ Q^{\pi^t}(s,a;\hat{r}^n+b^n) - Q^{\pi^t}(s,a;r+b^n) } }
		\\
		&\leq 2 \ex_{\begin{subarray}{l} \{\tau^{(1)}_i\}_{i=1}^{M_{\hf}} \sim \cO^{n}_{\hf} \\ \{\tau^{(2)}_i\}_{i=1}^{M_{\hf}} \sim \cO^{\pi^{\base}}_{s_{\init}}  \end{subarray}}\mbr{\varepsilon^{\nn,n}_{\mle} } \cdot \ex_{\tau \sim \cO^{\pi^t}_{\rho^n_{\cover}}}  \mbr{ \nbr{\sum_{h=0}^{H(\tau)} \psi_0(s_h,a_h)}_{(\Sigma^{\nn,n}_{\hf})^{-1}}  } + \frac{6 \sqrt{c_{\scale} R^3 \log\sbr{\frac{1}{\delta'}} } }{ (1-\gamma) \sqrt{\underline{c}} m^{\frac{1}{4}}} := \varsigma^{\nn,\pi^t  }_{\rho^n_{\cover}} .
	\end{align*}
\end{lemma}
\begin{proof}
	We have
	\begin{align*}
		&\quad\ \abr{ Q^{\pi^t}(s,a;\hat{r}^n+b^n) - Q^{\pi^t}(s,a;r+b^n) } 
		\\
		&= \abr{\ex_{\tau \sim \cO^{\pi^t}_{s,a}} \mbr{\sum_{h=0}^{H(\tau)} \sbr{h(s_h,a_h;\mu^n) - r(s_h,a_h)} }}
		\\
		&= \Bigg| \ex_{\tau \sim \cO^{\pi^t}_{s,a}} \Bigg[ \sum_{h=0}^{H(\tau)} \Big( h(s_h,a_h;\mu^n) - h_0(s_h,a_h;\mu^n) + h_0(s_h,a_h;\mu^n) - h_0(s_h,a_h;\mu^{\supproj}_{r}) \\&\quad\ + h_0(s_h,a_h;\mu^{\supproj}_{r}) - r(s_h,a_h) \Big) \Bigg] \Bigg|
		\\
		&\leq \ex_{\tau \sim \cO^{\pi^t}_{s,a}} \mbr{ \abr{\sum_{h=0}^{H(\tau)}  \sbr{h(s_h,a_h;\mu^n) - h_0(s_h,a_h;\mu^n)} } } +  \ex_{\tau \sim \cO^{\pi^t}_{s,a}}  \mbr{\abr{\sum_{h=0}^{H(\tau)} \psi_0(s_h,a_h)^\top \sbr{\mu^n - \mu^{\supproj}_{r}}}} \\&\quad\ + \ex_{\tau \sim \cO^{\pi^t}_{s,a}} \mbr{ \abr{\sum_{h=0}^{H(\tau)}  \sbr{h_0(s_h,a_h;\mu^{\supproj}_{r}) - r(s_h,a_h)} } }
		\\
		&\leq \ex_{\tau \sim \cO^{\pi^t}_{s,a}}  \mbr{ \nbr{\sum_{h=0}^{H(\tau)} \psi_0(s_h,a_h)}_{(\hat{\Sigma}^{\nn,n}_{\hf})^{-1}} \nbr{\mu^n - \mu^{\supproj}_{r}}_{\hat{\Sigma}^{\nn,n}_{\hf}} } + \frac{1}{1-\gamma} \ex_{(s',a') \sim d^{\pi^t}_{s,a}} \mbr{   \abr{h(s',a';\mu^n) - h_0('s,a';\mu^n)}  }   \\&\quad\ + \frac{1}{1-\gamma} \ex_{(s',a') \sim d^{\pi^t}_{s,a}} \mbr{   \abr{h_0(s',a';\mu^{\supproj}_{r}) - r(s',a')} }
		\\
		&\overset{\textup{(a)}}{\leq} 2 \varepsilon^{\nn,n}_{\mle} \ex_{\tau \sim \cO^{\pi^t}_{s,a}}  \mbr{ \nbr{\sum_{h=0}^{H(\tau)} \psi_0(s_h,a_h)}_{(\Sigma^{\nn,n}_{\hf})^{-1}}  } + \frac{1}{1-\gamma} \ex_{(s',a') \sim d^{\pi^t}_{s,a}} \mbr{   \abr{h(s',a';\mu^n) - h_0('s,a';\mu^n)}  }   \\&\quad\ + \frac{1}{1-\gamma} \ex_{(s',a') \sim d^{\pi^t}_{s,a}} \mbr{   \abr{h_0(s',a';\mu^{\supproj}_{r}) - r(s',a')} } ,
	\end{align*}
	where inequality (a) uses the definition of  $\cE^{\nn}_{\cover}$ and Lemma~\ref{lemma:nn_mle}.
	
	Then, taking $\ex_{(s,a) \sim \rho^n_{\cover}}[\cdot]$ on both sides, we have
	\begin{align*}
		&\quad\ \ex_{(s,a) \sim \rho^n_{\cover}}\mbr{\abr{ Q^{\pi^t}(s,a;\hat{r}^n+b^n) - Q^{\pi^t}(s,a;r+b^n) }} 
		\\
		&\leq 2 \varepsilon^{\nn,n}_{\mle} \ex_{\tau \sim \cO^{\pi^t}_{\rho^n_{\cover}}}  \mbr{ \nbr{\sum_{h=0}^{H(\tau)} \psi_0(s_h,a_h)}_{(\Sigma^{\nn,n}_{\hf})^{-1}}  } + \frac{1}{1-\gamma} \ex_{(s',a') \sim d^{\pi^t}_{\rho^n_{\cover}}} \mbr{   \abr{h(s',a';\mu^n) - h_0('s,a';\mu^n)}  }   
		\\ 
		&\quad\ + \frac{1}{1-\gamma} \ex_{(s',a') \sim d^{\pi^t}_{\rho^n_{\cover}}} \mbr{   \abr{h_0(s',a';\mu^{\supproj}_{r}) - r(s',a')} }
		\\
		&\overset{\textup{(a)}}{\leq} 2 \varepsilon^{\nn,n}_{\mle} \ex_{\tau \sim \cO^{\pi^t}_{\rho^n_{\cover}}}  \mbr{ \nbr{\sum_{h=0}^{H(\tau)} \psi_0(s_h,a_h)}_{(\Sigma^{\nn,n}_{\hf})^{-1}}  } + \frac{2 \sqrt{c_{\scale} R^3}}{ (1-\gamma) \sqrt{\underline{c}} m^{\frac{1}{4}}}   
		+ \frac{4R}{1-\gamma}  \sqrt{\frac{ \log\sbr{\frac{1}{\delta'}}}{m} }
		\\
		&\leq 2 \varepsilon^{\nn,n}_{\mle} \ex_{\tau \sim \cO^{\pi^t}_{\rho^n_{\cover}}}  \mbr{ \nbr{\sum_{h=0}^{H(\tau)} \psi_0(s_h,a_h)}_{(\Sigma^{\nn,n}_{\hf})^{-1}}  } + \frac{6 \sqrt{c_{\scale} R^3 \log\sbr{\frac{1}{\delta'}} } }{ (1-\gamma) \sqrt{\underline{c}} m^{\frac{1}{4}}}  ,
	\end{align*}
	where inequality (a) uses Lemma~\ref{lemma：distance_psi_0_psi_theta} and the definition of event $\cE_{\init}$.
	
	Furthermore, taking $\ex_{\{\tau^{(1)}_i\}_{i=1}^{M_{\hf}} \sim \cO^{n}_{\hf} ,\ \{\tau^{(2)}_i\}_{i=1}^{M_{\hf}} \sim \cO^{\pi^{\base}}_{s_{\init}}}[\cdot]$ on both sides, we have
	\begin{align*}
		&\quad\ \ex_{(s,a) \sim \rho^n_{\cover}}\mbr{\abr{ Q^{\pi^t}(s,a;\hat{r}^n+b^n) - Q^{\pi^t}(s,a;r+b^n) }} 
		\\
		&\leq 2 \ex_{\begin{subarray}{l} \{\tau^{(1)}_i\}_{i=1}^{M_{\hf}} \sim \cO^{n}_{\hf} \\ \{\tau^{(2)}_i\}_{i=1}^{M_{\hf}} \sim \cO^{\pi^{\base}}_{s_{\init}}  \end{subarray}} \mbr{\varepsilon^{\nn,n}_{\mle}} \ex_{\tau \sim \cO^{\pi^t}_{\rho^n_{\cover}}}  \mbr{ \nbr{\sum_{h=0}^{H(\tau)} \psi_0(s_h,a_h)}_{(\Sigma^{\nn,n}_{\hf})^{-1}}  } + \frac{6 \sqrt{c_{\scale} R^3 \log\sbr{\frac{1}{\delta'}} } }{ (1-\gamma) \sqrt{\underline{c}} m^{\frac{1}{4}}} 
		:= \varsigma^{\nn,\pi^t  }_{\rho^n_{\cover}} .
	\end{align*}
	
\end{proof}

In the following, for ease of notation, we use $\ex_{\{\tau^{(1)}_i\}_{i=1}^{M_{\hf}} \sim \cO^{n}_{\hf} ,\ \{\tau^{(2)}_i\}_{i=1}^{M_{\hf}} \sim \cO^{\pi^{\base}}_{s_{\init}}}[\cdot]$ and $\ex_{\hat{r}^n}[\cdot]$ interchangeably.

\begin{lemma} \label{lemma:nn_phi_theta_star_minus_theta_mid}
	Assume that event $\cE_{\init}\cap\cE_{\tau}\cap\cE^{\nn}_{\mu}\cap\cE^{\nn}_{\mle}\cap\cE^{\nn}_{\cover}$ holds. Then, for any phase $n \geq 0$, iteration $t \geq 0$, $s \in \cK^n$ and $a \in \cA$,
	\begin{align*}
		\abr{\psi_0(s,a)^\top \sbr{\theta^{t}_{*} - \theta^{t}_{\submid}}} \leq \sqrt{2\beta \sbr{8 (n+1) W^{\nn}_{Q} \varsigma^{\nn,\pi^t}_{\rho^n_{\cover}} + 4\zeta_{\cover} R^2}} .
	\end{align*}
\end{lemma}

\begin{proof}
	For any phase $n \geq 0$,
	\begin{align}
		&\ \quad \ex_{(s,a) \sim \rho^n_{\cover},\hat{r}^n} \mbr{\sbr{ Q^{\pi^t}(s,a;\hat{r}^n+b^n)- b^n(s,a) - \psi_0(s,a)^\top \theta }^2 - \sbr{  Q^{\pi^t}(s,a;r+b^n)- b^n(s,a) - \psi_0(s,a)^\top \theta }^2}
		\nonumber\\
		&\leq 4W^{\nn}_{Q} \ex_{(s,a) \sim \rho^n_{\cover}, \hat{r}^n} \mbr{ \abr{ Q^{\pi^t}(s,a;\hat{r}^n+b^n) - Q^{\pi^t}(s,a;r+b^n) } }
		\nonumber\\
		&\leq 4W^{\nn}_{Q} \varsigma^{\nn,\pi^t  }_{\rho^n_{\cover}}  , \label{eq:nn_derivation_varsigma1}
	\end{align}
	Here $W^{\nn}_{Q}$ satisfies  $\max\{|Q^{\pi^t}(s,a;r+b^n)- b^n(s,a)|, |\psi_0(s,a)^\top \theta^{t}_{\submid}|,  |\psi_0(s,a)^\top \theta^{t}_{*}|\} \leq W^{\nn}_{Q}$.
	
	Plugging $\theta^{t}_{*}$ into $\theta$, we have that for any fixed $(s,a)$,
	\begin{align}
		&\ \quad  \ex_{(s,a) \sim \rho^n_{\cover}} \mbr{\sbr{  Q^{\pi^t}(s,a;r+b^n)- b^n(s,a) - \psi_0(s,a)^\top \theta^{t}_{*} }^2 } 
		\nonumber\\
		&\geq \ex_{(s,a) \sim \rho^n_{\cover}, \hat{r}^n} \mbr{ \sbr{ Q^{\pi^t}(s,a;\hat{r}^n+b^n)- b^n(s,a) - \psi_0(s,a)^\top \theta^{t}_{*} }^2} - 4W^{\nn}_{Q} \varsigma^{\nn,\pi^t  }_{\rho^n_{\cover}}
		\nonumber\\
		&\overset{\textup{(a)}}{\geq} \ex_{(s,a) \sim \rho^n_{\cover}, \hat{r}^n} \mbr{ \sbr{ Q^{\pi^t}(s,a;\hat{r}^n+b^n)- b^n(s,a) - \psi_0(s,a)^\top \theta^{t}_{\submid} }^2} - 4W^{\nn}_{Q} \varsigma^{\nn,\pi^t  }_{\rho^n_{\cover}} , \label{eq:nn_derivation_varsigma3}
	\end{align}
	where inequality (a) is due to the definition of $\theta^{t}_{\submid}$. 
	
	Furthermore, we have
	\begin{align}
		&\ \quad  \ex_{(s,a) \sim \rho^n_{\cover}} \mbr{ \sbr{ Q^{\pi^t}(s,a;r+b^n)- b^n(s,a) - \psi_0(s,a)^\top \theta^{t}_{\submid} }^2} 
		\nonumber\\
		&\quad\  - \ex_{(s,a) \sim \rho^n_{\cover}} \mbr{\sbr{  Q^{\pi^t}(s,a;r+b^n)- b^n(s,a) - \psi_0(s,a)^\top \theta^{t}_{*} }^2 } 
		\nonumber\\
		&= \ex_{(s,a) \sim \rho^n_{\cover}, \hat{r}^n} \mbr{ \sbr{ Q^{\pi^t}(s,a;\hat{r}^n+b^n)- b^n(s,a) - \psi_0(s,a)^\top \theta^{t}_{\submid} }^2} 
		\nonumber\\
		&\quad\ - \ex_{(s,a) \sim \rho^n_{\cover}} \mbr{\sbr{  Q^{\pi^t}(s,a;r+b^n)- b^n(s,a) - \psi_0(s,a)^\top \theta^{t}_{*} }^2 } 
		\nonumber\\
		& \quad + \ex_{(s,a) \sim \rho^n_{\cover}} \mbr{ \sbr{ Q^{\pi^t}(s,a;r+b^n)- b^n(s,a) - \psi_0(s,a)^\top \theta^{t}_{\submid} }^2} 
		\nonumber\\
		&\quad\ - \ex_{(s,a) \sim \rho^n_{\cover}, \hat{r}^n} \mbr{\sbr{  Q^{\pi^t}(s,a;\hat{r}^n+b^n)- b^n(s,a) - \psi_0(s,a)^\top \theta^{t}_{\submid} }^2 } 
		\nonumber\\
		&\overset{\textup{(a)}}{\leq} 4W^{\nn}_{Q} \varsigma^{\nn,\pi^t  }_{\rho^n_{\cover}} + 4W^{\nn}_{Q} \ex_{(s,a) \sim \rho^n_{\cover}, \hat{r}^n} \mbr{\abr{ Q^{\pi^t}(s,a;\hat{r}^n+b^n) - Q^{\pi^t}(s,a;r+b^n) }}
		\nonumber\\
		&\leq 8 W^{\nn}_{Q} \varsigma^{\nn,\pi^t  }_{\rho^n_{\cover}} , \label{eq:derivation_varsigma4}
	\end{align}
	where inequality (a) uses Lemma~\ref{lemma:tech_sq_diff}.

	On the other hand, it holds that
	\begin{align}
		&\ \quad \ex_{(s,a) \sim \rho^n_{\cover}} \mbr{ \sbr{ Q^{\pi^t}(s,a;r+b^n)- b^n(s,a) - \psi_0(s,a)^\top \theta^{t}_{\submid} }^2} 
		\nonumber\\
		&\quad\ - \ex_{(s,a) \sim \rho^n_{\cover}} \mbr{\sbr{  Q^{\pi^t}(s,a;r+b^n)- b^n(s,a) - \psi_0(s,a)^\top \theta^{t}_{*} }^2 }
		\nonumber\\
		&= \ex_{(s,a) \sim \rho^n_{\cover}} \mbr{ \sbr{\psi_0(s,a)^\top \sbr{\theta^{t}_{*} -  \theta^{t}_{\submid} }}^2 } 
		\nonumber\\
		&\quad\ + 2 \underbrace{\ex_{(s,a) \sim \rho^n_{\cover}} \mbr{ \sbr{ Q^{\pi^t}(s,a;r+b^n)- b^n(s,a) - \psi_0(s,a)^\top \theta^{t}_{*} } \psi_0(s,a)^\top \sbr{\theta^{t}_{*} -  \theta^{t}_{\submid} } }}_{\textup{Term $\Gamma^{\nn}$}\ \geq\ 0} , \label{eq:nn_first_order_opt}
	\end{align}
	where Term $\Gamma^{\nn}$ is non-negative due to the the first-order optimality of $\theta^{t}_{*}$.
	
	Then,
	\begin{align*}
		&\ \quad
		(\theta^{t}_{*} -  \theta^{t}_{\submid})^\top \ex_{(s,a) \sim \rho^n_{\cover}} [\psi_0(s,a)  \psi_0(s,a)^\top] (\theta^{t}_{*} -  \theta^{t}_{\submid})
		\\ 
		&= \ex_{(s,a) \sim \rho^n_{\cover}} \mbr{ \sbr{\psi_0(s,a)^\top \sbr{\theta^{t}_{*} -  \theta^{t}_{\submid} }}^2 }
		\\
		&\overset{\textup{(a)}}{\leq} \ex_{(s,a) \sim \rho^n_{\cover}} \mbr{ \sbr{ Q^{\pi^t}(s,a;r+b^n)- b^n(s,a) - \psi_0(s,a)^\top \theta^{t}_{\submid} }^2} \\& \qquad - \ex_{(s,a) \sim \rho^n_{\cover}} \mbr{\sbr{  Q^{\pi^t}(s,a;r+b^n)- b^n(s,a) - \psi_0(s,a)^\top \theta^{t}_{*} }^2 } 
		\\
		&\leq 8 W^{\nn}_{Q} W^{\nn}_{Q} \varsigma^{\nn,\pi^t  }_{\rho^n_{\cover}} ,
	\end{align*} 
	where inequality (a) uses the same argument as Eq.~\eqref{eq:first_order_opt} (i.e., the first optimality of $\theta^{t}_{*}$).
	
	The above equation implies
	\begin{align*}
		\nbr{\theta^{t}_{*} -  \theta^{t}_{\submid}}_{\Sigma^{\nn,n}_{\cover}}^2
		&= (\theta^{t}_{*} -  \theta^{t}_{\submid})^\top \sbr{(n+1) \ex_{(s,a) \sim \rho^n_{\cover}} [\psi_0(s,a)  \psi_0(s,a)^\top] +\zeta_{\cover} I} (\theta^{t}_{*} -  \theta^{t}_{\submid})
		\\
		&\leq 8 (n+1) W^{\nn}_{Q} \varsigma^{\nn,\pi^t  }_{\rho^n_{\cover}} + 4\zeta_{\cover} R^2 .
	\end{align*}
	
	For any $s \in \cK^n$, using the definition of $\cK^n$ and event $\cE^{\nn}_{\cover}$, we have
	\begin{align}
		\frac{1}{\sqrt{2}} \nbr{ \psi_0(s,a) }_{(\Sigma^{\nn,n}_{\cover})^{-1}} \leq \nbr{ \psi_0(s,a) }_{(\hat{\Sigma}^{\nn,n}_{\cover})^{-1}} \leq \sqrt{\beta} . \label{eq:nn_covariance_matrix_inverse}
	\end{align}
	
	Thus, for any $s \in \cK^n$,
	\begin{align*}
		\abr{\psi_0(s,a)^\top \sbr{\theta^{t}_{*} - \theta^{t}_{\submid}}} &\leq \nbr{ \psi_0(s,a) }_{(\Sigma^{\nn,n}_{\cover})^{-1}} \nbr{ \theta^{t}_{*} - \theta^{t}_{\submid} }_{\Sigma^{\nn,n}_{\cover}}
		\\
		&\leq \sqrt{2\beta \sbr{8 (n+1) W^{\nn}_{Q} \varsigma^{\nn,\pi^t  }_{\rho^n_{\cover}} + 4\zeta_{\cover} R^2}} .
	\end{align*}
\end{proof}

\subsection{Proof of Theorem~\ref{thm:nn_ub_pgrlhf}} \label{apx:nn_main_thm_proof}

For any phase $n=0,\dots,N-1$ and iteration $t=0,\dots,T-1$, let
\begin{align*}
	\theta^{t}_{*}&:= \argmin_{\theta \in \cS_R} \ex_{(s,a)\sim\rho^n_{\cover}} \mbr{ \sbr{f_0(s,a;\theta) - \sbr{Q^{\pi^t}(s,a;r+b^n)-b^n(s,a)} }^2 } ,
	\\
	\theta^{t}_{\submid}&:= \argmin_{\theta \in \cS_R} \ex_{(s,a)\sim\rho^n_{\cover},\hat{r}^n} \mbr{ \sbr{f_0(s,a;\theta) - \sbr{Q^{\pi^t}(s,a;\hat{r}^n+b^n)-b^n(s,a)} }^2 } ,
	\\
	\theta^{t}&\overset{\subsgd}{\approx} \argmin_{\theta \in \cS_R} \ex_{(s,a)\sim\rho^n_{\cover}} \mbr{ \sbr{f_0(s,a;\theta) - \sbr{Q^{\pi^t}(s,a;\hat{r}^n+b^n)-b^n(s,a)} }^2 } .
\end{align*}


Let $\delta':=\frac{\delta}{24N(K+M_{\hf}+M^{\mu}_{\subsgd}+TM^{\theta}_{\subsgd})}$.
For any $n\geq0$, $t\geq0$ and $(s,a) \in \cS \times \cA$, let $\bar{b}^{n,t}(s,a) := b^n(s,a) - \ex_{a' \sim \pi^t(\cdot|s)} \mbr{ b^n(s,a') }$, and for any $w \in \R^{md}$, let $\bar{\psi}_{w}^{t}(s,a) := \psi_{w}(s,a) - \ex_{a' \sim \pi^t(\cdot|s)} \mbr{ \psi_{w}(s,a') }$.

\begin{proof}[Proof of Theorem~\ref{thm:nn_ub_pgrlhf}]
	First, we have $\Pr[\cE_{\init}\cap\cE^{\nn}_{\theta}\cap\cE_{\tau}\cap\cE^{\nn}_{\mu}\cap\cE^{\nn}_{\mle}\cap\cE^{\nn}_{\cover}]\geq 1-6\cdot2N(K+M_{\hf}+M^{\mu}_{\subsgd}+TM^{\theta}_{\subsgd})\cdot 2\delta' = 1-\delta$.
	In the following, we assume that event $\cE_{\init} \cap \cE^{\nn}_{\theta}\cap\cE_{\tau}\cap\cE^{\nn}_{\mu}\cap\cE^{\nn}_{\mle}\cap\cE^{\nn}_{\cover}$ holds. 
	
	For any phase $n=0,\dots,N-1$ and iteration $t=0,\dots,T-1$, we have
	\begin{align}
		&\ \quad  V_{\cM^n}^{\pi^{*,n}}(s_{\init}) - V_{\cM^n}^{\pi^t}(s_{\init}) 
		\nonumber\\
		&\leq  \frac{1}{1-\gamma} \ex_{(s,a) \sim d_{\cM^n;s_{\init}}^{\pi^{*,n}} }\mbr{ A_{\cM_{b^n}}^{\pi^t}(s,a) \cdot \indicator{s\in\cK^n} } 
		\nonumber\\
		&=  \frac{1}{1-\gamma} \ex_{(s,a) \sim d_{\cM^n;s_{\init}}^{\pi^{*,n}}} \bigg[  \sbr{ \bar{\psi}^t_{w^t}(s,a)^\top \theta^{t} + \bar{b}^{n,t}(s,a) } \cdot \indicator{s \in \cK^n}  
		\nonumber\\&\ \quad + \underbrace{\sbr{ A_{\cM_{b^n}}^{\pi^t}(s,a)  - \sbr{ \bar{\psi}^t_{0}(s,a)^\top \theta^{t}_{*} + \bar{b}^{n,t}(s,a)  } } \cdot \indicator{s \in \cK^n} }_{\textup{Term 1}}
		\nonumber\\&\ \quad +  \underbrace{  \bar{\psi}^t_{0}(s,a)^\top \sbr{\theta^{t}_{*} - \theta^{t}_{\submid}} \cdot \indicator{s \in \cK^n}}_{\textup{Term 2}}
		+  \underbrace{  \bar{\psi}^t_{0}(s,a)^\top \sbr{\theta^{t}_{\submid} - \theta^{t} } \cdot \indicator{s \in \cK^n}}_{\textup{Term 3}}
		\nonumber\\&\ \quad + \underbrace{\sbr{\bar{\psi}^t_{0}(s,a) - \bar{\psi}^t_{w^t}(s,a) }^\top \theta^{t} }_{\textup{Term 4}}
		\bigg]  . \label{eq:nn_regret_decomposition}
	\end{align}

	Below we bound Terms 1-4. 
	
	\paragraph{Term 1.} 
	We first bound Term 1.
	
	\begin{align*}
		\textup{Term 1} &=  \ex_{(s,a) \sim d_{\cM^n;s_{\init}}^{\pi^{*,n}} } \mbr{ \sbr{Q_{\cM_{b^n}}^{\pi^t}(s,a)  - \sbr{ \psi_{0}(s,a)^\top \theta^{t}_{*} + b^{n}(s,a) } } \cdot \indicator{s\in\cK^n} } \\&\quad\ 
		- \ex_{s \sim d_{\cM^n;s_{\init}}^{\pi^{*,n}}, a' \sim \pi^t(\cdot|s) } \mbr{ \sbr{ Q_{\cM_{b^n}}^{\pi^t}(s,a') - \sbr{ \psi_{0}(s,a')^\top \theta^{t}_{*} + b^{n}(s,a') } } \cdot \indicator{s\in\cK^n} } 
		\\
		&\overset{\textup{(a)}}{\leq}  \sqrt{\ex_{(s,a) \sim d_{s_{\init}}^{\pi^{*}} } \mbr{ \sbr{Q_{\cM_{b^n}}^{\pi^t}(s,a)  - \sbr{ \psi_{0}(s,a)^\top \theta^{t}_{*} + b^{n}(s,a) } }^2 } } 
		\\&\quad\ 
		+ \sqrt{\ex_{s \sim d_{s_{\init}}^{\pi^{*}}, a' \sim \pi^t(\cdot|s) } \mbr{ \sbr{ Q_{\cM_{b^n}}^{\pi^t}(s,a') - \sbr{ \psi_{0}(s,a')^\top \theta^{t}_{*} + b^{n}(s,a') } }^2 } } 
		\\
		&\leq  2\sqrt{ |\cA| \ex_{(s,a) \sim d_{s_{\init}}^{\star} } \mbr{ \sbr{Q_{\cM_{b^n}}^{\pi^t}(s,a)  - \sbr{ \psi_{0}(s,a)^\top \theta^{t}_{*} + b^{n}(s,a) } }^2 } } 
		\\
		&= 2 \sqrt{|\cA| \varepsilon^{\nn}_{\bias}} ,
	\end{align*}
	where inequality (a) uses Lemma~\ref{lemma:d_star_n_leq_d_star}.
	
	\paragraph{Term 2.}
	Then, we bound Term 2.
	
	Using Lemma~\ref{lemma:nn_phi_theta_star_minus_theta_mid}, we have that for any $s \in \cK^n$ and $a \in \cA$,
	\begin{align*}
		\abr{\psi_0(s,a)^\top \sbr{\theta^{t}_{*} - \theta^{t}_{\submid}}} \leq \sqrt{2\beta \sbr{8 (n+1) W^{\nn}_{Q} \varsigma^{\nn,\pi^t  }_{\rho^n_{\cover}} + 4\zeta_{\cover} R^2}} .
	\end{align*}
	
	Thus, 
	\begin{align*}
		\textup{Term 2} &=  \ex_{(s,a) \sim d_{\cM^n;s_{\init}}^{\pi^{*,n}}} \mbr{ \bar{\psi}^t_{0}(s,a)^\top \sbr{\theta^{t}_{*} - \theta^{t}_{\submid}} \cdot \indicator{s \in \cK^n} } 
		\\
		&\leq  \ex_{(s,a) \sim d_{\cM^n;s_{\init}}^{\pi^{*,n}}} \mbr{ \abr{\psi_{0}(s,a)^\top \sbr{\theta^{t}_{*} - \theta^{t}_{\submid}}} \cdot \indicator{s \in \cK^n} }
		\\
		&\quad\ + \ex_{s \sim d_{\cM^n;s_{\init}}^{\pi^{*,n}}, a' \sim \pi^t(\cdot|s)} \mbr{ \abr{\psi_{0}(s,a')^\top \sbr{\theta^{t}_{*} - \theta^{t}_{\submid}}} \cdot \indicator{s \in \cK^n} } 
		\\
		&\leq 2 \sqrt{2\beta \sbr{8 (n+1) W^{\nn}_{Q} \varsigma^{\nn,\pi^t  }_{\rho^n_{\cover}} + 4\zeta_{\cover} R^2}} .
	\end{align*}

	\paragraph{Term 3.}
	Next, we bound Term 3. 
	
	Using the same argument as Eq.~\eqref{eq:nn_first_order_opt} (i.e., the first optimality of $\theta^{t}_{\submid}$),
	\begin{align*}
		&\quad\  \ex_{(s,a)\sim\rho^n_{\cover}, \hat{r}^n} \mbr{ \sbr{\psi_0(s,a)^\top \theta^{t}_{\submid} - \psi_0(s,a)^\top \theta^{t} }^2 }
		\\
		&\leq \ex_{(s,a)\sim\rho^n_{\cover}, \hat{r}^n} \mbr{ \sbr{\psi_0(s,a)^\top \theta^{t} - \sbr{Q^{\pi^t}(s,a;\hat{r}^n+b^n)-b^n(s,a)} }^2 } \\&\quad\ - \ex_{(s,a)\sim\rho^n_{\cover}, \hat{r}^n} \mbr{ \sbr{\psi_0(s,a)^\top \theta^{t}_{\submid} - \sbr{Q^{\pi^t}(s,a;\hat{r}^n+b^n)-b^n(s,a)} }^2 }
		\\
		&= \ex_{\hat{r}^n} \Bigg[ \ex_{(s,a)\sim\rho^n_{\cover}} \bigg[ \sbr{\psi_0(s,a)^\top \theta^{t} - \sbr{Q^{\pi^t}(s,a;\hat{r}^n+b^n)-b^n(s,a)} }^2  
		\\
		&\quad\ - \sbr{\psi_0(s,a)^\top \theta^{t}_{\submid} - \sbr{Q^{\pi^t}(s,a;\hat{r}^n+b^n)-b^n(s,a)} }^2 \bigg] \bigg| \hat{r}^n \Bigg]
		\\
		&= \ex_{\hat{r}^n} \mbr{ F^{\hat{r}^n}(\theta^t) - F^{\hat{r}^n}(\theta^t_{\submid}) | \hat{r}^n }
		\\
		&\leq \ex_{\hat{r}^n} \mbr{ F^{\hat{r}^n}(\theta^t) - F^{\hat{r}^n}(\theta^{t, \hat{r}^n}_{\submid}) | \hat{r}^n }
		\\
		&\overset{\textup{(a)}}{\leq} \varepsilon^{\nn}_{\stat} .
	\end{align*}
	where inequality (a) is due to Lemma~\ref{lemma:nn_sgd_Q}.

	Then, we have
	\begin{align*}
		\nbr{\theta^{t}_{\submid} - \theta^{t}}_{\Sigma^{\nn,n}_{\cover}}^2 
		&\leq
		\sbr{\theta^{t}_{\submid} - \theta^{t}}^\top \sbr{ (n+1) \ex_{(s,a)\sim \rho^n_{\cover}} \mbr{\psi_0(s,a)\psi_0(s,a)^\top} + \zeta_{\cover} I } \sbr{\theta^{t}_{\submid} - \theta^{t}}
		\\
		&= (n+1) \ex_{(s,a)\sim\rho^n_{\cover}} \mbr{ \sbr{\psi_0(s,a)^\top \theta^{t}_{\submid} - \psi_0(s,a)^\top \theta^{t} }^2 } + 4R^2\zeta_{\cover}
		\\
		&\leq (n+1)\varepsilon^{\nn}_{\stat} + 4R^2\zeta_{\cover} .
	\end{align*}
	
	For any $s \in \cK^n$ and $a \in \cA$,
	\begin{align*}
		\abr{ \psi_{0}(s,a)^\top \sbr{\theta^{t}_{\submid} - \theta^{t} } }
		&\leq \nbr{ \psi_{0}(s,a)}_{\sbr{\Sigma^{\nn,n}_{\cover}}^{-1}} \nbr{\theta^{t}_{\submid} - \theta^{t} }_{\Sigma^{\nn,n}_{\cover}}
		\\
		&\overset{\textup{(a)}}{\leq} \sqrt{ 2\beta \sbr{(n+1) \varepsilon^{\nn}_{\stat} +  4R^2\zeta_{\cover}} } ,
	\end{align*}
	where inequality (a) uses Eq.~\eqref{eq:nn_covariance_matrix_inverse}.
	
	Hence, we have
	\begin{align*}
		\textup{Term 3}&=  \ex_{(s,a) \sim d_{\cM^n;s_{\init}}^{\pi^{*,n}}} \mbr{ \bar{\psi}^t_{0}(s,a)^\top \sbr{\theta^{t}_{\submid} - \theta^{t} } \cdot \indicator{s \in \cK^n} } 
		\\
		&\leq 2\sqrt{ \beta(n+1) \varepsilon^{\nn}_{\stat} } + 4R\sqrt{\beta \zeta_{\cover} } .
	\end{align*}
	
	\paragraph{Term 4.}
	Finally, we bound Term 4 as follows.	
	
	\begin{align*}
		\textup{Term 4}&=  \ex_{(s,a) \sim d_{\cM^n;s_{\init}}^{\pi^{*,n}}} \mbr{ \sbr{\bar{\psi}^t_{0}(s,a) - \bar{\psi}^t_{w^t}(s,a) }^\top \theta^{t} } 
		\\
		&\leq  \ex_{(s,a) \sim d_{\cM;s_{\init}}^{\pi^{*}}} \mbr{ \abr{\sbr{\psi_{0}(s,a) - \psi_{w^t}(s,a) }^\top \theta^{t}} }
		\\
		&\quad\ + |\cA| \ex_{(s,a) \sim d_{\cM;s_{\init}}^{\pi^{*}}} \mbr{ \abr{\sbr{\psi_{0}(s,a') - \psi_{w^t}(s,a') }^\top \theta^{t}} } 
		\\
		&\leq \frac{4 |\cA| \sqrt{c_{\scale} R^3}}{ \sqrt{\underline{c}} m^{\frac{1}{4}}}
		.
	\end{align*}
	
	\paragraph{The Total Suboptimality.}
	
	Combining Lemma~\ref{lemma:optimistic_M_and_true_M} and Eq.~\eqref{eq:nn_regret_decomposition}, we have
	\begin{align*}
		V^{\pi^{*}}(s_{\init}) - V^{\pi^t}(s_{\init})  \leq \textup{RHS in Eq.~\eqref{eq:nn_regret_decomposition}} + \frac{1}{1-\gamma} \sum_{(s,a) \notin \cK^n} d^{\pi^t}_{s_{\init}}(s,a) .
	\end{align*}
	
	Summing over $t=0,\dots,T-1$, dividing $T$ and applying the regret for natural policy gradient (Lemma~\ref{lemma:nn_regret_npg}), we have
	\begin{align*}
		&\ \quad V^{\pi^{*}}(s_{\init}) - V^{\pi^{n+1}}(s_{\init}) 
		\nonumber\\
		&= \frac{1}{T} \sum_{t=0}^{T-1} \sbr{V^{\pi^{*}}(s_{\init}) - V^{\pi^t}(s_{\init}) }
		\nonumber\\
		&\leq \frac{\log(|\cA|)}{(1-\gamma)\eta T} + \frac{\eta W_{S} (W^{\nn}_{\theta})^2 }{(1-\gamma)} + \frac{2 \sqrt{|\cA| \varepsilon^{\nn}_{\bias}}}{1-\gamma} 
		+ \frac{1}{1-\gamma} 8 \sqrt{ \beta (n+1)  W^{\nn}_{Q} } \cdot  \frac{1}{T} \sum_{t=0}^{T-1} \sqrt{ \varsigma^{\nn,\pi^t  }_{\rho^n_{\cover}} } 
		\\
		&\quad\
		+ \frac{12R\sqrt{\beta \zeta_{\cover} }}{1-\gamma}
		+  \frac{2}{1-\gamma} \sqrt{ \beta(n+1) \varepsilon^{\nn}_{\stat}}  
		+ \frac{4 |\cA| \sqrt{c_{\scale} R^3}}{ \sqrt{\underline{c}} m^{\frac{1}{4}}} + \frac{1}{1-\gamma} \sum_{(s,a) \notin \cK^n} d^{\pi^{n+1}}_{s_{\init}}(s,a) .
	\end{align*}
	
	Next, we handle the term $\frac{1}{T} \sum_{t=0}^{T-1} \sqrt{\varsigma^{\nn,\pi^t  }_{\rho^n_{\cover}}}$.
	\begin{align*}
		\frac{1}{N} \sum_{n=0}^{N-1} \sqrt{n+1} \cdot \frac{1}{T} \sum_{t=0}^{T-1} \sqrt{\varsigma^{\nn,\pi^t  }_{\rho^n_{\cover}}} 
		&\leq \frac{1}{NT} \sum_{n=0}^{N-1} \sum_{t=0}^{T-1} \sqrt{(n+1) \varsigma^{\nn,\pi^t  }_{\rho^n_{\cover}}} 
		\\
		&\leq \frac{1}{NT} \sum_{n=0}^{N-1} \sum_{t=0}^{T-1} \Bigg( \Bigg( \frac{N}{c^{\nn}_{\mle}} \sqrt{\frac{5md \log \sbr{\frac{N}{\delta'}}}{M_{\hf}}} + \frac{2N}{\sqrt{c^{\nn}_{\mle}} } \sbr{ 17W_{\tau} R \sqrt{ \frac{\log\sbr{\frac{1}{\delta'}}}{M^{\mu}_{\subsgd}}  }}^{\frac{1}{2}}
		\\
		&\quad\ + 4(n+1)R\sqrt{\frac{\zeta_{\hf}}{n}} + \frac{38 N c_{\scale}^{\frac{1}{4}} R^{\frac{5}{4}} M_{\hf}^{\frac{1}{4}}   \sqrt{W_{\tau} \exp(4W_{\tau})} }{\underline{c}^{\frac{1}{4}} c^{\nn}_{\mle} \sqrt{1-\gamma}  m^{\frac{1}{8}} } \log \sbr{\frac{1}{\delta'}}  \Bigg) \cdot
		\\
		& \quad \ex_{\tau \sim \cO^{\pi^t}_{\rho^n_{\cover}}}  \mbr{ \nbr{\sum_{h=0}^{H(\tau)} \psi_0(s_h,a_h)}_{(\Sigma^{\nn,n}_{\hf})^{-1}}  } + \frac{6 N \sqrt{c_{\scale} R^3 \log\sbr{\frac{1}{\delta'}} } }{ (1-\gamma) \sqrt{\underline{c}} m^{\frac{1}{4}}} \Bigg)^{\frac{1}{2}}
		\\
		&\leq \Bigg( \sqrt{\frac{N}{c^{\nn}_{\mle}}} \sbr{\frac{5md \log \sbr{\frac{N}{\delta'}}}{M_{\hf}}}^{\frac{1}{4}} + \frac{\sqrt{2N}}{\sbr{c^{\nn}_{\mle}}^{\frac{1}{4}} } \sbr{ 17W_{\tau} R \sqrt{ \frac{\log\sbr{\frac{1}{\delta'}}}{M^{\mu}_{\subsgd}}  }}^{\frac{1}{4}}
		\\
		&\quad\ + 4\sqrt{R}N^{\frac{1}{4}}\zeta_{\hf}^{\frac{1}{4}} + \frac{7 \sqrt{N} c_{\scale}^{\frac{1}{8}} R^{\frac{5}{8}} M_{\hf}^{\frac{1}{8}}   \sbr{W_{\tau} \exp(4W_{\tau})}^{\frac{1}{4}} }{\underline{c}^{\frac{1}{8}} \sqrt{c^{\nn}_{\mle}} \sbr{1-\gamma}^{\frac{1}{4}}  m^{\frac{1}{16}} } \sqrt{\log \sbr{\frac{1}{\delta'}}}  \Bigg) \cdot
		\\
		& \quad \! \frac{1}{NT} \!\sum_{n=0}^{N-1} \!\sum_{t=0}^{T-1} \!\!\sqrt{\ex_{\tau \sim \cO^{\pi^t}_{\rho^n_{\cover}}} \!\! \mbr{ \nbr{\sum_{h=0}^{H(\tau)} \psi_0(s_h,a_h)}_{(\Sigma^{\nn,n}_{\hf})^{-1}} } } \!+\! \frac{3 \sqrt{N} \sbr{c_{\scale} R^3 \log\sbr{\frac{1}{\delta'}} }^{\frac{1}{4}} }{ \sqrt{1-\gamma} \underline{c}^{\frac{1}{4}} m^{\frac{1}{8}}} 
		\\
		&\overset{\textup{(a)}}{\leq} \Bigg( \sqrt{\frac{N}{c^{\nn}_{\mle}}} \sbr{\frac{5md \log \sbr{\frac{N}{\delta'}}}{M_{\hf}}}^{\frac{1}{4}} + \frac{\sqrt{2N}}{\sbr{c^{\nn}_{\mle}}^{\frac{1}{4}} } \sbr{ 17W_{\tau} R \sqrt{ \frac{\log\sbr{\frac{1}{\delta'}}}{M^{\mu}_{\subsgd}}  }}^{\frac{1}{4}}
		\\
		&\quad\ + 4\sqrt{R}N^{\frac{1}{4}}\zeta_{\hf}^{\frac{1}{4}} + \frac{7 \sqrt{N} c_{\scale}^{\frac{1}{8}} R^{\frac{5}{8}} M_{\hf}^{\frac{1}{8}}   \sbr{W_{\tau} \exp(4W_{\tau})}^{\frac{1}{4}} }{\underline{c}^{\frac{1}{8}} \sqrt{c^{\nn}_{\mle}} \sbr{1-\gamma}^{\frac{1}{4}}  m^{\frac{1}{16}} } \sqrt{\log \sbr{\frac{1}{\delta'}}}  \Bigg) \cdot
		\\
		& \quad  \underbrace{\sbr{ 2m^{\frac{1}{4}} d^{\frac{1}{4}} \log^{\frac{1}{4}}\sbr{ 1+ \frac{4 N W_{\tau}^2}{\zeta_{\hf} md} } + \frac{ 2m^{\frac{1}{4}} d^{\frac{1}{4}} \log^{\frac{1}{4}}(N) }{c_{\base}^{\frac{1}{4}}}}}_{\tilde{d}_{\hf}} + \frac{3 \sqrt{N} \sbr{c_{\scale} R^3 \log\sbr{\frac{1}{\delta'}} }^{\frac{1}{4}} }{ \sqrt{1-\gamma} \underline{c}^{\frac{1}{4}} m^{\frac{1}{8}}} ,
	\end{align*}
	where inequality (a) uses Lemma~\ref{lemma:sum_matrix_norm} with feature dimension $md$.
	
	Recall that $\delta':=\frac{\delta}{24N(K+M_{\hf}+M^{\mu}_{\subsgd}+TM^{\theta}_{\subsgd})}$, $\zeta_{\cover}:=1$, $\zeta_{\hf}:=4W_{\tau}^2$, $W_{S}:=1$, $W^{\nn}_{\theta}:= \sqrt{m} \bar{c} + R$, $W^{\nn}_{\nabla F}:=\frac{4}{(1-\gamma)^2} + \frac{4(\sqrt{m} \bar{c}+R)}{1-\gamma}$, $W^{\nn}_{f}:= \sqrt{m}\bar{c}+R$,  $W^{\nn}_{Q}:= \frac{\sqrt{m}\bar{c}+R}{1-\gamma}+\frac{2}{(1-\gamma)^2}$, $\xi_{\theta}:=\frac{R }{ W^{\nn}_{\nabla F} \sqrt{M^{\theta}_{\subsgd}}}$, $\xi_{\mu}:=\frac{R}{W_{\tau} \sqrt{M^{\mu}_{\subsgd}}}$, 
	$\eta:=\frac{\log(|\cA|)}{W^{\nn}_{\theta}\sqrt{W_{S} T} }$ and
	$c^{\nn}_{\mle}:= (2+\exp(-2W_{\tau}(\sqrt{m}\bar{c}+R))+\exp(2W_{\tau}(\sqrt{m}\bar{c}+R)))^{-1}$. $K$ and $M_{\hf}$ should satisfy that $K \geq \frac{16 (N+1)^2  \log^2\sbr{\frac{4dN}{\delta'}} }{\zeta_{\cover}^2}$ and $M_{\hf} \geq \frac{16 W_{\tau}^4 \log^2\sbr{\frac{4dN}{\delta'}} }{\zeta_{\hf}^2}$, respectively.

	Therefore, summing over $n=0,\dots,N-1$, dividing $N$, and applying Lemma~\ref{lemma:sum_occupancy_not explore}, we have
	\begin{align}
		&\ \quad  V^{\pi^{*}}(s_{\init}) - V^{\pi^{\out}}(s_{\init}) 
		\nonumber\\
		&= \frac{1}{N} \sum_{n=0}^{N-1} \sbr{ V^{\pi^{*}}(s_{\init}) - V^{\pi^{n+1}}(s_{\init}) }
		\nonumber\\
		&\leq \frac{2 \sqrt{|\cA| \varepsilon^{\nn}_{\bias}}}{1-\gamma} + \underbrace{\frac{\log(|\cA|)}{(1-\gamma)\eta T} + \frac{\eta W_{S} (W^{\nn}_{\theta})^2 }{(1-\gamma)}}_{ =\frac{ W^{\nn}_{\theta} \sqrt{W_{S} \log(|\cA|)} }{ (1-\gamma) \sqrt{ T} } } 
		+ \frac{12R\sqrt{\beta \zeta_{\cover} }}{1-\gamma} + \frac{2 md}{(1-\gamma) N \beta} \log\sbr{ 1+ \frac{N  }{\zeta_{\cover} md} }
		\nonumber\\
		&\quad\ 
		+ \frac{8\sqrt{ \beta N W^{\nn}_{\nabla F} R}  }{1-\gamma}  \sbr{ \frac{\log\sbr{\frac{1}{\delta'}}}{M^{\theta}_{\subsgd}} }^{\frac{1}{4}} 
		+ \frac{8 \tilde{d}_{\hf} \sqrt{\beta W^{\nn}_{Q}}}{1-\gamma}  \Bigg( \frac{4 \sqrt{N} m^{\frac{1}{4}} d^{\frac{1}{4}} \log^{\frac{1}{4}} \sbr{\frac{N}{\delta'}}}{ \sqrt{c^{\nn}_{\mle}} M_{\hf}^{\frac{1}{4}} } + \frac{5 \sqrt{N} W_{\tau}^{\frac{1}{4}}  R^{\frac{1}{4}} \log^{\frac{1}{8}}\sbr{\frac{1}{\delta'}} }{(c^{\nn}_{\mle})^{\frac{1}{4}} (M^{\mu}_{\subsgd})^{\frac{1}{8}}} 
		\nonumber\\
		&\quad\
		+ 4\sqrt{R}N^{\frac{1}{4}}\zeta_{\hf}^{\frac{1}{4}} 
		+ \frac{9 \sqrt{N} c_{\scale}^{\frac{1}{8}} R^{\frac{5}{8}} M_{\hf}^{\frac{1}{8}}   \sbr{W_{\tau} \exp(4W_{\tau})}^{\frac{1}{4}} }{\underline{c}^{\frac{1}{8}} \sqrt{c^{\nn}_{\mle}} \sbr{1-\gamma}^{\frac{1}{4}}  m^{\frac{1}{16}} } \sqrt{\log \sbr{\frac{1}{\delta'}}}  \Bigg)
		\nonumber\\
		&\quad\ + \frac{8 \sqrt{\beta W^{\nn}_{Q}}}{1-\gamma} \cdot \frac{3 \sqrt{N} \sbr{c_{\scale} R^3 \log\sbr{\frac{1}{\delta'}} }^{\frac{1}{4}} }{ \sqrt{1-\gamma} \underline{c}^{\frac{1}{4}} m^{\frac{1}{8}}}
		+ \frac{ 16R\sqrt{\beta N (W^{\nn}_{f} + W^{\nn}_{Q})} c_{\scale}^{\frac{1}{4}}  R^{\frac{1}{4}} }{(1-\gamma)\underline{c}^{\frac{1}{4}} m^{\frac{1}{8}} } 
		+ \frac{4 |\cA| \sqrt{c_{\scale} R^3}}{ \sqrt{\underline{c}} m^{\frac{1}{4}}} . \label{eq:nn_suboptimality}
	\end{align}
	
\end{proof}

\section{Technical Tools}

\begin{lemma} \label{lemma:concentration_cD}
	Let $\cD$ be a distribution of random vector $\phi \in \R^d$ such that $\|\phi\|_2 \leq W$ and $\Sigma=\ex_{\phi \sim \cD} [\phi \phi^\top]$. Given $K$ i.i.d. samples $\phi_1,\dots,\phi_{K} \sim \cD$, then with probability at least $1-\delta'$,
	\begin{align*}
		\Pr \mbr{ \nbr{\frac{1}{K} \sum_{i=1}^{K} \phi_i \phi_i^\top - \Sigma } \leq \frac{2 W^2 \log\sbr{\frac{4d}{\delta'}} }{\sqrt{K}}  } .
	\end{align*}
\end{lemma}
\begin{proof}
	This analysis is originated from Lemma H.3 in \cite{agarwal2020pc}.
	
	Let $X_i=\phi_i \phi_i^\top - \Sigma$, and it holds that $\ex[X_i]=0$ and $\|X_i\| \leq W^2$.
	
	Then, we have
	\begin{align*}
		\ex\mbr{X_i^2} &= \ex\mbr{\sbr{\phi_i \phi_i^\top - \Sigma}^2}
		\\
		&= \ex\mbr{\sbr{\phi_i \phi_i^\top}^2 + \Sigma^2 - 2 \Sigma \phi_i \phi_i^\top}
		\\
		&= \ex\mbr{\sbr{\phi_i \phi_i^\top}^2} + \Sigma^2 - 2 \Sigma \ex\mbr{\phi_i \phi_i^\top}
		\\
		&= \ex\mbr{\sbr{\phi_i \phi_i^\top}^2} + \Sigma^2 - 2 \Sigma^2
		\\
		&= \ex\mbr{\sbr{\phi_i \phi_i^\top}^2} - \Sigma^2 .
	\end{align*}
	
	For any $x \in \R^d$,
	\begin{align*}
		x^\top \sbr{\ex\mbr{\sbr{\phi_i \phi_i^\top}^2} - \ex\mbr{X_i^2}} x = x^\top \Sigma^2 x = \sbr{\Sigma x}^\top \Sigma x \geq 0 ,
	\end{align*}
	which implies
	\begin{align*}
		\ex\mbr{\sbr{\phi_i \phi_i^\top}^2} \succeq \ex\mbr{X_i^2} .
	\end{align*}
	
	For any $x \in \R^d$,
	\begin{align*}
		x^\top \sbr{ W^2 \ex\mbr{\phi_i \phi_i^\top} - \ex\mbr{\sbr{\phi_i \phi_i^\top}^2} } x &= W^2 \cdot x^\top  \ex\mbr{\phi_i \phi_i^\top} x - x^\top \ex\mbr{\sbr{\phi_i \phi_i^\top} \sbr{\phi_i \phi_i^\top}} x
		\\
		&\geq W^2 \cdot x^\top \ex\mbr{\phi_i \phi_i^\top} x - W^2 \cdot x^\top \ex\mbr{\phi_i \phi_i^\top} x
		\\
		&= 0 ,
	\end{align*}
	which implies
	\begin{align*}
		W^2 \ex\mbr{\phi_i \phi_i^\top} \succeq \ex\mbr{\sbr{\phi_i \phi_i^\top}^2} .
	\end{align*}
	
	Then, we have
	\begin{align*}
		\ex\mbr{X_i^2} \preceq \ex\mbr{\sbr{\phi_i \phi_i^\top}^2} \preceq W^2 \ex\mbr{\phi_i \phi_i^\top} = W^2 \Sigma ,
	\end{align*}
	and thus, 
	\begin{align*}
		\sum_{i=1}^{K} \ex\mbr{X_i^2} &\preceq K W^2 \Sigma ,
		\\
		\nbr{\sum_{i=1}^{K} \ex\mbr{X_i^2}} &\leq K W^4 .
	\end{align*}
	
	Using the Matrix Bernstein inequality (Theorem 7.7.1 in \cite{tropp2015introduction}), we have that for any $t \geq W^2 \sqrt{K} + \frac{1}{3} W^2$,
	\begin{align*}
		\Pr \mbr{ \nbr{\sum_{i=1}^{K} X_i} \geq t } \leq 4d \exp\sbr{ \frac{-\frac{1}{2} t^2}{W^4 K + \frac{1}{3} W^2 t} } , 
	\end{align*}
	which is equivalent to that for any $z \geq \frac{W^2}{\sqrt{K}} + \frac{1}{3} \frac{W^2}{K}$,
	\begin{align*}
		\Pr \mbr{ \nbr{\frac{1}{K} \sum_{i=1}^{K} X_i} \geq z } \leq 4d \exp\sbr{ \frac{-\frac{1}{2} K^2 z^2 }{W^4 K + \frac{1}{3} W^2 K z} } .
	\end{align*}
	
	Let
	\begin{align*}
		z=\frac{2 W^2 \log\sbr{\frac{4d}{\delta'}} }{\sqrt{K}}  .
	\end{align*}
	Then,
	\begin{align*}
		&\ \quad \Pr \mbr{ \nbr{\frac{1}{K} \sum_{i=1}^{K} X_i} \geq \frac{2 W^2 \log\sbr{\frac{4d}{\delta'}} }{\sqrt{K}} }
		\\
		& \leq 4d \exp\sbr{ \frac{-\frac{1}{2} K^2 \cdot \frac{4 W^4 \log^2\sbr{\frac{4d}{\delta'}} }{K} }{W^4 K + \frac{1}{3} W K \cdot \frac{2 W^2 \log\sbr{\frac{4d}{\delta'}} }{\sqrt{K}} } } 
		\\
		& = 4d \exp\sbr{ - \frac{ 2 W^4 K \log^2\sbr{\frac{4d}{\delta'}}   }{W^4 K + \frac{2}{3} W^3 \sqrt{K} \log\sbr{\frac{4d}{\delta'}}  } } 
		\\
		& \leq 4d \exp\sbr{ - \log\sbr{\frac{4d}{\delta'}} }
		\\
		& = \delta' .
	\end{align*}
	
	Thus, with probability at least $1-\delta'$,
	\begin{align*}
		\Pr \mbr{ \nbr{\frac{1}{K} \sum_{i=1}^{K} \phi_i \phi_i^\top - \Sigma } \leq \frac{2 W^2 \log\sbr{\frac{4d}{\delta'}} }{\sqrt{K}}  } .
	\end{align*}
\end{proof}

\begin{lemma} \label{lemma:con_matrix_inverse}
	For any $n \in [N]$, let $\cD^n$ be a distribution of random vector $\phi \in \R^d$ such that $\|\phi\|_2 \leq W$, and define $\Sigma^n=\ex_{\phi \sim \cD^n} [\phi \phi^\top]$ and $\Sigma=\sum_{n=1}^{N} \Sigma^n$. For any $n \in [N]$, given $K$ i.i.d. samples $\phi^n_1,\dots,\phi^n_K \sim \cD^n$, let $\hat{\Sigma}^n=\frac{1}{K} \sum_{j=1}^{K} \phi^n_j (\phi^n_j)^{\top}$ and $\hat{\Sigma}=\sum_{n=1}^{N} \hat{\Sigma}^n$. Letting $K \geq \frac{16 N^2 W^4 \log^2\sbr{\frac{4dN}{\delta'}} }{\zeta^2}$, then with probability at least $1-\delta'$, we have that for any $x \in \R^d$,
	\begin{align*}
		\frac{1}{2} x^\top \sbr{\Sigma+\zeta I}^{-1} x \leq x^\top \sbr{\hat{\Sigma}+\zeta I}^{-1} x \leq 2 x^\top \sbr{\Sigma+\zeta I}^{-1} x.
	\end{align*}
\end{lemma}
\begin{proof}
	This proof is originated from Lemma H.4 in \cite{agarwal2020pc}.
	
	According to Lemma~\ref{lemma:concentration_cD}, we have that for any $n \in [N]$,
	with probability at least $1-\frac{\delta'}{N}$,
	\begin{align*}
		\Pr \mbr{ \nbr{\frac{1}{K} \sum_{j=1}^{K} \phi^n_j (\phi^n_j)^\top - \Sigma^n } \leq \frac{2 W^2 \log\sbr{\frac{4dN}{\delta'}} }{\sqrt{K}}  } .
	\end{align*}
	Thus, 
	\begin{align*}
		\Sigma^n - \frac{2 W^2 \log\sbr{\frac{4dN}{\delta'}} }{\sqrt{K}} I \preceq \frac{1}{K} \sum_{j=1}^{K} \phi^n_j (\phi^n_j)^\top  \preceq  \Sigma^n + \frac{2 W^2 \log\sbr{\frac{4dN}{\delta'}} }{\sqrt{K}} I  ,
	\end{align*}
	and then summing over $n \in [N]$, we have
	\begin{align*}
		\Sigma - \frac{2 N W^2 \log\sbr{\frac{4dN}{\delta'}} }{\sqrt{K}} I + \zeta I \preceq \hat{\Sigma}  + \zeta I \preceq  \Sigma + \frac{2 N W^2 \log\sbr{\frac{4dN}{\delta'}} }{\sqrt{K}} I  + \zeta I .
	\end{align*}
	This implies that for $\zeta \geq  \frac{2 N W^2 \log\sbr{\frac{4dN}{\delta'}} }{\sqrt{K}}$,
	\begin{align*}
		\sbr{\Sigma + \frac{2 N W^2 \log\sbr{\frac{4dN}{\delta'}} }{\sqrt{K}} I  + \zeta I}^{-1} \preceq \sbr{\hat{\Sigma}  + \zeta I}^{-1} \preceq  \sbr{\Sigma - \frac{2 N W^2 \log\sbr{\frac{4dN}{\delta'}} }{\sqrt{K}} I + \zeta I}^{-1} .
	\end{align*}
	
	Let $U \Lambda U^\top$ be the eigendecomposition of $\Sigma$, where $\Lambda=diag([\lambda_1,\dots,\lambda_d])$ and $U=[u_1,\dots,u_d]$. Then, we have
	\begin{align*}
		&\ \quad x^\top \sbr{\hat{\Sigma}+\zeta I}^{-1} x - x^\top \sbr{\Sigma+\zeta I}^{-1} x 
		\\
		&\leq x^\top \sbr{\Sigma - \frac{2 N W^2 \log\sbr{\frac{4dN}{\delta'}} }{\sqrt{K}} I + \zeta I}^{-1} x - x^\top \sbr{\Sigma+\zeta I}^{-1} x
		\\
		&= \sum_{i \in [d]} \sbr{\sbr{\sigma_i + \zeta - \frac{2 N W^2 \log\sbr{\frac{4dN}{\delta'}} }{\sqrt{K}} }^{-1} - \sbr{\sigma_i+\zeta }^{-1} } \sbr{u_i x}^2 .
	\end{align*}
	
	Since
	$\zeta \geq \frac{4 N W^2 \log\sbr{\frac{4dN}{\delta'}} }{\sqrt{K}}$,
	we have 
	\begin{align*}
		2 \sbr{\sigma_i + \zeta - \frac{2 N W^2 \log\sbr{\frac{4dN}{\delta'}} }{\sqrt{K}}} &= \sigma_i + \zeta - \frac{4 N W^2 \log\sbr{\frac{4dN}{\delta'}} }{\sqrt{K}} + \sigma_i+\zeta
		\\
		&\geq \sigma_i+\zeta ,
	\end{align*}
	and thus
	\begin{align*}
		\sbr{\sigma_i + \zeta - \frac{2 N W^2 \log\sbr{\frac{4dN}{\delta'}} }{\sqrt{K}}}^{-1} \leq 2 \sbr{\sigma_i+\zeta}^{-1} .
	\end{align*}
	
	Hence,
	\begin{align*}
		\ \quad x^\top \sbr{\hat{\Sigma}+\zeta I}^{-1} x - x^\top \sbr{\Sigma+\zeta I}^{-1} x &\leq \sum_{i \in [d]}  \sbr{\sigma_i+\zeta }^{-1}  \sbr{u_i x}^2
		\\
		&= x^\top \sbr{\Sigma+\zeta I}^{-1} x .
	\end{align*}
	
	On the other hand, we have
	\begin{align*}
		&\ \quad x^\top \sbr{\Sigma+\zeta I}^{-1} x - x^\top \sbr{\hat{\Sigma}+\zeta I}^{-1} x
		\\
		&\leq x^\top \sbr{\Sigma+\zeta I}^{-1} x - x^\top \sbr{\Sigma + \frac{2 N W^2 \log\sbr{\frac{4dN}{\delta'}} }{\sqrt{K}} I + \zeta I}^{-1} x
		\\
		&= \sum_{i \in [d]} \sbr{ \sbr{\sigma_i+\zeta }^{-1} - \sbr{\sigma_i + \zeta + \frac{2 N W^2 \log\sbr{\frac{4dN}{\delta'}} }{\sqrt{K}} }^{-1} } \sbr{u_i x}^2 .
	\end{align*}
	
	Since
	$\zeta \geq \frac{2 N W^2 \log\sbr{\frac{4dN}{\delta'}} }{\sqrt{K}}$, we have 
	\begin{align*}
		2 \sbr{\sigma_i + \zeta} &=  \sigma_i + \zeta + \sigma_i + \zeta
		\\
		&\geq \sigma_i + \zeta + \frac{2 N W^2 \log\sbr{\frac{4dN}{\delta'}} }{\sqrt{K}} ,
	\end{align*}
	and thus
	\begin{align*}
		\sbr{\sigma_i+\zeta}^{-1} \leq 2 \sbr{\sigma_i + \zeta + \frac{2 N W^2 \log\sbr{\frac{4dN}{\delta'}} }{\sqrt{K}}}^{-1}  .
	\end{align*}
	
	Hence,
	\begin{align*}
		&\ \quad x^\top \sbr{\Sigma+\zeta I}^{-1} x - x^\top \sbr{\hat{\Sigma}+\zeta I}^{-1} x
		\\
		&\leq \sum_{i \in [d]}  \sbr{\sigma_i + \zeta + \frac{2 N W^2 \log\sbr{\frac{4dN}{\delta'}} }{\sqrt{K}} }^{-1}  \sbr{u_i x}^2
		\\
		&= x^\top \sbr{\hat{\Sigma}+\zeta I}^{-1} x .
	\end{align*}
	
\end{proof}

\begin{lemma}\label{lemma:tech_sq_diff}
	For any $a,b,c \in \R$, we have
	\begin{align*}
		(b-a)^2-(c-a)^2\leq 4\max\{|a|,|b|,|c|\} |b-c| .
	\end{align*}
\end{lemma}
\begin{proof}
	It holds that
	\begin{align*}
		(b-a)^2-(c-a)^2&=(a^2+b^2-2ab)-(a^2+c^2-2ac)
		\\
		&=b^2-c^2-2a(b-c)
		\\
		&=(b+c)(b-c)-2a(b-c)
		\\
		&=(b+c-2a)(b-c)
		\\
		&\leq 4\max\{|a|,|b|,|c|\} |b-c| .
	\end{align*}
\end{proof}	

\begin{lemma}[Theorem 2.1 in \cite{hsu2012tail}]\label{lemma:con_Ax}
	Let $A \in \R^{n \times n}$ be a matrix, and let $\Sigma:=A^\top A$. Suppose that $x=(x_1,\dots,x_n)$ is a random vector such that, for some $\nu \in \R^n$ and $\sigma \geq 0$,
	\begin{align*}
		\ex\mbr{ \exp\sbr{ \alpha^\top \sbr{x-\nu} } } \leq \exp\sbr{ \frac{ \nbr{\alpha}^2 \sigma^2}{2} }
	\end{align*}
	for all $\alpha \in \R^n$. For all $t>0$,
	\begin{align*}
		\Pr\mbr{ \nbr{Ax}^2 > \sigma^2 \sbr{ \trace(\Sigma) + 2\sqrt{\trace(\Sigma^2) t} + 2 \nbr{\Sigma} t } + \trace(\Sigma \nu \nu^\top) \sbr{1+2\sbr{ \frac{\nbr{\Sigma}^2}{\trace\sbr{\Sigma^2}} t}^{\frac{1}{2}} } } \leq \exp(-t) .
	\end{align*}
\end{lemma}

\begin{lemma} [Lemma 11 in \cite{abbasi2011improved}]\label{lemma:tech_abbasi2011}
	Let $X_1,\dots,X_N$ be a sequence of $d \times d$-dimensional positive semi-definite matrices, and $\|X_n\|\leq W_{x}$ for all $n \in [N]$. Let $A_0=\zeta I_d$ with $\zeta \geq \max\{1,W_{x}\}$.
	For any $n \in [N]$, let $A_n = A_0 + \sum_{i=1}^{n} X_i$.
	Then, we have
	\begin{align*}
		\sum_{n=1}^{N} \trace\sbr{A_{n-1}^{-1} X_n} \leq 2 \log\sbr{\frac{\det(A_N)}{\det(A_0)}} .
	\end{align*}
\end{lemma}

\end{document}